\documentclass{siamart0516}
\usepackage{braket,amsfonts}
\usepackage{array}
\usepackage{pgfplots}
\pgfplotsset{compat=newest}
\usepackage[caption=false]{subfig}
\usepackage{mathtools}
\usepackage{eqparbox}
\usepackage{todonotes}
\newsiamthm{claim}{Claim}
\newsiamremark{conjecture}{Conjecture}
\newsiamremark{rem}{Remark}
\newsiamremark{expl}{Example}
\newsiamremark{hypothesis}{Hypothesis}
\crefname{hypothesis}{Hypothesis}{Hypotheses}
\newtheorem{assumptions}{\textbf{Assumptions}}

\usepackage{algorithmic}

\usepackage{graphicx,epstopdf}
\usepackage{float}

\Crefname{ALC@unique}{Line}{Lines}

\usepackage{amsopn}

\usepackage{verbatim}

\usepackage{MnSymbol}

\usepackage{ulem}

\usepackage{lineno}
\usepackage{dsfont} 

\newcommand{\zp}{Z_{\rm{p}}}
\newcommand{\zl}{Z_{\rm{ls}}}
\newcommand{\sZ}{\mathsf{Z}}

\newcommand{\xii}{\zeta}
\newcommand{\noise}{\xi}

\newcommand{\nil}{\nu_{\rm {ls}}}
\newcommand{\nilo}{\nu_{\rm {ls,1}}}
\newcommand{\nilt}{\nu_{\rm {ls,2}}}
\newcommand{\nip}{\nu_{\rm {p}}}
\newcommand{\nipo}{\nu_{\rm {p,1}}}
\newcommand{\nipt}{\nu_{\rm {p,2}}}

\newcommand{\bbR}{\mathbb R}
\newcommand{\la}{\langle}
\newcommand{\ra}{\rangle}

\newcommand{\cL}{\mathcal{L}}
\newcommand{\cC}{\mathcal{C}}
\newcommand{\cB}{\mathcal{B}}

\newcommand{\cP}{\mathcal{P}}
\newcommand{\cA}{\mathcal{A}}
\newcommand{\cH}{\mathcal{H}}

\newcommand{\one}{\mathds{1}}

\newcommand{\Id}{\mathrm{Id}}
\newcommand{\Jk}{\mathsf{J}}
\newcommand{\Jp}{\mathsf{J}_{\rm p}}
\newcommand{\Jl}{\mathsf{J}_{\rm ls}}
\newcommand{\Jkr}{\mathsf{J}_{\rm k}}

\newcommand{\Phip}{\Phi_{\rm p}}
\newcommand{\Phipo}{\Phi_{\rm p,1}}
\newcommand{\Phipt}{\Phi_{\rm p,2}}
\newcommand{\Phil}{\Phi_{\rm ls}}
\newcommand{\Philo}{\Phi_{\rm ls,1}}
\newcommand{\Philt}{\Phi_{\rm ls,2}}

\newcommand{\rom}{\rho^-}
\newcommand{\rop}{\rho^+}

\newcommand{\iid}{\stackrel{\mathrm{iid}}{\sim}}

\newcommand{\rn}{r_n}

\definecolor{darkred}{rgb}{0.6,0.1,0.1}
\definecolor{darkgreen}{rgb}{0.1,0.6,0.1}
\definecolor{darkblue}{rgb}{0.1,0.1,0.6}
\definecolor{darkyellow}{rgb}{0.8,0.8,0}

\def\mt#1{#1} %{\textbf{\textcolor{darkblue}{#1}}}
\definecolor{mygreen}{rgb}{0.1,0.75,0.2}

\newcommand{\te}{\textrm}
\newcommand{\tacka}{\,\cdot\,}

\newcommand{\R}{\mathbb{R}}
\newcommand{\bbE}{\mathbb{E}}

\newcommand{\bbN}{\mathbb{N}}

\newcommand{\N}{\mathbb{N}}

\newcommand{\X}{\mathcal{X}}
\newcommand{\VE}{\varepsilon}
\newcommand{\eps}{\varepsilon}
\newcommand{\dd}{ \mathrm{d}}

\newcommand{\be}{\begin{equation}}
\newcommand{\en}{\end{equation}}

\DeclareMathOperator*{\muesssup}{\mu\text{-}ess\,sup}

\DeclareMathOperator*{\Glim}{\Gamma\text{-}\lim}
\DeclareMathOperator*{\toweak}{\Rightarrow}

\newcommand{\dist}{\mathrm{dist}}
\newcommand{\Lip}{\mathrm{Lip}}

\newcommand{\disP}{A}
\newcommand{\ctsP}{\mathcal{A}}

\graphicspath{{./images/}}

\title{Large Data and Zero Noise Limits of Graph-Based\\Semi-Supervised Learning Algorithms
}

\author{Matthew M. Dunlop
\thanks{Computing and Mathematical Sciences, Caltech, Pasadena, CA 91125 (\email{mdunlop@caltech.edu })}
\and
Dejan Slep\v{c}ev
\thanks{Department of Mathematical Sciences, Carnegie Mellon University, 
Pittsburgh, PA 15213 (\email{slepcev@math.cmu.edu})} 
\and
Andrew M. Stuart
\thanks{Computing and Mathematical Sciences, Caltech, Pasadena, CA 91125 (\email{astuart@caltech.edu }).}
\and
Matthew Thorpe
\thanks{Department of Applied Mathematics and Theoretical Physics, University of Cambridge, Cambridge, CB3 0WA (\email{m.thorpe@maths.cam.ac.uk})} 
}

\begin{document}
\date{today}
\maketitle

%% ------------------------------------------------------------------
%% ABSTRACT
%% ------------------------------------------------------------------
\begin{abstract}
Scalings in which the graph Laplacian approaches a differential operator
in the large graph limit are used to develop understanding of a 
number of algorithms for semi-supervised learning; in particular the
extension, to this graph setting, of the probit algorithm, 
level set and kriging methods, are studied. Both optimization and
Bayesian approaches are considered, based around a regularizing quadratic
form found from an affine transformation of the Laplacian, raised
to a, possibly fractional, exponent. Conditions on the parameters
defining this quadratic form are identified under which well-defined
limiting continuum analogues of the optimization and Bayesian 
semi-supervised learning problems may be found, thereby shedding
light on the design of algorithms in the large graph setting.
The large graph limits of the optimization formulations
are tackled through $\Gamma-$convergence,
using the recently introduced $TL^p$ metric. The small labelling noise   
\mt{limits} of the Bayesian formulations are also identified, and contrasted with
pre-existing harmonic function approaches to the problem.
\end{abstract}



%% ------------------------------------------------------------------
%% END HEADER
%% ------------------------------------------------------------------

\section{Introduction}
\label{sec:I}

\subsection{Context}
\label{ssec:B}
This paper is concerned with the semi-supervised learning problem of
determining labels on an entire set of (feature) vectors $\{x_j\}_{j \in Z}$,
given (possibly noisy) labels $\{y_j\}_{j \in Z'}$ on a subset of feature vectors
with indices  $j \in Z' \subset Z$. 
To be concrete we will assume that the $x_j$ are elements of $\R^d$, 
$d \ge 2$, and
consider the binary classification problem in which the $y_j$ are elements
of $\{\pm 1\}$. Our goal is to characterize algorithms
for this problem in the large data limit where $n=|Z| \to \infty$; 
additionally we will study the limit
where the  noise in the label data disappears. Studying these
limits yields insight into the
classification problem and algorithms for it.

Semi-supervised learning as a subject has been developed primarily over the
last two decades and the references \cite{zhu2005semi,zhuthesis} provide 
an excellent source for the historical context. Graph based methods proceed by
forming a graph with $n$ nodes $Z$, 
and use the \mt{unlabeled} data $\{x_j\}_{j \in Z}$ to provide an
$n \times n$ weight matrix $W$ quantifying the affinity of the nodes of the graph 
with one another. The labelling information on $Z'$ is then spread to
the whole of $Z$, exploiting these affinities. In the absence of labelling
information we obtain the problem of unsupervised learning; 
for example the spectrum of the graph Laplacian $L$ forms the basis of
widely used spectral clustering methods \cite{belkin2002laplacian,ng2002spectral,von2007tutorial}. 
Other approaches are combinatorial, and largely focussed on graph cut methods 
\cite{blum2001learning,boykov2001fast,shi2000normalized}. 
However relaxation and approximation
are required to beat the combinatorial hardness of these 
problems \cite{madry2010fast}
leading to a range of methods based on Markov random fields \cite{li2012markov} and total variation relaxation \cite{SzlamBresson}.
In \cite{zhuthesis} a number of new approaches were introduced, including label
propagation and the generalization of kriging, or Gaussian process regression
\cite{wahba1990spline}, to the graph setting \cite{zhu2003semi}. These regression methods
opened up new approaches to the problem, but were limited in scope because the
underlying real-valued Gaussian process was linked directly to the 
categorical label data which is (arguably) not natural from a modelling
perspective; see \cite{neal} for a discussion of the
distinctions between regression and classification. The logit and probit methods of classification \cite{williams1996gaussian}
side-step this problem by postulating a link function which relates
the underlying Gaussian process to the categorical data, amounting to a model
linking the \mt{unlabeled} and \mt{labeled} data. The support vector machine
\cite{bishop} makes a similar link, but it lacks a natural probabilistic interpretation.

The probabilistic formulation is important when it is desirable to equip the
classification with measures of uncertainty.  Hence, we will concentrate on 
the probit algorithm in this paper, and variants on it, as it has 
a probabilistic formulation. The statement of the probit algorithm in the context of graph based 
semi-supervised learning may be found in \cite{UQ17}.
An approach bridging the combinatorial and Gaussian process approaches is the
use of Ginzburg-Landau models which work with real numbers but use a penalty to
constrain to values close to the range of the label data $\{\pm 1\}$; these
methods were introduced in \cite{bertozzi2012diffuse}, large data limits studied in~\cite{cristoferi18,thorpe17,vangennip12a}, and given a probabilistic
interpretation in \cite{UQ17}. Finally we mention the Bayesian level set
method. This approach takes the idea of using level sets for inversion in the
class of interface problems \cite{BO05} and gives it a probabilistic formulation
which has both theoretical foundations and leads to efficient 
algorithms \cite{iglesias2015bayesian}; classification may be viewed as an interface
problem on a graph (a graph cut is an interface for example) and thus the
Bayesian level set method is naturally extended to this setting as shown in
\cite{UQ17}. As part of this paper we will show that the probit and
Bayesian level set methods are closely related.

A significant challenge for the field, both in terms of algorithmic development,
and in terms of fundamental theoretical understanding, is the setting in which
the volume of \mt{unlabeled} data is high, relative to the volume of \mt{labeled} data.
One way to understand this setting is through the study of large data limits in
which $n=|Z| \to \infty.$ This limit is studied in \cite{von2008consistency}, and
was addressed more recently under different assumptions in \cite{trillos}. Both
papers assume that the \mt{unlabeled} data is drawn i.i.d. from a measure with
Lebesgue density on a subset of $\R^d$, but the assumptions on graph construction
differ: in \cite{von2008consistency} the graph bandwidth is fixed as $n \to  \infty$ resulting in the limit of the graph Laplacian being a non-local operator, whilst in \cite{trillos} the bandwidth vanishes in the limit resulting in the limit being a weighted Laplacian (divergence form elliptic operator).

In \cite{nadler09} it is demonstrated that algorithms
based on use of the discrete Dirichlet energy computed from the graph 
Laplacian can behave poorly for $d \ge 2$, in the large data limit,
if they attempt pointwise labelling. 
In \cite{belkin09} it is argued that use of quadratic forms
based on powers $\alpha>\frac{d}{2}$ of the graph Laplacian can ameliorate this problem.
Our work, which studies a range of algorithms all based on optimization
or Bayesian formulations exploiting quadratic forms, will take this body of work
considerably further, proving large data limit theorems for a variety of 
algorithms, and showing the role of the
parameter $\alpha$ in this 
infinite data limit. In doing so we shed light
on the difficult question of how to scale and tune algorithms for graph based
semi-supervised learning; in particular we state limit theorems
of various kinds which require, respectively, either $\alpha>\frac{d}{2}$
or $\alpha>d$ to hold.  
We also study the small noise limit and show how both the probit and
Bayesian level set algorithms coincide and, furthermore, provide
a natural generalization of the harmonic functions approach of
\cite{zhu2003semi,zhu2003semib}, a generalization which is arguably more natural
from a modeling perspective.

Our large data limit theorems concern the maximum a posteriori (MAP)
estimator rather than a Bayesian posterior distribution.
However two remarkable recent papers \cite{trillos2017continuum,trillos2017consistency} demonstrate a methodology for proving
limit theorems concerning Bayesian posterior distributions themselves,
exploiting the variational characterization of Bayes theorem;
extending the work in those papers to the algorithms considered
in this paper would be of great interest.

\subsection{Our Contribution}
\label{ssec:C}

We derive a canonical continuum  inverse problem which characterizes graph based semi-supervised learning:  find function $u: \Omega \subset \R^d \mapsto \R$ from knowledge of 
${\rm sign}(u)$ on $\Omega' \subset \Omega$.
\footnote{ We note that throughout the paper $\Omega$ is the physical domain, and not the set of events of a probability space.  }
 The latent variable $u$ 
characterizes the \mt{unlabeled} data and its sign is the \mt{labeling} information. 
This highly ill-posed inverse problem is potentially solvable because of the 
very strong prior information provided by the \mt{unlabeled} data; we
characterize this information via a mean zero Gaussian process prior 
on $u$ with covariance operator $\cC \propto (\cL+\tau^2 I)^{-\alpha}.$
The operator $\cL$ is a weighted Laplacian found as a limit of the graph 
Laplacian, and as a consequence depends on the distribution of the
\mt{unlabeled} data.

In order to derive this canonical inverse problem
we study the probit and Bayesian
level set algorithms for semi-supervised learning. 
We build on the large \mt{unlabeled} 
data limit setting of \cite{trillos}.
In this setting there is an intrinsic scaling parameter $\eps_n$ 
that characterizes the length scale on which edge weights 
between nodes are significant; the analysis 
identifies a lower bound on $\eps_n$ which is necessary in 
order for the graph to remain connected in the large data limit
and under which the graph Laplacian $L$ converges to a differential
operator $\cL$ of weighted Laplacian form. The work uses $\Gamma-$convergence
in the $TL^2$ optimal transport metric, introduced in \cite{trillos},
and proves convergence of the quadratic form defined by $L$ 
to one defined by $\cL.$
We make the following contributions which significantly extend this work
to the semi-supervised learning setting.

\begin{itemize}
\item We prove $\Gamma-$convergence in $TL^2$ of the quadratic form 
defined by $(L+\tau^2 I)^{\alpha}$ to that defined 
by $(\cL+\tau^2 I)^{\alpha}$ and identify parameter choices in
which the limiting Gaussian
measure with covariance $(\cL+\tau^2 I)^{-\alpha}$ is well-defined.
See Theorems \ref{thm:LimitThmDir:LimitThmDir}, \ref{t:g} and
Proposition \ref{t:gold}.

\item We introduce large data limits of the probit and Bayesian level
set  problem formulations in which the volume of 
\mt{unlabeled} data $n=|Z| \to \infty$, distinguishing between the cases where 
the volume of \mt{labeled} data $|Z'|$ is fixed and where $|Z'|/n$ is fixed.  
See section \ref{sec:F} for the function space analogues of the
graph based algorithms introduced in section \ref{sec:G}.

\item We use the theory of $\Gamma-$convergence to derive a continuum limit of the probit algorithm when employed in MAP estimation mode; this theory 
demonstrates the need for $\alpha>\frac{d}{2}$ and an upper bound on $\eps_n$ in the large data limit where the volume of
\mt{labeled} data $|Z'|$ is fixed.
See Theorems \ref{thm:LimitThmOpt:Probit:pr} and
\ref{thm:LimitThmOpt:Probit:pr2neg}

\item We use the properties of Gaussian measures on function spaces to write down
well defined limits of the probit and Bayesian level set algorithms,
when employed in Bayesian probabilistic mode, to determine the posterior distribution
on labels given observed data; this theory demonstrates the need for $\alpha>\frac{d}{2}$
in order for the limiting probability distribution to be meaningful for both
large data limits; indeed, depending on the geometry of the domain from
which the feature vectors are drawn, it may require $\alpha>d$
for the case where the volume of \mt{labeled} data is fixed.
See Theorem \ref{t:g} and Proposition \ref{t:gold} for these
conditions on $\alpha$, and
for details of the limiting probability measures
see equations \eqref{eq:probprob}, \eqref{def_nipo},
\eqref{eq:probB} and \eqref{eq:probBc}.

\item We show that the probit and Bayesian level set \mt{methods} have a common
Bayesian inverse problem limit, mentioned above,
by studying their weak limits as noise levels on the
\mt{labeled} data tends to zero.
See Theorems \ref{t:add} and \ref{thm:ltp:probit:ZeroNoise}.

\item We provide numerical experiments which illusrate the
large graph limits introduced and studied in this paper; see
section \ref{sec:N}.

\end{itemize}

\subsection{Paper Structure}
\label{ssec:P}

In section \ref{sec:Q} we study a family of quadratic forms which arise naturally
in all the algorithms that we study. By means of the $\Gamma-$convergence techniques
pioneered in \cite{trillos} we show that these quadratic forms have a limit
defined by families of differential operators in which the finite graph  
parameters appear in an explicit and easily understood fashion.
Section \ref{sec:G} is devoted to the definition of the three graph
based algorithms that we study in this paper: the probit and Bayesian level set 
algorithms, and the graph analogue of kriging. In section \ref{sec:F} we write down the function
space limits of these algorithms, obtained when the volume $n$ of \mt{unlabeled}
data tends to infinity, and in the case of the maximum a posteriori
estimator for probit use $\Gamma-$convergence to study 
large graph limits rigorously; we also show that the 
probit and Bayesian level set algorithms have
a common zero noise limit. Section \ref{sec:N} contains numerical experiments
for the function space limits of the algorithms, in both optimization (MAP)
and sampling (fully Bayesian MCMC) modalities. 
We conclude in section \ref{sec:C} with a summary and
directions for future research. All proofs are given in the Appendix, section
\ref{sec:A}. This choice is made in order to separate the form and 
implications of 
the theory from the proofs; both the statements and proofs comprise the 
contributions of this work, but since they may be of interest to different 
readers they are separated, by use of the Appendix.

\section{Key Quadratic Form and Its Limits}
\label{sec:Q}

\subsection{Graph Setting}
\label{ssec:G}

From the \mt{unlabeled} data $\{x_j\}_{j=1}^n$ we construct a weighted graph 
$G=(Z,W)$ where $Z=\{1, \cdots, n\}$  are the vertices of the graph and $W$
the edge weight matrix; $W$ is assumed to have entries $\{w_{ij}\}$ between nodes
$i$ and $j$ given by
\[ w_{ij}=\eta_{\eps}(|x_i-x_j|). \]
We will discuss \mt{the} choice of the function $\eta_\eps: \bbR \mapsto \bbR^+$ in
detail below; heuristically it should be thought of as proportional to
a mollified Dirac mass, or a characteristic function of a small interval.
From $W$ we construct the graph Laplacian as follows.
We define the diagonal matrix $D={\rm diag}\{d_{ii}\}$ with entries
$d_{ii} = \sum_{j \in Z} w_{ij}.$ 
%If we assume that the graph $G$ is connected, then $d_{ii}>0$ for all nodes $i \in Z$. 
% Dejan: I removed the above line, since $\eta(0)>0$, d_{ii}>0 even if not connected
We can then define the
unnormalized graph Laplacian $L=D-W$. 
Our results may be generalized to the normalized graph 
Laplacian $L=I-D^{-\frac12}WD^{-\frac12}$ and we will comment on this in the
conclusions.

\subsection{Quadratic Form}
\label{ssec:Q}

We view $u: Z \mapsto \R$ as a vector
in $\R^n$ and define the quadratic form
%\begin{equation}
%\label{eq:unscaled}
\[  \langle u, Lu \rangle=\frac12\sum_{i,j \in Z} w_{ij}|u(i)-u(j)|^2; \]
%\end{equation} 
here $\langle \cdot, \cdot \rangle$ denotes the standard Euclidean
inner-product on $\R^n$.
This is the discrete Dirichlet energy defined via the graph Laplacian $L$ \mt{which}
appears as a basic quantity in many unsupervised and semi-supervised learning 
algorithms.  In this paper our interest focusses on forms based on powers of $L$:
%\begin{equation}
%\label{eq:Jn}
\[ J^{(\alpha,\tau)}_n(u) = \frac{1}{2n} \langle u, \disP^{(n)} u\rangle \]
%\end{equation}
where, for $\tau\geq 0$ and $\alpha>0$,
\begin{equation}
\label{eq:A}
\disP^{(n)}=(s_n L+\tau^2 I)^{\alpha}.
\end{equation}
The sequence parameters $s_n$ will be chosen appropriately to ensure that the quadratic form $J^{(\alpha,\tau)}_n(u)$ converges to a well-defined limit as $n \to \infty.$

In addition to working in a set-up which results in a well-defined limit, we 
will also ask that this limit results in a quadratic form defined by a differential
operator. This, of course, requires some form of localization  and we will
encode this as follows: we will assume that 
$\eta_{\eps}(\cdot)=\eps^{-d}\eta(\cdot/\eps)$, inducing
a Dirac mass approximation as $\eps \to 0$; later we will discuss how to
relate $\eps$ to $n$. For now we state the assumptions on $\eta$ that
we employ throughout the paper: 

\begin{assumptions}[on $\eta$]  
\label{a:eta}
The edge weight profile function $\eta$ satisfies:
\begin{itemize}
\addtolength{\itemsep}{4pt}
\addtolength{\itemindent}{24pt}
\item[(K1)] $\eta(0)>0$ and $\eta(\cdot)$ \mt{is} continuous at 0; %the origin;
\item[(K2)] $\eta$ \mt{is} non-increasing;
\item[(K3)] $\int_0^{\infty} \eta(r)r^{d+1}dr<\infty$;
%\item[(K4)] $\sigma_{\eta}= \frac{1}{d} \int_{\bbR^d} \eta(h)|h|^2 dh<\infty$.
% Dejan: K4 is not needed since it follows from (K3).
% Dejan: K5  is not needed as it follows from (K4)     \item[(K5)] $\beta_{\eta}=\int_{\bbR^d} \eta(h) dh<\infty$.
\end{itemize}
\end{assumptions}

\mt{
\begin{rem}
The prototypical example for $\eta$ is $\eta(t) = 1$ if $|t|<1$ and $\eta(t) = 0$ otherwise.
In this example the graph has edges between any two nodes closer than $\eps$; this is often referred to as the \textit{random geometric graph}.
Clearly this choice of $\eta$ satisfies Assumptions~\ref{a:eta}.
\end{rem}
}

Notice that assumption  (K3) implies that 
\begin{equation} \label{sig_bet}
 \sigma_{\eta}:= \frac{1}{d} \int_{\bbR^d} \eta(|h|)|h|^2 dh<\infty \quad \te{ and } \quad \beta_{\eta}:=\int_{\bbR^d} \eta(|h|) dh<\infty.
\end{equation}
A notable fact about the limits that we study in the remainder of the paper is that
they depend on $\eta$ only through the constants $\sigma_{\eta}, \beta_{\eta}$, provided
Assumptions \ref{a:eta} \mt{holds} and $\eps=\eps_n$ and $s_n$ are
chosen as appropriate functions of $n$.

\subsection{Limiting Quadratic Form}
\label{ssec:L}

\hspace{0.05in}

The limiting quadratic form is defined on an  open and bounded set  $\Omega \subset \bbR^d$.
\begin{assumptions}[on $\Omega$] \label{a:omega}
We assume that $\Omega$ is a connected, open and bounded subset of $\bbR^d$. We also assume that $\Omega$ has $C^{1,1}$boundary.
\footnote{The assumption that $\Omega$ is connected is not essential but makes stating the results simpler. We remark that a number of the results, and in particular the convergence of Theorem \ref{thm:LimitThmDir:LimitThmDir}, hold if we only assume that the boundary of $\Omega$ is Lipschitz. We need the stronger assumption in order to be able to employ elliptic regularity to characterize functions in fractional Sobolev spaces, see Section \ref{subsec:fsp} and Lemma \ref{lem:HiscH}; this is essential to be able to define Gaussian measures on function spaces,  and therefore needed to define a Bayesian approach in which uncertainty of 
classifiers may be estimated.}
\end{assumptions}

\begin{assumptions}[on density $\rho$] \label{a:rho}
 We assume that $n$ feature vectors $x_j \in \Omega$ are sampled i.i.d. from a probability measure $\mu$ supported on $\Omega$ with smooth Lebesgue density $\rho$ bounded above and below by
finite strictly positive constants $\rho^{\pm}$ uniformly on ${\overline \Omega}$.
\end{assumptions}
\medskip

We index the data by $Z=\{1, \cdots, n\}$ and let $\Omega_n = \{x_i\}_{i\in Z}$ be the data set.
This data set induces the empirical measure 
\[ \mu_n = \frac{1}{n}\sum_{i\in Z} \delta_{x_i}. \]
Given a measure $\nu$ on $\Omega$ we define the weighted Hilbert space
$L^2_{\nu}=L^2_{\nu}(\Omega;\R)$ with inner-product
\begin{equation}
\label{eq:IP}
\langle a,b \rangle_{\nu}=\int_{\Omega} a(x)b(x)\nu(dx)
\end{equation}
and \mt{the} induced norm defined by the identity $\|\cdot\|_{L^2_{\nu}}^2= \langle \cdot, \cdot \rangle_{\nu}.$ Note
that with these definitions we have 
%\begin{equation}
%\label{eq:Background:Discrete:Jn}
\[ J^{(\alpha,\tau)}_n : L^2_{\mu_n} \mapsto [0,+\infty), \quad \quad J^{(\alpha,\tau)}_n(u) = \frac{1}{2} \langle u, \disP^{(n)} u\rangle_{\mu_n}. \]
%\end{equation}
In what follows we apply a form of $\Gamma-$convergence to establish that
for large $n$ the quadratic form $J^{(\alpha,\tau)}_n$ is well approximated
by the limiting quadratic form
%\begin{equation}
%\label{eq:Background:Cont:Jinfty}
\[ J^{(\alpha,\tau)}_\infty: L^2_{\mu} \mapsto [0,+\infty)\cup\{+\infty\}, \quad \quad J^{(\alpha,
\tau)}_\infty(u) = \frac{1}{2} \langle u, \ctsP u\rangle_{\mu}. \]
%\end{equation}
Here $\mu$ is the measure on $\Omega$ with density $\rho$, and we
define the $L^2_{\mu}$ self-adjoint differential operator $\cL$ by
\begin{equation}
\label{eq:Background:Cont:cL}
\cL u = - \frac{1}{\rho} \nabla \cdot (\rho^2 \nabla u), \quad x \in \Omega,
\quad\quad\quad
 \frac{\partial u}{\partial n} =0,\quad x \in \partial \Omega. 
\end{equation}
The operator $\ctsP$ is then defined by $\ctsP = (\cL+\tau^2 I)^\alpha.$ 

We may now relate the quadratic forms defined by $A^{(n)}$ and $\ctsP$. The
$TL^2$ topology is introduced in \cite{trillos} and defined in the Appendix
section \ref{sssec:Background:Passage:TLp} for convenience.
The following theorem is proved in section~\ref{ssec:LQF}.

\begin{theorem}
\label{thm:LimitThmDir:LimitThmDir}
Let Assumptions \ref{a:eta}--\ref{a:rho}  hold. Let $\alpha>0$, $\{\eps_n\}_{n=1,2,\dots}$ be a positive sequence converging to zero, and such that
\begin{equation}  \label{eq:LimitThmDir:epsSca}
\begin{alignedat}{2}
\lim_{n \to \infty} \Bigl(\frac{\log n}{n}\Bigr)^{1/d} \frac{1}{\eps_n}& =0  \quad \quad  \text{if } d && \geq 3, \\ 
\lim_{n \to \infty} \Bigl(\frac{\log n}{n}\Bigr)^{1/2} \frac{(\log n)^{\frac14}}{\eps_n} \; & =0 \quad \quad  \text{if } d &&= 2, 
\end{alignedat}
\end{equation}
and assume that the scale factor $s_n$ is defined by
\begin{equation} \label{eq:sn}
s_n = \frac{2}{\sigma_\eta n \eps_n^2}.
\end{equation}
Then, with probability one, we have
\begin{enumerate}
\item $\Glim_{n\to\infty} J_n^{(\alpha,\tau)} = J_\infty^{(\alpha,\tau)}$ with respect to the $TL^2$ topology;
\item if $\tau=0$, any sequence $\{u_n\}$ with $u_n:\Omega_n\to \bbR$ satisfying $\sup_n \|u_n\|_{L^2_{\mu_n}} < \infty$ and $\sup_{n\in\bbN} J_n^{(\alpha,0)}(u_n)<\infty$ is pre-compact in the $TL^2$ topology;
\item if $\tau>0$, any sequence $\{u_n\}$ with $u_n:\Omega_n\to \bbR$ satisfying $\sup_{n\in\bbN} J_n^{(\alpha,\tau)}(u_n)<\infty$ is pre-compact in the $TL^2$ topology.
\end{enumerate}
\end{theorem}

\begin{rem}
As we discuss in section \ref{sssec:Background:Passage:GammaConv}
of the appendix, $\Gamma$-convergence and pre-compactness allow one to show that minimizers of a sequence of functionals converge to the minimizer of the limiting functional. The results of Theorem \ref{thm:LimitThmDir:LimitThmDir} provide the $\Gamma$-convergence and pre-compactness of
fractional Dirichlet energies, which are the key term of the functionals, 
such as \eqref{probit_nfun} below, that define the learning algorithms
that we study. In particular Theorem \ref{thm:LimitThmDir:LimitThmDir} enables
us to prove the convergence, in the large data limit $n \to \infty$, of 
minimizers of functionals such as \eqref{probit_nfun}  (i.e. of outcomes of 
learning algorithms), as shown in Theorem \ref{thm:LimitThmOpt:Probit:pr}.
\end{rem}

\subsection{Function Spaces} \label{subsec:fsp}

The operator $\cL$ given by (\ref{eq:Background:Cont:cL}) is uniformly elliptic as a consequence of the assumptions on $\rho$, and is self-adjoint with respect to the  inner product \eqref{eq:IP} on $L^2_{\mu}$. By standard theory, it has a discrete spectrum: $0=\lambda_1 < \lambda_2 \leq \cdots$, where the fact that $0 < \lambda_2$ uses the connectedness of the domain
and the uniform positivity of $\rho$ on the domain. Let $\varphi_i$ for $i=1,\dots$ be the associated $L^2_{\mu}$-orthonormal  eigenfunctions. They form a basis of $L^2_\mu$. 

By Weyl's law the eigenvalues of $\{\lambda_j\}_{j\geq 1}$ of $\cL$ satisfy 
$\lambda_j \asymp j^{2/d}.$
For completeness a simple proof is proved in Lemma \ref{lem:weyl}; 
the analogous  and more general results applicable to the Laplace-Beltrami operator 
may be found in, H\"ormander~\cite{Hormander}. %; see also \cite{SoggeZelditch}}.

\medskip
\textit{Spectrally defined Sobolev spaces.}
For $s \geq 0$ we define
%\begin{equation} \label{def:cH}
\[ \cH^s(\Omega) = \Big\{u \in L^2_\mu  \::\:  \sum_{k=1}^\infty \lambda_k^{s} a_k^2 < \infty \Big\}\mt{,} \]
%\end{equation}
where $a_k = \langle u , \varphi_k \rangle_\mu$ and thus 
$u = \sum_k a_k \varphi_k$ in $L^2_\mu.$ 
We note that $\cH^s(\Omega)$ is a Hilbert space with respect to the inner product
\[  \llangle u , v  \rrangle_{s,\mu} = a_1 b_1 + \sum_{k=2}^\infty \lambda_k^{s} a_k b_k  \]
where $b_k = \langle v , \varphi_k \rangle_\mu$.
It follows from the definition that for any $s \geq 0$, $\cH^s(\Omega)$ is isomorphic to a weighted $\ell^2(\mathbb N)$ space, where the weights are formed by the sequence 
$1, \lambda_2^s, \lambda_3^s, \dots$.

In Lemma \ref{lem:HiscH} in the Appendix section \ref{app:funsp} we show that for any integer $s>0$, $\cH^s(\Omega) \subset H^s(\Omega)$ where $H^s(\Omega)$ is the standard fractional Sobolev space. More precisely we characterize $\cH^s(\Omega)$ as the set of those functions in $H^s(\Omega)$ which satisfy the appropriate boundary condition and show that the norms of $\cH^s(\Omega)$ and $ H^s(\Omega)$  are equivalent on $\cH^s(\Omega)$. 

We also note that  for any integer $s$ and $\theta \in (0,1)$ the space $\cH^{s+\theta}$ is a interpolation space between $\cH^s$ and $\cH^{s+1}$. In particular $\cH^{s+\theta} = [\cH^s, \cH^{s+1}]_{\theta,2}$, where the real interpolation space used is as in Definition 3.3 of Abels \cite{Abels}. This  identification of $\cH^s$ follows from the characterization of interpolation  spaces of weighted $L^p$ spaces  by Peetre \cite{Peetre}, as referenced by Gilbert \cite{Gilbert}.
Together these facts allow us to characterize the H\"older regularity of functions in $\cH^s(\Omega)$.

\begin{lemma} \label{lem:emb_frac}
Under Assumptions \ref{a:omega}--\ref{a:rho}, 
for all $s \geq 0$ there exists a bounded, linear, extension mapping $E: \cH^s(\Omega) \to H^{s}(\R^d)$. That is for all $f \in \cH^s(\Omega)$, $E(f)|_\Omega = f$ a.e.
Furthermore:
\begin{itemize}
\item[(i)] if $s < \frac{d}{2}$ then $\cH^s(\Omega)$ embeds continuously in $L^q(\Omega)$ for any $q \leq \frac{2d}{d-2s}$;
\item[(ii)] if $s > \frac{d}{2}$ then $\cH^s(\Omega)$ embeds continuously in $C^{0, \gamma}(\Omega)$ for any $\gamma <  \min\{ 1,  s - \frac{d}{2}\}$.
\end{itemize}
\end{lemma}
The proof is presented in the Appendix \ref{app:funsp}.
\medskip

We note that  this implies that when $\alpha>\frac{d}{2}$ pointwise evaluation is well-defined in the limiting quadratic form $J^{(\alpha,\tau)}_\infty$;
this will be used in what follows to show that the 
the limiting labelling model obtained when $|Z'|$ is fixed is well-posed.

\subsection{Gaussian Measures of Function Spaces} \label{subsec:gau}

Using the ellipticity of $\cL$, Weyl's law, and Lemma \ref{lem:emb_frac} allows us to characterize the regularity of samples of Gaussian measures on $L^2_{\mu}$.
The proof of the following theorem is a straightforward application
of the techniques in \cite[Theorem 2.10]{S13} to obtain the Gaussian
measures on $\cH^{s}(\Omega)$. Concentration of the measure
on $H^s$ and on $C^{0, \gamma}(\Omega)$ then follows from
Lemma \ref{lem:emb_frac}.
When $\tau=0$ we work on the space orthogonal to constants in order that $\cC$ (defined in the theorem below) is well defined.

\begin{theorem}
\label{t:g}
Let Assumptions \ref{a:omega}--\ref{a:rho} hold.
Let $\cL$ be the operator defined in \eqref{eq:Background:Cont:cL}, and define $\cC=(\cL+\tau^2 I)^{-\alpha}.$
For any fixed $\alpha> \frac{d}{2}$ and
$\tau \ge 0$, the Gaussian measure $N\bigl(0,\cC\bigr)$
is well-defined on $L^2_{\mu}.$ Draws from this measure are almost
surely in $H^{s}(\Omega)$ 
for any  $s<\alpha-\frac{d}{2}$, and consequently in 
$C^{0, \gamma}(\Omega)$ for any  $\gamma <  \min\{ 1,  \alpha - d \}$ if 
$\alpha >d$. 
\end{theorem}
%The proof of this lemma follows from the discussion in \cite[Theorem 2.10]{S13} and Lemma \ref{lem:emb_frac}.

We note that if the operator $\cL$ has eigenvectors which are as regular as those of the Laplacian on a flat torus then the conclusions of Theorem \ref{t:g} can be strengthened. Namely if in addition to what we know about $\cL$,  there is $C>0$ such that
\begin{equation} \label{condcL}
\sup_{j\geq 1}\left(\|\varphi_j\|_{L^\infty} + \frac{1}{j^{1/d}}\mathrm{Lip}(\varphi_j)\right) \leq C,
\end{equation}
then the Kolmogorov continuity technique \cite[Section 7.2.5]{S13} can be used to show additional H\"{o}lder continuity.

%Using an operator $\cL$ satisfying these assumptions, we may construct Gaussian measures whose samples almost surely have a desired level of Sobolev and H\"older regularity \cite{S10}:
\begin{proposition}
\label{t:gold}
Let Assumptions \ref{a:omega}--\ref{a:rho} hold.
Assume the operator $\cL$ satisfies condition \eqref{condcL}  and define $\cC=(\cL+\tau^2 I)^{-\alpha}.$
For any fixed $\alpha>d/2$ and
$\tau \ge 0$, the Gaussian measure $N\bigl(0,\cC\bigr)$
is well-defined on $L^2_{\mu}.$  Draws from this measure are almost
surely in $H^{s}(\Omega;\R)$ for any $s < \alpha-d/2$, and in
$C^{0,\gamma}(\Omega;\R)$  for any  $\gamma <  \min\{ 1,  \alpha - \frac{d}{2} \}$ if $\alpha>\frac{d}{2}.$
\end{proposition}

We note that in general one cannot expect that the operator $\cL$ satisfies the bound \eqref{condcL}. For example, for the ball there is a sequence of  eigenfunctions which satisfy 
$ \|\varphi_k\|_{L^\infty} \sim \lambda_k^{(d-1)/4} \sim k^{(d-2)/(2d)}$, see \cite{Grieser02}.
In fact this is the largest growth of eigenfunctions possible, as on  general domains 
with smooth boundary $ \|\varphi_k\|_{L^\infty} \lesssim \lambda_k^{(d-1)/4}$, as follows from the work of Grieser,  \cite{Grieser02}. Analogous bounds have first been established for operators on manifolds without boundary by H\"ormander, \cite{Hormander}. This bound is rarely saturated as shown by Sogge and Zeldtich \cite{SoggeZelditch}, but determining the scaling for most sets and manifolds remains open. Establishing the conditions on $\Omega$ under which the Theorem \ref{t:g} can be strengthened as in Proposition \ref{t:gold} is of great interest.

\section{Graph Based Formulations}
\label{sec:G}

We now assume that we have access to label data defined as follows.
Let $\Omega' \subset \Omega$ and let $\Omega^{\pm}$ be two subsets of $\Omega'$ such that
\[ \Omega^+ \cup \Omega^- = \Omega', \quad \overline{\Omega^+} \cap \overline{\Omega^-}=\emptyset. \]
We will consider two labelling scenarios:

\begin{itemize}
\item 
\hypertarget{labmod1}{{\bf Labelling Model 1}}.  
$|Z'|/n \to \mathfrak{r} \in (0,\infty)$. 
We assume that $\Omega^{\pm}$
have positive Lebesgue measure. We assume that the $\{x_j\}_{j \in \bbN}$ are drawn i.i.d. from
measure $\mu$. 
Then if $x_j \in \Omega^+$ we set $y_j=1$ and if $x_j \in \Omega^-$ then $y_j=-1$.
The label variables $y_j$ are not defined if $x_j \in \Omega\backslash \Omega'$ where $\Omega'=\Omega^+\cup \Omega^-$.
We assume $\mathrm{dist}(\Omega^+,\Omega^-)>0$ and define $Z' \subset Z$ to be the subset of indices for which we have labels.

\hypertarget{labmod2}{{\bf Labelling Model 2}}.  
$|Z'|$ fixed as $n \to \infty.$ We assume that $\Omega^{\pm}$
comprise a fixed number of points, $n^{\pm}$ respectively. We assume that the $\{x_j\}_{j > n^++n^-}$ are drawn i.i.d. from
measure $\mu$ whilst $\{x_j\}_{1 \le j \le n^+}$ are a fixed set of points in $\Omega^+$
and  $\{x_j\}_{n^++1 \le j \le n^++n^-}$ are a fixed set of points in $\Omega^-.$ 
We label these fixed points by $y: \Omega^{\pm} \mapsto \{\pm 1\}$ 
as in \hyperlink{labmod1}{{\bf Labelling Model 1}}. We define $Z' \subset Z$ to be the subset of indices 
$\{1, \cdots, n^++n^-\}$ for which we have labels and $\Omega'= \Omega^+\cup \Omega^-$.
\end{itemize} 
In both cases $j \in Z'$ if and only if $x_j \in \Omega'$. But in Model 1
the $x_j$ are drawn i.i.d. and assigned labels when they lie in $\Omega'$, assumed
to have positive Lebesgue measure; in Model 2 the $\{(x_j,y_j)\}_{j \in Z'}$ 
are provided,  in a  possibly non-random way, independently of the \mt{unlabeled} data.

We will identify $u \in \R^n$ and $u \in L^2_{\mu_n}(\Omega;\R)$ by $u_j = u(x_j)$ for each $j \in Z$. Similarly, we will identify $y \in \R^{n^+ +n^-}$ and $y \in L^2_{\mu_n}(\Omega';\R)$ by $y_j = y(x_j)$ for each $j \in Z'$. We may therefore write, for example,
\[
\frac{1}{n}\la u,Lu\ra_{\R^n} = \la u,Lu\ra_{\mu_n}
\]
where $u$ is viewed as a vector on the left-hand side and a function on $Z$
on the right-hand side.

The algorithms that we study in this paper have interpretations
through both optimization and probability. The labels are found from a real-valued
function $u: Z \mapsto \R$ by setting $y=S\circ u: Z \mapsto \R$ with $S$ 
the sign function defined by
$$S(0)=0; \quad S(u)=1, \;u>0;\quad{\rm and}\quad S(u)=-1, \; u<0.$$
The objective
function of interest takes the form
$$\quad \Jk^{(n)}(u)=\frac12 \langle u, \disP^{(n)} u \rangle_{\mu_n} + \rn \Phi^{(n)}(u).$$
The quadratic form depends only on the \mt{unlabeled} data,
while the function $\Phi^{(n)}$ is determined by the \mt{labeled} data.
Choosing $\rn=\frac{1}{n}$ in \hyperlink{labmod1}{{\bf Labeling Model 1}} and $\rn=1$ in 
\hyperlink{labmod2}{{\bf Labeling Model 2}} ensures that the total 
labelling information remains of ${\mathcal O}(1)$ in the large $n$ limit.
Probability distributions constructed by exponentiating multiples of
$\Jk^{(n)}(u)$ will be of interest to us;
the  probability is then high where the objective function is small, 
and vice-versa.  Such probabilities represent the Bayesian posterior 
distribution on the conditional random variable $u|y$.

\subsection{Probit}
\label{ssec:PR}

The probit algorithm on a graph is defined in \cite{UQ17} and here
generalized to a quadratic form based on $\disP^{(n)}$ rather than $L$. We define
\begin{equation} \label{eq:Background:Discrete:Probit:Psi}
\Psi(v;\gamma)=\frac{1}{\sqrt{2\pi\gamma^2}}\int_{-\infty}^v \exp\big(-t^2/2\gamma^2\bigr) \, \dd t
\end{equation}
and then
\begin{equation} \label{eq:Background:Discrete:Probit:Phip}
\Phip^{(n)}(u;\gamma)=-\sum_{j \in Z'}\log \Bigl(\Psi(y_j u_j;\gamma)\Bigr).
\end{equation}
The function $\Psi$ and its logarithm are shown in Figure \ref{fig:psi} in the case $\gamma = 1$. The probit objective function is
%\begin{equation} \label{eq:Background:Discrete:Probit:Jp}
\begin{equation} \label{probit_nfun}
\Jp^{(n)}(u)= J_n^{(\alpha,\tau)}(u) + \rn \Phip^{(n)}(u;\gamma)\mt{,}
\end{equation}
%\end{equation}
where $\rn=\frac{1}{n}$ in \hyperlink{labmod1}{{\bf Labeling Model 1}} and $\rn=1$ in \hyperlink{labmod2}{{\bf Labeling Model 2}}. 
The proof of Proposition 1 in \cite{UQ17} is readily modified to
prove the following.

\begin{proposition}
\label{prop:G:Probit:Convex}
Let $\alpha>0$, $\tau\geq 0$, $\gamma>0$ and $r_n>0$.
Then $\Jp^{(n)}$, defined by~\textup{(\ref{eq:Background:Discrete:Probit:Psi}-\ref{probit_nfun})}, is strictly convex.
\end{proposition} 

It is also straightforward to check, by expanding $u$
in the basis given by eigenvectors of $A^{(n)}$, that $\Jp^{(n)}$ is coercive. 
This is proved by establishing that $J_n^{(\alpha,\tau)}$ is coercive on the orthogonal complement of the constant function. The coercivity in the remaining direction is provided by $\Phip^{(n)}(u;\gamma)$ using the fact that $\Omega^+$ and $\Omega^-$ are nonempty. Consequently  $\Jp^{(n)}$ has a unique minimizer; 
Lemma~\ref{lem:Limits:Probit:Unique} has the proof of the continuum analog of this;
the proof on a graph is easily reconstructed from this.

\begin{figure}
\centering
\includegraphics[width=\textwidth, trim=2cm 2cm 2cm 2cm]{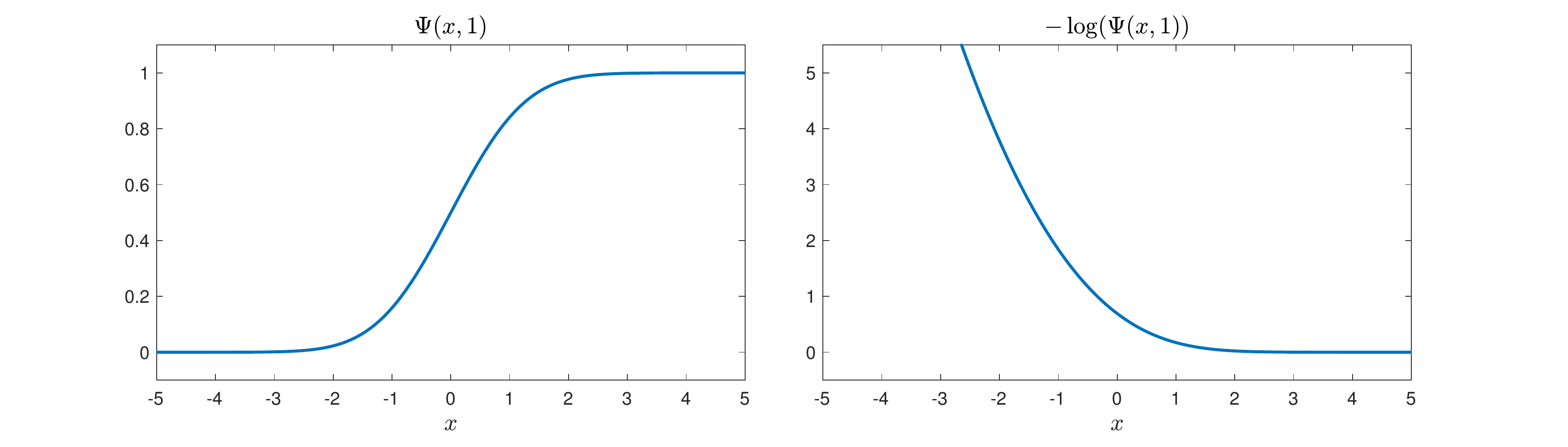}
\caption{The function $\Psi(\cdot;1)$, defined by (\ref{eq:Background:Discrete:Probit:Psi}), and its logarithm, which appears in the probit objective function.}
\label{fig:psi}
\end{figure}

The probabilistic analogue of the optimization problem
for $\Jp^{(n)}$ is as follows.
We let $\nu_0^{(n)}(\dd u;r)$ denote the centred Gaussian with covariance
$C=r_n(\disP^{(n)})^{-1}$ (with respect to the inner product $\langle\cdot,\cdot\rangle_{\mu_n}$). %(If $\tau=0$ we work on the space orthogonal 
%to constants so that the covariance $C$ is defined.)
We assume that the latent variable $u$ is 
a priori distributed according to measure $\nu_0^{(n)}(\dd u;r_n).$ 
If we then define the likelihood $y|u$ through the generative model
\begin{equation}
\label{eq:L1}
y_j=S\bigl(u_j+\noise_j\bigr)
\end{equation}
with $\noise_j \iid N(0,\gamma^2)$ then
the posterior probability on $u|y$ is given by
\begin{equation}
\label{eq:probPostn}
\nip^{(n)}(\dd u)=\frac{1}{\zp^{(n)}}e^{-\Phip^{(n)}(u;y)}\nu_0^{(n)}(\dd u;r_n)
\end{equation}
%where $\nu_0^{(n)}$ is the centred Gaussian with covariance $C$.
with $\zp^{(n)}$ the normalization to a probability measure.
The measure $\nip^{(n)}$ has Lebesgue density proportional to
$e^{-r_n^{-1}\Jp^{(n)}(u)}.$

\subsection{Bayesian Level Set}
\label{ssec:BLS}

We now define
\begin{equation}
\label{eq:blsb}
\Phil^{(n)}(u;\gamma)=\frac{1}{2\gamma^2} \sum_{j \in Z'} \bigl|y_j-S\bigl(u_j\bigr)\bigr|^2.
\end{equation}
The relevant objective function is  
%\begin{equation}
%\label{eq:blsa}
\[ \Jl^{(n)}(u)= J_n^{(\alpha,\tau)}(u)+r_n \Phil^{(n)}(u;\gamma)\mt{,} \]
%\end{equation}
where again $\rn=\frac{1}{n}$ 
in \hyperlink{labmod1}{{\bf Labeling Model 1}} and $\rn=1$ in \hyperlink{labmod2}{{\bf Labeling Model 2}}. We have the following:

\begin{proposition}
\label{prop:G:BLS:Infimum}
The infimum of of $\Jl^{(n)}$ is not attained.
\end{proposition}
This follows using the argument introduced in a related context in
\cite{iglesias2015bayesian}: assuming that a non-zero minimizer does exist 
leads to a contradiction upon multiplication of that minimizer
by any number less than one; and zero does not achieve the infimum.

We modify the generative model \eqref{eq:L1} slightly to read
%\begin{equation}
%\label{eq:L2}
\[ y_j=S\bigl(u_j\bigr)+\noise_j, \]
%\end{equation}
where now 
$\noise_j \iid N(0,r_n^{-1}\gamma^2)$.
In this case, because the noise is additive, multiplying the objective
function by $r_n$ simply results in a rescaling of the observational
noise; multiplication by $r_n$ does not have such a simple
interpretation in the case of probit.
As a consequence the resulting
Bayesian posterior distribution has significant
differences with the probit case:
the latent variable $u$ is now assumed a priori to be 
distributed according to measure $\nu_0^{(n)}(\dd u;1)$
Then
\begin{equation}
\label{eq:lsPostn}
\nil^{(n)}(\dd u)=\frac{1}{\zl^{(n)}}e^{-r_n\Phil^{(n)}(u;\gamma)}\nu_0^{(n)}(\dd u;1)
\end{equation}
where $\nu_0^{(n)}$ is the same centred Gaussian as in the probit case. 
Note that $\nil^{(n)}$ is also the measure with Lebesgue density proportional 
to $e^{-\Jl^{(n)}(u)}.$

\subsection{Small Noise Limit}
\label{ssec:SNL}

When the size of the noise on the labels is small, the probit and Bayesian level set approaches behave similarly. More precisely, the measures $\nip^{(n)}$ and $\nil^{(n)}$ share a common weak limit as $\gamma\to 0$. The following result is given without proof -- this is because its proof is almost identical to 
that arising in the continuum limit setting of Theorem \ref{thm:ltp:probit:ZeroNoise}(ii) given in the appendix; indeed it is technically easier 
due to the fully discrete setting.
Here $\Rightarrow$ denotes the weak convergence of probability measures.
%Furthermore we consider the case where the prior
%and $\nu_0^{(n)}(\dd u)$ denotes Gaussian measure of the 
%form $\nu_0^{(n)}(\dd u;r)$ for any $r$, possibly depending on $n$;
%this includes the specific choices of $r$ for probit and Bayesian 
%level set made above, but is not limited to those specific choices.

\begin{theorem}
\label{t:add}
Let $\nu_0^{(n)}(\dd u)$ denote a Gaussian measure of the form $\nu_0^{(n)}(\dd u;r)$ for any $r$, possibly depending on $n$. Define the set
$$B_{n}=\{u \in \R^n \,|\,y_ju_j >0 \;{\rm for\,each}\; j \in Z'\}$$
and the probability measure
$$\nu^{(n)}(\dd u)=\sZ^{-1}\one_{B_n}(u)\nu_0^{(n)}(\dd u)$$
 where $\sZ = \nu_0^{(n)}(B_n)$.
Consider the posterior measures $\nip^{(n)}$ defined in~\eqref{eq:probPostn} and $\nil^{(n)}$ defined in~\eqref{eq:lsPostn}.
Then $\nip^{(n)} \toweak \nu^{(n)}$ and $\nil^{(n)} \toweak \nu^{(n)}$ as $\gamma \to 0$.
\end{theorem}

\subsection{Kriging}
\label{ssec:KR}
Instead of classification, where the sign of the latent variable $u$ is made to agree with the labels, one can alternatively consider regression where $u$ itself is made to agree with the labels \cite{zhu2003semi,zhu2003semib}. We consider this situation numerically in section \ref{sec:N}. Here the objective is to
\[
\text{minimize}\,\, \Jkr^{(n)}(u): = J_n^{(\alpha,\tau)}(u)\text{ subject to } u(x_j) = y_j\text{ for all }j \in Z'.
\]
In the continuum setting this minimization is referred to as kriging, and
we extend the terminology to our graph based setting.
Kriging may also be defined in the case where the constraint is enforced
as a soft least squares penalty; however we do not discuss this here.

The probabilistic analogue of this problem can be linked with the original work of Zhu et al
\cite{zhu2003semi,zhu2003semib} which based classification on a centred
Gaussian measure  with inverse covariance given by the graph
Laplacian, conditioned to take the value exactly $1$ on
\mt{labeled} nodes where $y_j=1$, and to take the value exactly $-1$
on \mt{labeled} nodes where $y_j=-1.$

\section{Function Space Limits of Graph Based Formulations}
\label{sec:F}

In this section we state $\Gamma-$limit theorems for the objective functions
appearing in the probit algorithm. The proofs are
given in the appendix. They rely on arguments which use the fact that we study
perturbations of the $\Gamma-$limit theorem for the quadratic forms
stated in section \ref{sec:Q}. We also write down formal infinite dimensional
formulations of the probit and Bayesian level set posterior distributions, 
although we do not
prove that these limits are attained. We do, however, show that the probit
and level set posteriors have a common limit as $\gamma \to 0$, as they do on a
finite graph. 

\subsection{Probit}
\label{ssec:PR2}

Under \hyperlink{labmod1}{{\bf Labelling Model 1}}, the natural continuum limit of the probit 
objective functional is 
\begin{equation} \label{eq:Background:Cont:Probit:Jp}
\Jp(v) = J_\infty^{(\alpha,\tau)}(v) + \Phipo(v;\gamma)
\end{equation}
where
\begin{equation} \label{eq:Background:Cont:Probit:Phip}
\Phipo(v;\gamma) = -\int_{\Omega'} \log(\Psi(y(x)v(x);\gamma) )\, \dd \mu(x)
\end{equation}
for a given measurable function $y:\Omega'\to \{\pm 1\}$. For any $v\in L^2_{\mu}$, $ \log(\Psi(y(x)v(x);\gamma))$  is integrable by Corollary~\ref{lem:Background:Cont:Probit:PhipL1}.
The proof of the following theorem is given in the appendix, in section~\ref{ssec:PRLabelMod1}.

\begin{lemma}
\label{lem:Limits:Probit:Unique}
Let Assumptions \ref{a:eta}--\ref{a:rho} hold.
For $\alpha \geq 1$ and $\tau\geq 0$, consider the functional $\Jp$ with \hyperlink{labmod1}{{\bf Labelling Model 1}} defined by \eqref{eq:Background:Cont:Probit:Jp}. Then, the functional $\Jp$ has a unique minimizer in $\cH^\alpha(\Omega) $. 
%?? check that the power is $\alpha$
%Moreover the minimizer is smooth on $\Omega \setminus \Omega'$. }
\end{lemma}
\begin{proof}
Convexity of $\Jp$ follows from the proof of Proposition 1 in \cite{UQ17}. Let $\bar v_+$ and $\bar v_-$ be the averages of $v$ on $\Omega_+$  and $\Omega_-$ respectively. Namely let 
$\bar v_\pm = \frac{1}{|\Omega_\pm|} \int_{\Omega_\pm} v(x)\,\dd x$. 
Note that
%\[ \Jp(v) \geq J^{(\alpha,\tau)}_\infty(v) = \frac{1}{2} \sum_{k=1}^\infty \lambda_k^\alpha \langle v,q_k \rangle_{\mu}^2 \geq  \frac{1}{2} \sum_{k=1}^\infty \lambda_k^\alpha \langle v,q_k \rangle_{\mu}^2 =  \|\nabla v\|_{L^2(\Omega)}^2. \]
\[ \Jp(v) \geq J_\infty^{(\alpha,\tau)}(v) \geq \lambda_2^{\alpha-1} J_\infty^{(1,0)}(v) = -\frac12 \lambda_2^{\alpha-1} \int_\Omega v \nabla \cdot (\rho^2 \nabla v) \, \dd x \geq \frac{(\rho^-)^2 \lambda_2^{\alpha-1}}{2} \|\nabla v\|_{L^2(\Omega)}^2. \]
Using the form of Poincar\'e inequality given in Theorem 13.27 of \cite{Leoni} %2nd edition
 implies that 
 \begin{equation} \label{temp1}
\Jp(v) \gtrsim \|\nabla v\|_{L^2(\Omega)}^2 \gtrsim \int_{\Omega} |v - \bar v_+|^2 + |v - \bar v_-|^2 \, \dd x. 
\end{equation}
The convexity of $\Phipo(v;\gamma)$ implies that 
\[ \Phipo(v;\gamma) \geq - \log(\Psi(\bar v_+);\gamma) \mu(\Omega_+) - \log(\Psi(-\bar v_-);\gamma) \mu(\Omega_-) \]
Using that $\lim_{s \to -\infty} -\log(\Psi(s;\gamma)) = \infty$ we see that a bound on $\Phipo(v;\gamma) $ provides a lower bound on $\bar v_+$ and an upper bound on $\bar v_-$. 
To see this let $\Theta$ be the inverse of $s \mapsto -\log(\Psi(s;\gamma))$. 
The preceding shows that 
\[ \bar v_+ \geq \Theta \left( \frac{\Phipo(v;\gamma) }{ \mu(\Omega_+)} \right) \geq \Theta \left( \frac{\Jp(v)}{ \mu(\Omega_+)} \right) 
\quad \te{ and } \quad 
 \bar v_- \leq - \Theta \left( \frac{\Phipo(v;\gamma) }{ \mu(\Omega_-)} \right) \leq - \Theta \left( \frac{\Jp(v)}{ \mu(\Omega_-)} \right) . \]
Let $c = \max \left\{- \Theta \left( \frac{\Jp(v)}{ \mu(\Omega_+)}\right),  - \Theta \left( \frac{\Jp(v)}{ \mu(\Omega_-)} \right), 0 \right\} $.  Then $\bar v_+ \geq -c$ and $\bar v_- \leq c$. Using that, for any $a \in \R$,
$v^2 \leq 2|v-a|^2 + 2a^2$, we obtain
\begin{align*}
\int_\Omega v^2(x) \,\dd x & \leq \int_{\{v(x) \leq -c\}} v^2(x) \, \dd x +  \int_{\{v(x) \geq c\}} v^2(x) \, \dd x + c^2 |\Omega| \\
& \leq 2 \int_{\{v(x) \leq -c\}} |v+ c|^2 +c^2 \, \dd x + 2 \int_{\{v(x) \geq c\}} |v- c|^2 +c^2 \, \dd x   + c^2 |\Omega| \\
& \leq 5c^2 |\Omega| + 2 \int_{\{v(x) \leq -c\}} |v - \bar v_+|^2 \, \dd x +  2 \int_{\{v(x) \geq c\}} |v- \bar v_-|^2 \, \dd x \\
& \lesssim c^2|\Omega| +  \Jp(v).
\end{align*}
Then $\|v \|_{L^2}$ is bounded by a function of $\Jp(v)$ and $\Omega$.
 
Combining with \eqref{temp1} implies that a function of $\Jp(v)$ bounds $\| v \|^2_{\cH^\alpha(\Omega)}$ which establishes the  coercivity of $\Jp$. 
The functional $\Jp$ is weakly lower-\mt{semicontinuous} in $\cH^\alpha$, due to \mt{the} convexity of both
$J_\infty^{(\alpha,\tau)}$ and $\Phipo$. 
Thus the  direct method of the calculus of variations proves  that $\Jp$ has a unique minimizer in $\cH^\alpha(\Omega)$.
\end{proof}

The following theorem is proved in section \ref{ssec:PRLabelMod1}.
\begin{theorem}
\label{thm:LimitThmOpt:Probit:pr}
Let the assumptions of \hyperlink{labmod1}{{\bf Labelling Model 1}}
and Theorem~\ref{thm:LimitThmDir:LimitThmDir} hold with $\tau\geq 0$.
Then, with probability one, any sequence of minimizers $v_n$ of $\Jp^{(n)}$ converge in $TL^2$ to $v_\infty$, the unique minimizer of $\Jp$ in $L^2_{\mu}$, and furthermore $\lim_{n\to \infty} \Jp^{(n)}(v_n) = \Jp(v_\infty) = \min_{v\in  L^2_{\mu}} \Jp(v)$.
\end{theorem}

%We have not been able to prove the analogous result under {\bf Labelling Model 2}.
The analogous result under \hyperlink{labmod2}{{\bf Labelling Model 2}}, i.e. convergence of minimizers, is an open question.
In this case the natural continuum limit of the probit 
objective functional is
\begin{equation} \label{eq:Background:Cont:Probit:Jp2}
\Jp(v) = J_\infty^{(\alpha,\tau)}(v) + \Phipt(v;\gamma)
\end{equation}
where
\begin{equation} \label{eq:Background:Cont:Probit:Phip2}
\Phipt(v;\gamma) = -\sum_{j \in Z'} \log(\Psi(y(x_j)u(x_j);\gamma) \, 
\end{equation}
for a given measurable function $y:\Omega'\to \{\pm 1\}$.
%We have a similar convergence result to the first labelling model which we give in the theorem below, the proof can be found in Section~\ref{ssec:PRLabelMod2}.
When $\alpha \leq \frac{d}{2}$ this limiting model is not well-posed. In particular the regularity of the functional is not sufficient to impose pointwise data. More precisely, when $\alpha \leq \frac{d}{2}$ then  there exists a sequence of smooth functions $v_k \in C^\infty(\Omega)$ such that $\lim_{k \to \infty} \Jp(v_k) = 0$. In particular when $\alpha < \frac{d}{2}$, consider a smooth, compactly supported,  mollifier $\xii$, with $\xii(0)>0$ and define $v_k(x) = c_k \sum_{i=1}^N y(x_i) \xii_{1/k} (x - x_i)$ where $c_k \to \infty$ sufficiently slowly. Then $\Phipt(v_k;\gamma) \to 0$ as $k \to \infty$ and, by a simple scaling argument (for appropriate $c_k$),
$J_\infty^{(\alpha,\tau)} (v_k) \to 0$  as $k \to \infty$. 
Another way to see that the problem is not well defined is that the functions in 
$\cH^{\alpha}(\Omega)$ (which is the natural space to consider $\Jp$ on) are not continuous in general and evaluating $\Phipt(v;\gamma) $ is not well defined. 

When $\alpha > \frac{d}{2}$  the existence of minimizers of \eqref{eq:Background:Cont:Probit:Jp2} 
in $\cH^\alpha(\Omega)$ is established by the direct method of the calculus of variations using the convexity of $\Jp$ and the fact that, by Lemma \ref{lem:emb_frac}, $\cH^\alpha$ 
 continuously embeds into a set of H\"older continuous functions.

For $\alpha > \frac{d}{2}$ we believe that  the minimizers of $\Jp^{n}$ of  \hyperlink{labmod2}{{\bf Labelling Model 2}} converge to minimizers of \eqref{eq:Background:Cont:Probit:Jp2} in an appropriate regime, but the situation is more complicated than for \hyperlink{labmod1}{{\bf Labelling Model 1}}: under \hyperlink{labmod2}{{\bf Labelling Model 2}} 
\eqref{eq:LimitThmDir:epsSca} is no longer a sufficient condition on the scaling of $\eps$ with $n$ for the convergence to hold. Thus if $\eps \to 0$ too slowly the problem degenerates. In particular 
in the following theorem we identify  the asymptotic behavior of minimizers of $\Jp$ both when 
$\alpha < \frac{d}{2}$, and if $\alpha > \frac{d}{2}$ but $\eps \to 0$ too slowly. 

The proof of the following may be found in section \ref{ssec:PRLabelMod2}.
The theorem is similar in spirit to Proposition 2.2(ii) in \cite{SlepcevThorpe} where a similar phenomenon was discussed for the $p$-Laplacian regularized semi-supervised learning. 
We also mention that the PDE approach to a closely related $p$-Laplacian problem was recently introduced by Calder~\cite{calder17game}.

\begin{theorem}
\label{thm:LimitThmOpt:Probit:pr2neg}
Let the assumptions of \hyperlink{labmod2}{{\bf Labelling Model 2}}, and Theorem~\ref{thm:LimitThmDir:LimitThmDir}  hold.  If  $\alpha>\frac{d}{2}$, $\tau > 0$, and 
\begin{equation} \label{epsn_spike}
 \eps_n n^{\frac{1}{2\alpha}} \to \infty \qquad \te{ as } n \to \infty 
\end{equation}
or if $\alpha<\frac{d}{2}$ then, with probability one, the sequence of minimizers $v_n$ of $\Jp^{(n)}$ converge to $0$  in $TL^2$ as $n \to \infty$. That is, the minimizers of $\Jp^{(n)}$ converge to the minimizer of $J_\infty^{(\alpha,\tau)}$ with the information about the labels being lost in the limit. 
\end{theorem}

\begin{rem}
\label{rem:LimitThmOpt:Probit:pr2neg}
We believe, but do not have a proof, that 
%a positive convergence result holds 
%under the assumptions of {\bf Labelling Model 2} and Theorem~\ref{thm:LimitThmDir:LimitThmDir}.  
for $\alpha>\frac{d}{2}$ and $\tau>0$, if 
\[  \eps_n n^{\frac{1}{2\alpha}} \to 0 \qquad \te{ as } n \to \infty \]
then, with probability one, any sequence of minimizers $v_n$ of $\Jp^{(n)}$ is sequentially compact in $TL^2$ with $\lim_{n\to \infty} \Jp^{(n)}(v_n) = \min_{v\in L^2_{\mu}} \Jp(v)$ given by  
\eqref{eq:Background:Cont:Probit:Jp2}, \eqref{eq:Background:Cont:Probit:Phip2}. 
If this holds then, under \hyperlink{labmod2}{{\bf Labelling Model 2}}, 
$\Jp^{(n)}(u)$ converges
in an appropriate sense to a limiting objective function $\Jp(u)$. 
Our numerical results support this conjecture.

It is also of interest to consider the
limiting probability distributions which arise under the
two labelling models. Under \hyperlink{labmod2}{{\bf Labelling Model 2}} this density has,
in physicist's notation, ``Lebesgue density'' $\exp\bigl(-\Jp(u)\bigr).$ 
Under \hyperlink{labmod1}{{\bf Labelling Model 1}}, however,
we have shown that $\Jp^{(n)}(u)$ converges
in an appropriate sense to a limiting objective function $\Jp(u)$
implying  that (again in physicist's notation) 
$\exp\bigl(-r_n^{-1}\Jp^{(n)}(u)\bigr) \approx \exp\bigl(- n\Jp(u)\bigr)$.
Thus under \hyperlink{labmod1}{{\bf Labelling Model 1}} the posterior probability concentrates on a Dirac measure at the minimizer of $\Jp(u)$.
\end{rem}

Based on this remark, the natural continuum probability limit concerns 
\hyperlink{labmod2}{{\bf Labelling Model 2}}.
The posterior probability is then given by
\begin{equation}
\label{eq:probprob}
\nipt(\dd u)=\frac{1}{Z_{\rm p,2}}e^{-\Phipt(u;\gamma)}\nu_0(\dd u)
\end{equation}
where $\nu_0$ is the centred Gaussian with covariance $\cC$ given in
Theorem \ref{t:g} and $\Phipt$ is given by
\eqref{eq:Background:Cont:Probit:Phip2}.
Since we require pointwise evaluation to make
sense of $\Phipt(u;\gamma)$ we, in general, require
$\alpha>d$; however Proposition \ref{t:gold} gives conditions
under which $\alpha>\frac{d}{2}$ will suffice. 
We will also consider
the probability measure $\nipo$ defined by
\begin{equation} \label{def_nipo}
\nipo(\dd u) = \frac{1}{Z_{\rm p,1}}e^{-\Phipo(u;\gamma)}\nu_0(\dd u)
\end{equation}
where $\Phipo$ is given by \eqref{eq:Background:Cont:Probit:Phip}.
The function $\Phipo(u;\gamma)$ is defined in an $L^2_{\mu}$ sense
and thus we require only $\alpha>\frac{d}{2}$ -- see Theorem \ref{t:g}.
Note, however, that this is not the limiting probability distribution
that we expect for \hyperlink{labmod1}{{\bf Labelling Model 1}} with the parameter choices
leading to Theorem \ref{thm:LimitThmOpt:Probit:pr} since the argument
above suggests that this will concentrate on a Dirac. However we include
the measure $\nipo$ in our discussions because, as we will show,
it coincides with the analogous Bayesian level set measure $\nilo$ 
(defined below) in the small observational noise limit.
Since $\nilo$ can be obtained by a natural scaling of the graph algorithm,
which does not concentrate on Dirac, the relationship between $\nipo$
and $\nilo$ is of interest as they are both, for small noise, relaxations
of the same limiting object.

\subsection{Bayesian Level Set}
\label{ssec:BLS2}

We now study probabilistic analogues of the Bayesian level set method,
again using the measure $\nu_0$ which is the centred Gaussian with covariance $\cC$ given in
Theorem \ref{t:g} for some $\alpha>\frac{d}{2}$.
Note that, from equation \eqref{eq:blsb}, for \hyperlink{labmod1}{{\bf Labelling Model 1}},
\begin{align*}
r_n \Phil^{(n)}(u;\gamma)&=\frac{1}{2\gamma^2} \frac{1}{n} \sum_{j \in Z'} \bigl|y(x_j)-S\bigl(u(x_j)\bigr)\bigr|^2\\
%&=\frac{1}{n}\frac{1}{2\gamma^2} \sum_{j \in Z'} \bigl|y(x_j)-S\bigl(u(x_j)\bigr)\bigr|^2\\
&\approx \int_{\Omega'} \frac{1}{2\gamma^2}\bigl|y(x)-S\bigl(u(x)\bigr)\bigr|^2 \, \dd \mu(x)\\
%&=\int_{\Omega'} \frac{1}{2\gamma^2}\bigl|y(x)-S\bigl(u(x)\bigr)\bigr|^2 \, \dd \mu(x)\\
&:= \Philo(u;\gamma)
\end{align*}
by a law of large numbers type argument of the type underlying the proof of
Theorem \ref{thm:LimitThmOpt:Probit:pr}. 

Recall that, from the discussion following Proposition \ref{prop:G:BLS:Infimum},
this scaling corresponds to employing the finite dimensional Bayesian
level set model with observational variance 
$\gamma^2 n$ so that the variance per observation is constant.
Then the natural limiting probability measure is, in physicists notation,
$\exp\bigl(-\Jl(u)\bigr)$ where
\[ \Jl(u) = J_\infty^{(\alpha,\tau)}(u) + \Philo(u;\gamma). \]
Expressed in terms of densities with respect to the Gaussian prior this gives
\begin{equation}
\label{eq:probB}
\nilo(\dd u)=\frac{1}{Z_{\rm ls,1}} e^{-\Philo(u;\gamma)}\nu_0(\dd u).
\end{equation}
Since $\Philo(u;\gamma)$ makes sense in $L^2_\mu$ we require 
only $\alpha>\frac{d}{2}$. 
The measure $\nilo$ is the natural analogue of the finite dimensional
measure $\nil^{(n)}$ under this label model.
Under \hyperlink{labmod2}{{\bf Labelling Model 2}} we take $r_n=1$. % no longer scale $\gamma$ with $n$ but rather fix it.
We obtain a measure $\nilt$ in the form \eqref{eq:probB} found by
replacing $\nilo$ by $\nilt$ and $\Philo$ by 
\begin{equation}
\label{eq:probBc}
\Philt(u;\gamma):=
\sum_{j \in Z'} \frac{1}{2\gamma^2}\bigl|y(x_j)-S\bigl(u(x_j)\bigr)\bigr|^2. 
\end{equation}
In this case the observational variance is not-rescaled by $n$ since
the total number of labels is fixed.
Since we require pointwise evaluation to make
sense of $\Philt(u;\gamma)$ we, in general, require 
$\alpha>d$; however Proposition \ref{t:gold} gives conditions 
under which $\alpha>\frac{d}{2}$ will suffice.

\begin{rem}
Note that $\Jl^{(n)}$ and $\Jl$
cannot be connected via $\Gamma$-convergence.
Indeed, if $\Jl = \Glim_{n\to \infty} \Jl^{(n)}$ then $\Jl$ would be lower semi-continuous~\cite{braides02}.
When $\tau>0$ compactness of minimizers follows directly from the compactness property of the quadratic forms $J_n^{(\alpha,\tau)}$, see Theorem~\ref{thm:LimitThmDir:LimitThmDir}.
Now since compactness of minimizers plus lower semi-continuity implies existence of minimizers then the above reasoning implies there exists minimizers of $\Jl$.
But as in the discrete case, Proposition \ref{prop:G:BLS:Infimum}, multiplying any $u$ by a constant less than one leads to a smaller value of $\Jl$.
Hence the infimum cannot be achieved.
It follows that $\Jl \neq \Glim_{n\to \infty} \Jl^{(n)}$.
\end{rem}

\subsection{Small Noise Limit}
\label{ssec:SNL2}

As for the finite graph problems, the labeled data can be viewed as arising from different generative models. 
In the probit formulation, the generative models for the labels are given by
\begin{align*}
y(x) &= S(u(x) + \noise(x)),\quad \noise \sim N(0,\gamma^2I),\\
y(x_j) &= S(u(x_j) + \noise_j),\quad \noise_j \iid N(0,\gamma^2)\mt{,}
\end{align*}
for \hyperlink{labmod1}{{\bf Labelling Model 1}}, \hyperlink{labmod2}{{\bf Labelling Model 2}} respectively; $S$ is the sign function. The functionals $\Phipo$, $\Phipt$ then arise as the negative log-likelihoods from these models. Similarly, in the Bayesian level set formulation the generative models are given by
\begin{align*}
y(x) &= S(u(x)) + \noise(x),\quad \noise\sim N(0,\gamma^2I),\\
y(x_j) &= S(u(x_j)) + \noise_j,\quad \noise_j \iid N(0,\gamma^2).
\end{align*}
leading to the functionals $\Philo$, $\Philt$.

We show that in the zero noise limit the Bayesian level set and probit posterior distributions coincide. However for $\gamma > 0$ they differ: note, for example, that the probit model enforces binary data, whereas the Bayesian level set model does not. It has been observed that the Bayesian level set posterior can be used to produce similar quality classification to the Ginzburg-Landau posterior, at significantly lower computational cost \cite{barcode}.
The small noise limit is important for two reasons: firstly in many 
applications labelling is very accurate and considering the
zero noise limit is therefore instructive; secondly recent work
\cite{Luo18} shows that the zero noise limit provides useful information
about the efficiency of algorithms applied to sample the posterior
distribution and, in particular, constants derived from the
zero noise limit appear in lower bounds on average acceptance
probability and mean square jump in such algorithms.

Proof of the following is given in section~\ref{ssec:SNL3}.

\begin{theorem}
\label{thm:ltp:probit:ZeroNoise}
\leavevmode
\begin{enumerate}
\item[(i)] Let Assumptions \ref{a:omega}--\ref{a:rho} hold, and assume that $\alpha > d$. Let the assumptions of \hyperlink{labmod1}{{\bf Labelling Model 1}} hold.
Define the set
\[
B_{\infty,1} = \{u \in C(\Omega;\bbR)\,|\,y(x)u(x)>0 \;{\rm for\,a.e.}\; x \in \Omega'\}
\]
and the probability measure
$$\nu_1(\dd u)=\sZ^{-1}\one_{B_{\infty,1}}(u)\nu_0(\dd u)$$
where $\sZ = \nu_0(B_{\infty,1})$.
Consider the posterior measures $\nipo$ defined in \eqref{def_nipo} and 
$\nilo$ defined in \eqref{eq:probB}.
Then $\nipo \toweak \nu_1$ and $\nilo \toweak \nu_1$ as $\gamma \to 0$.

\item[(ii)] Let Assumptions \ref{a:omega}--\ref{a:rho} hold, and assume that $\alpha > d$. 
Let the assumptions of \hyperlink{labmod2}{{\bf Labelling Model 2}} hold. Define the set
\[
B_{\infty,2} = \{u \in C(\Omega;\bbR)\,|\,y(x_j)u(x_j) >0 \;{\rm for\,each}\; j \in Z'\}
\]
and the probability measure
$$\nu_2(\dd u)=\sZ^{-1}\one_{B_{\infty,2}}(u)\nu_0(\dd u)$$
 where $\sZ = \nu_0(B_{\infty,2})$.
Then $\nipt \toweak \nu_2$ and $\nilt \toweak \nu_2$ as $\gamma \to 0$.
\end{enumerate}
\end{theorem}

\begin{rem}
\label{rem:zeronoise_alpha}
The assumption that $\alpha > d$ in both parts of the above theorem can be relaxed to $\alpha > d/2$ if the conclusions of Proposition \ref{t:gold} are satisfied.
\end{rem}

\subsection{Kriging}
\label{ssec:KR2}
One can define kriging in the continuum setting \cite{wahba1990spline} analogously to the discrete setting; we consider this numerically in section \ref{sec:N}. In the case of \hyperlink{labmod2}{{\bf Labelling Model 2}}, the limiting
problem is to 
\[
\text{minimize}\,\,\Jkr(u) := J_\infty^{(\alpha,\tau)}(u)\text{ subject to } u(x_j) = y_j\text{ for all }j \in Z'.
\]
Kriging may also be defined for \hyperlink{labmod1}{{\bf Labelling Model 1}} and without the hard constraint in the continuum setting, but we do not discuss either of these
scenarios here.

\section{Numerical Illustrations}
\label{sec:N}

In this section we describe the results of numerical experiments
which illustrate or extend the developments in the preceding
sections.  In section \ref{ssec:51} we study the effect of the geometry
of the data on the classification problem, by studying an illustrative
example in dimension $d=2.$ Section \ref{ssec:Epsilon} studies how
the relationship between the length-scale $\epsilon$ and the graph size
$n$ affects limiting behaviour. In section \ref{ssec:Extrap} we
study graph based kriging. Finally, in section \ref{ssec:MCMC}, we study
continuum problems from the Bayesian perspective,   
studying the quantification of
uncertainty in the resulting classification.

\subsection{Effect of Data Geometry on Classification}
\label{ssec:51}
We study how the geometry of the data affects the classification under \hyperlink{labmod1}{{\bf Labelling Model 1}}, using the continuum probit model. Let $\Omega = (0,1)^2$. We first consider a uniform distribution $\rho$ on the domain, and choose $\Omega_+,\Omega_-$ to be balls of radius 0.05  centred at (0.25,0.25), (0.75,0.75) respectively. The decision boundary is then naturally the perpendicular bisector of the line segment joining the centers of these balls. We then modify $\rho$ by introducing a channel of increasing depth in $\rho$ dividing the domain in two vertically, and look at how this affects the decision boundary. Specifically, given $h \in [0,1]$ we define $\rho_h$ to be constant in the $y$-direction, and assume the cross-sections in the $x$-direction are as shown in Figure \ref{fig:cross}, so that the channel has depth $1-h$. In order to numerically estimate the continuum probit minimizers, we construct a finite-difference approximation to each $\cL$ on a uniform grid of 65536 points, which then provides an approximation to $\cA$. The objective function $\Jp^{(\infty)}$ is then minimized numerically using the linearly-implicit gradient flow method described in \cite{UQ17}, Algorithm 4.

We consider both the effect of the channel depth parameter $h$ and the parameter $\alpha$ on the classification; we fix $\tau = 10$ and $\gamma = 0.01$. In Figure \ref{fig:channel} we show the minimizers arising from $5$ different choices of $h$ and $\alpha = 1,2,3$. As the depth of the channel is increased, the minimizers begin to develop a jump along the channel.  As $\alpha$ is increased, the minimizers become less localized around the \mt{labeled} regions, and the jump along the channel becomes sharper as a result. Note that the scale of the minimizers decreases as $\alpha$ increases. This could formally be understood from a probabilistic point of view: under the prior we have $\mathbb{E}\|u\|^2_{L^2} = \mathrm{Tr}(\cA^{-1}) \asymp \tau^{-2\alpha}$, and so a similar scaling may be expected to hold for the MAP estimators. In Figure \ref{fig:channel_sign} we show the sign of each minimizer in Figure \ref{fig:channel} to illustrate the resulting classifications. As the depth of the channel is increased, the decision boundary moves continuously from the diagonal to the vertical bisector of the domain, with the transitional boundaries appearing almost as a piecewise linear combination of both boundaries. We also see that, despite the minimizers themselves differing significantly for different $\alpha$, the classifications are almost invariant with respect to $\alpha$.

\label{ssec:Channel}
\begin{figure}
\centering
\includegraphics[width=\textwidth,trim=3cm 2cm 3cm 2cm,clip]{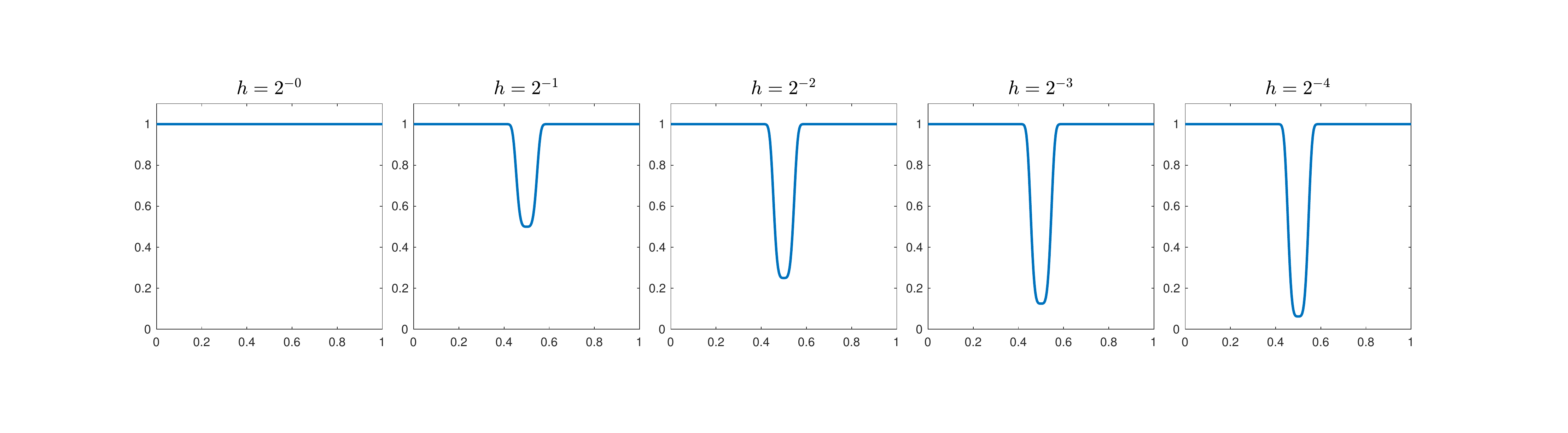}
\caption{The cross sections of the data densities $\rho_h$ we consider in subsection \ref{ssec:Channel}.}
\label{fig:cross}
\end{figure}

\begin{figure}
\centering
\includegraphics[width=\textwidth,trim=4cm 1.5cm 1cm 0cm,clip]{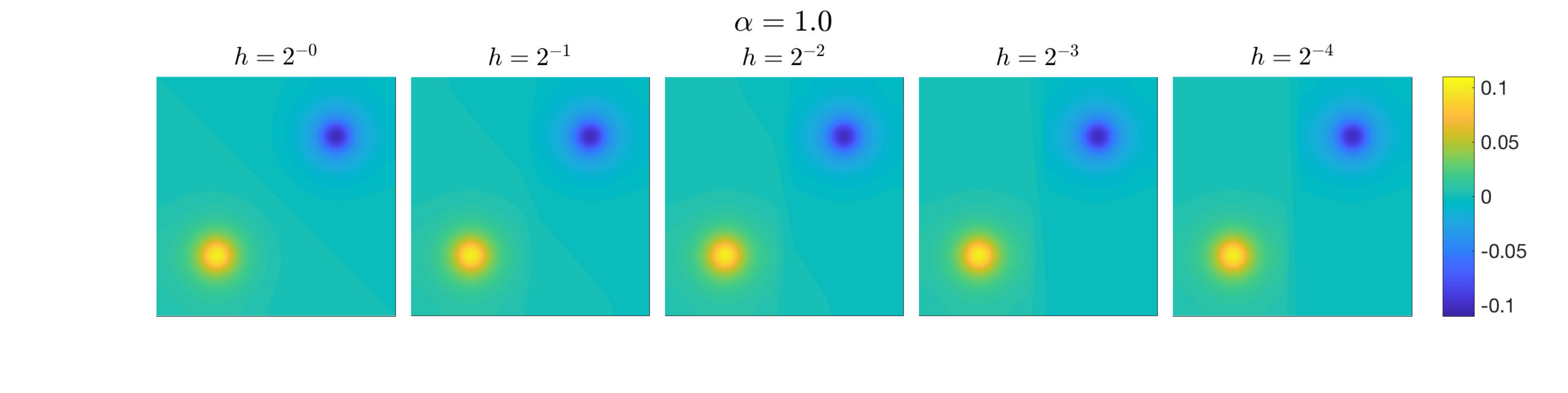}
\includegraphics[width=\textwidth,trim=4cm 1.5cm 1cm 0cm,clip]{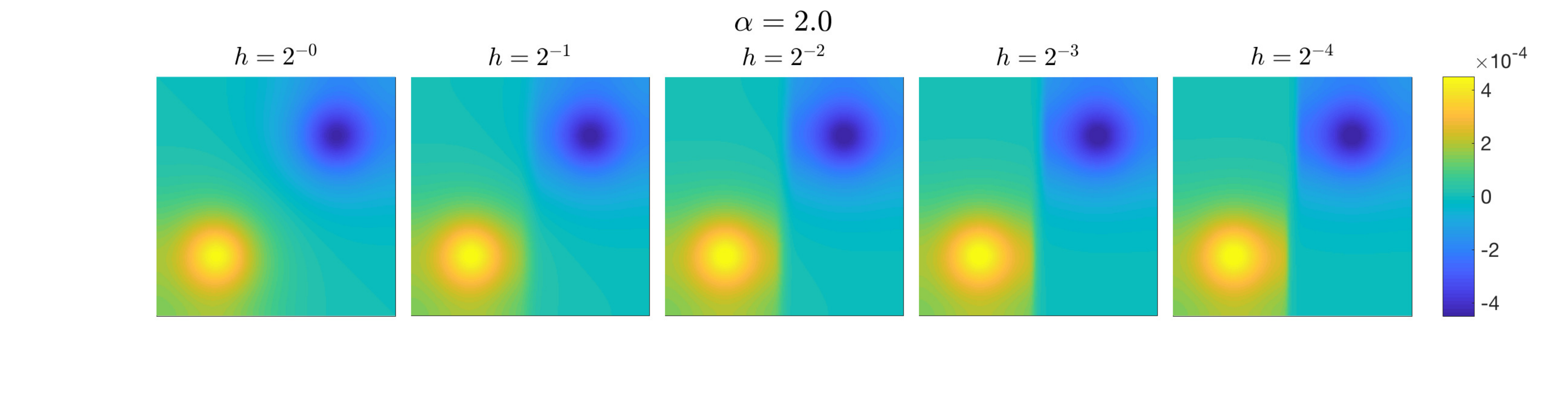}
\includegraphics[width=\textwidth,trim=4cm 2cm 1cm 0cm,clip]{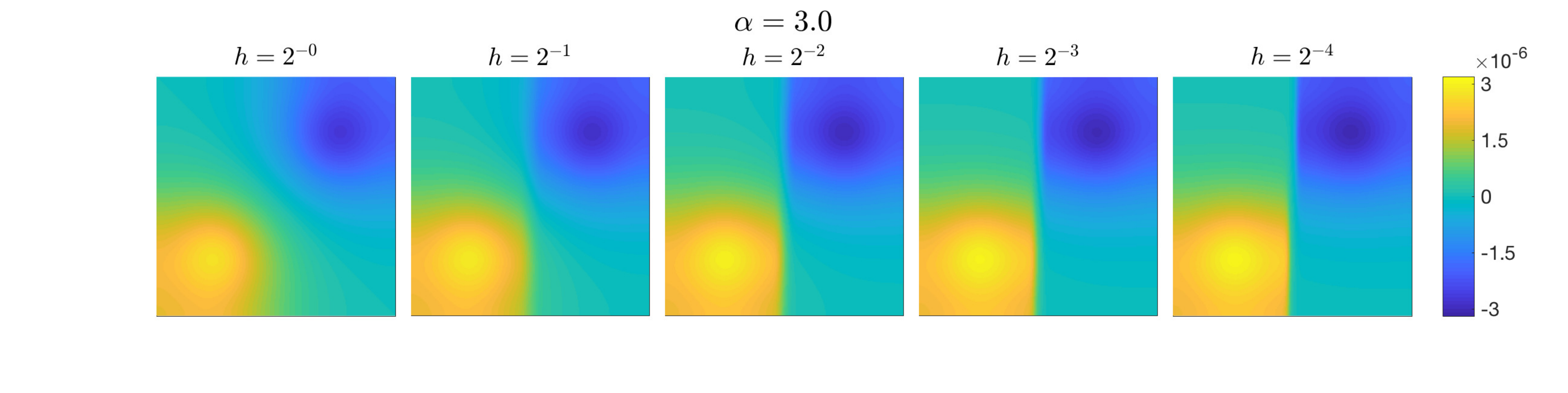}
\caption{The minimizers of the functional $\Jp^{(\infty)}$ for different values of $h$ and $\alpha$, as described in subsection \ref{ssec:Channel}.}
\label{fig:channel}
\end{figure}

\begin{figure}
\centering
\includegraphics[width=\textwidth,trim=4cm 1.5cm 1cm 0cm,clip]{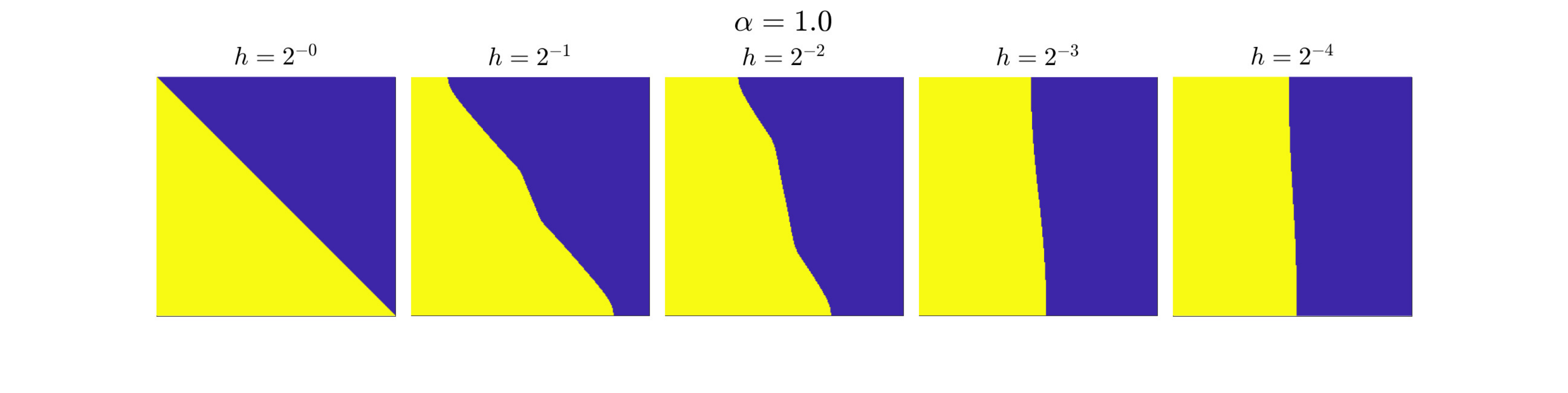}
\includegraphics[width=\textwidth,trim=4cm 1.5cm 1cm 0cm,clip]{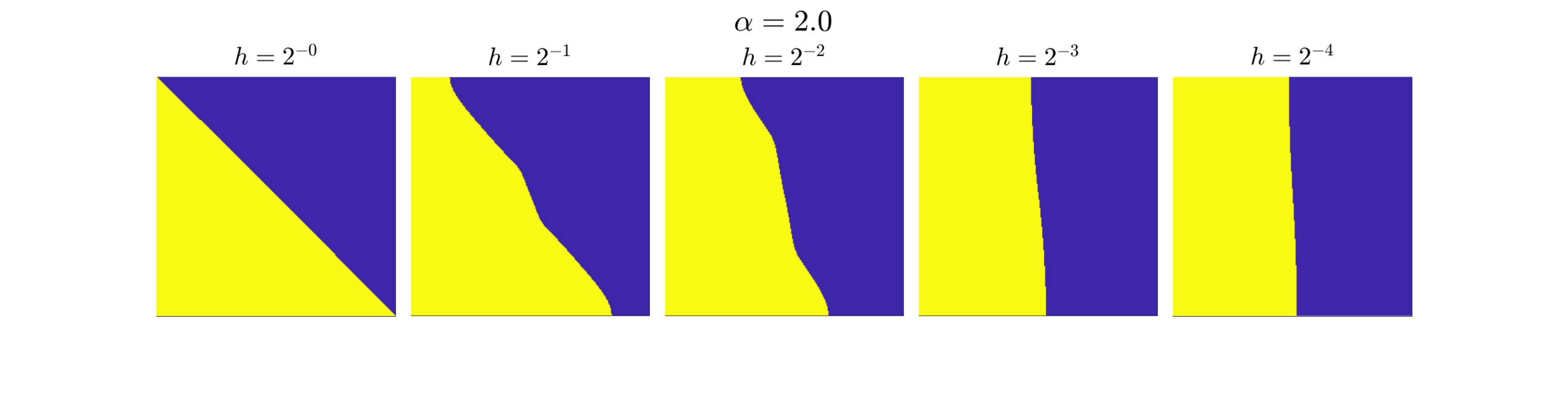}
\includegraphics[width=\textwidth,trim=4cm 2cm 1cm 0cm,clip]{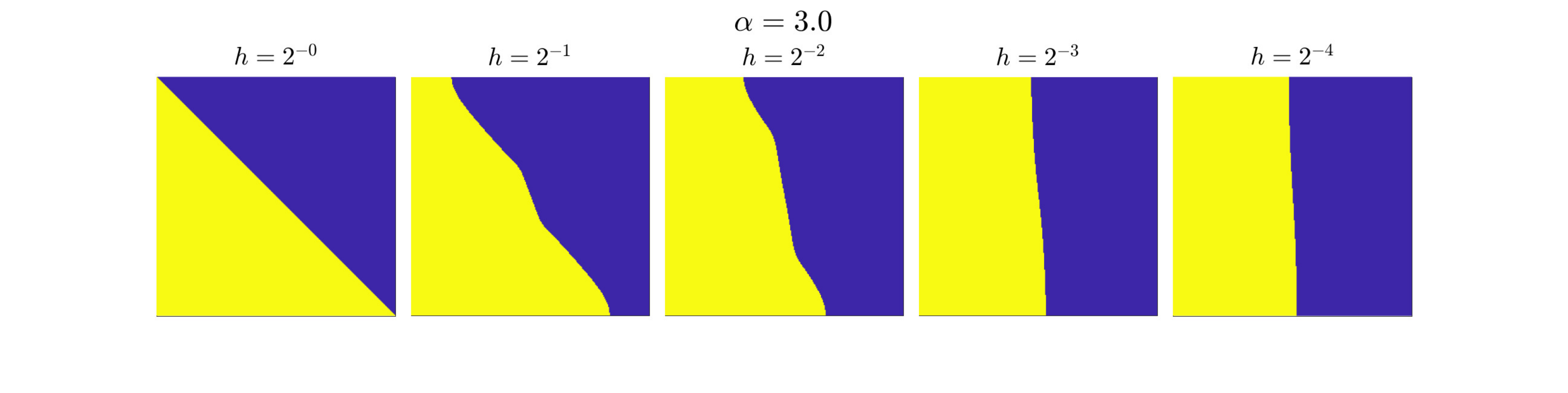}
\caption{The sign of minimizers from Figure \ref{fig:channel}, showing the resulting classification.}
\label{fig:channel_sign}
\end{figure}

\subsection{Localization Bounds for Kriging and Probit}
\label{ssec:Epsilon}
\mt{
We study how the rate affects convergence to the continuum limits when the localization parameter decreases and the number of data points $n$ is increased.}
We consider \hyperlink{labmod2}{{\bf Labelling model 2}} using both the kriging and probit models; this serves to illustrate the result of Theorem \ref{thm:LimitThmOpt:Probit:pr2neg}, motivate Remark \ref{rem:LimitThmOpt:Probit:pr2neg}, and provide a relation to the results of \cite{SlepcevThorpe}.

We work on the domain $\Omega = (0,1)^2$ and take a uniform data distribution $\rho$. In all cases we fix two datapoints which we label with opposite signs, and sample the remaining $n-2$ datapoints. For kriging we consider the situation where
the data is viewed as noise-free so that the label values are interpolated. We calculate the minimizer $u_n$ of $\Jkr^{(n)}$ numerically via the closed form solution
\[
u_n = A^{(n),-1}R^*(RA^{(n),-1}R^*)^{-1}y,
\]
where $R \in \R^{2\times n}$ is the mapping taking vectors to their values at the \mt{labeled} points. In order to numerically estimate the continuum minimizer $u$ of $\Jkr^{(\infty)}$, we construct a finite-difference approximation to $\cL$ on a uniform grid of 65536 points. This leads to an approximation $\hat{\cA}$ to $\cA$, from which we again use the closed form solution to compute $\hat{u}\approx u$:
\[
\hat{u} = \hat{\cA}^{-1}\hat{R}^*(\hat{R}\hat{\cA}^{-1}\hat{R}^*)^{-1}y,
\]
where $\hat{R}\in\R^{2\times 65556}$ takes discrete functions to their values at the \mt{labeled} points.

In Figure \ref{fig:rates_krig} (left) we show how the $L^2_{\mu_n}$ error between $u_n$ and $\hat{u}$ varies with respect to $\eps$ for increasing values of $n$. All errors are averaged over 200 realizations of the \mt{unlabeled} datapoints, and we consider 100 uniformly spaced values of $\eps$ between 0.005 and 0.5. We see that $\eps$ must belong to a `sweet-spot' in order to make the error small -- if $\eps$ is too small or too large convergence doesn't occur. The right hand side of the figure shows how these lower and upper bounds vary with $n$; the bounds are defined numerically as the points where the second derivative of the error curve changes sign. The rates are in agreement with the results and conjectures up to logarithmic terms, although the sharp bounds are not obtained -- we see that the lower bounds are larger than $\mathcal{O}(n^{-\frac{1}{2}})$, and the upper bounds are smaller than $\mathcal{O}(n^{-\frac{1}{2\alpha}})$. It is possible that the sharp bounds may be approached in a more asymptotic (and computationally infeasible) regime.

Similarly, we note that the minimum error for $\alpha=2$ in Figure~\ref{fig:rates_krig} decreases very slowly in the range of $n$ we considered. This again indicates that we are not yet in the asymptotic regime at $n=1600$. Further experiments (not included) for larger values of $n$ show that the minimum error does converge as $n\to \infty$ as expected.

For the probit model we take $\gamma = 0.01$ and use the same gradient flow algorithm as in subsection \ref{ssec:Channel} for both the continuum and discrete minimizers. Figure \ref{fig:rates_prob} shows the errors, analogously to Figure \ref{fig:rates_krig}. Note that the errors are plotted on logarithmic axes here, as unlike the kriging minimizers, there is no restriction for the minimizers to be on the same scale as the labels. We see that the same trend is observed in terms of requiring upper and lower bounds on $\eps$, and a shift of the error curves towards the left as $n$ is increased.

\begin{figure}
\centering
\includegraphics[width=0.49\textwidth,trim=0cm 0cm 0cm 0cm,clip]{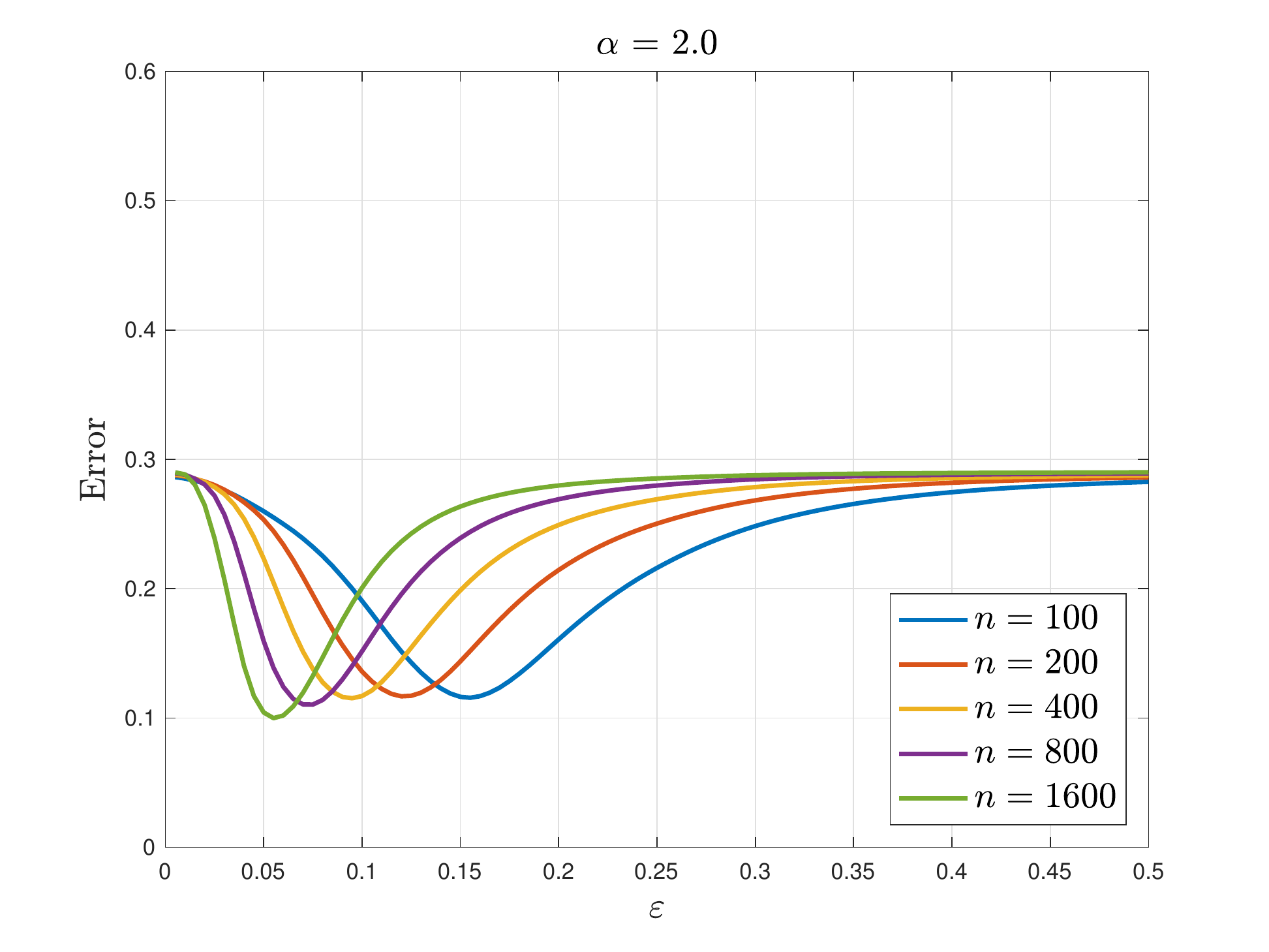}
\includegraphics[width=0.49\textwidth,trim=0cm 0cm 0cm 0cm,clip]{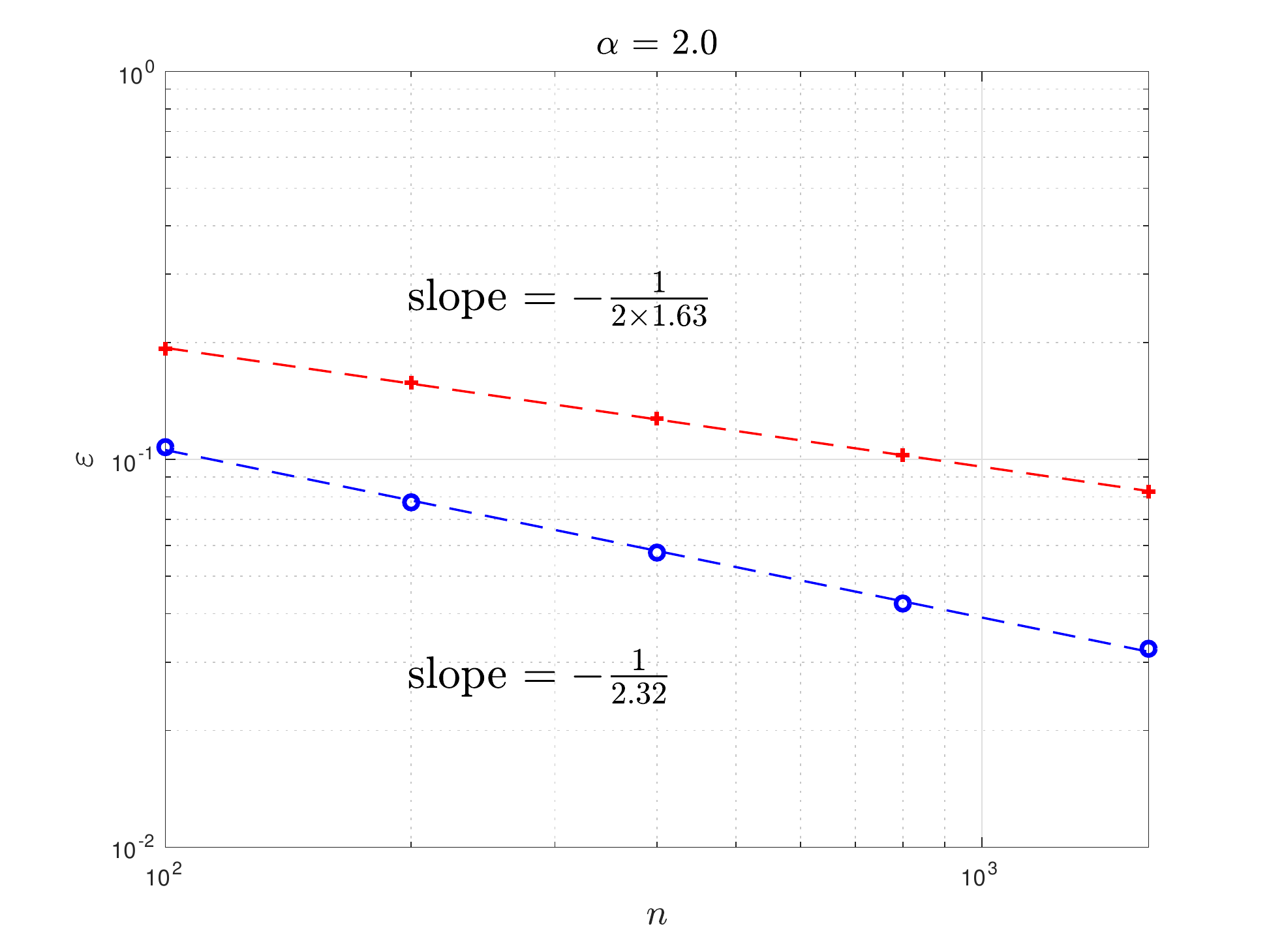}
\includegraphics[width=0.49\textwidth,trim=0cm 0cm 0cm 0cm,clip]{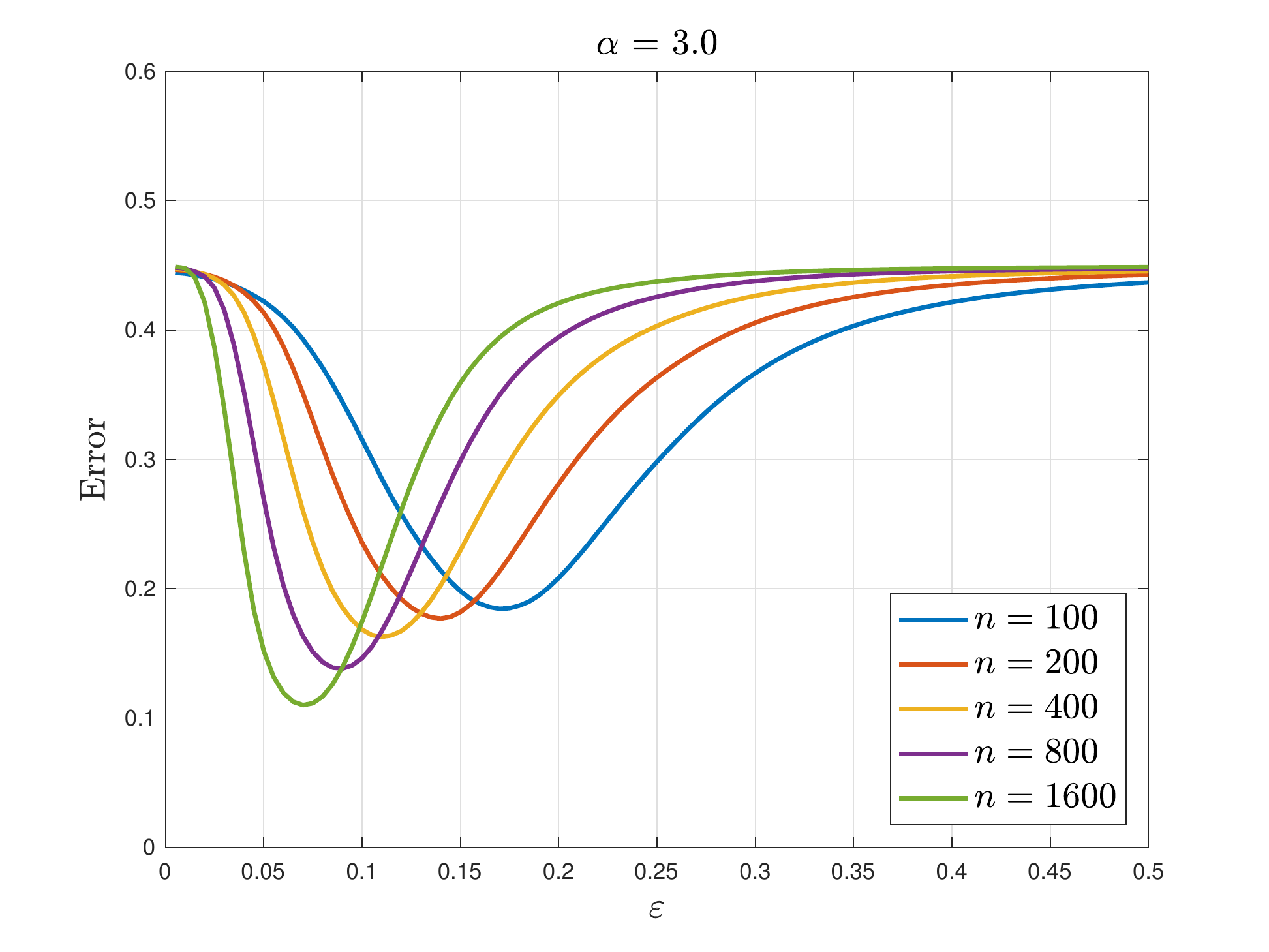}
\includegraphics[width=0.49\textwidth,trim=0cm 0cm 0cm 0cm,clip]{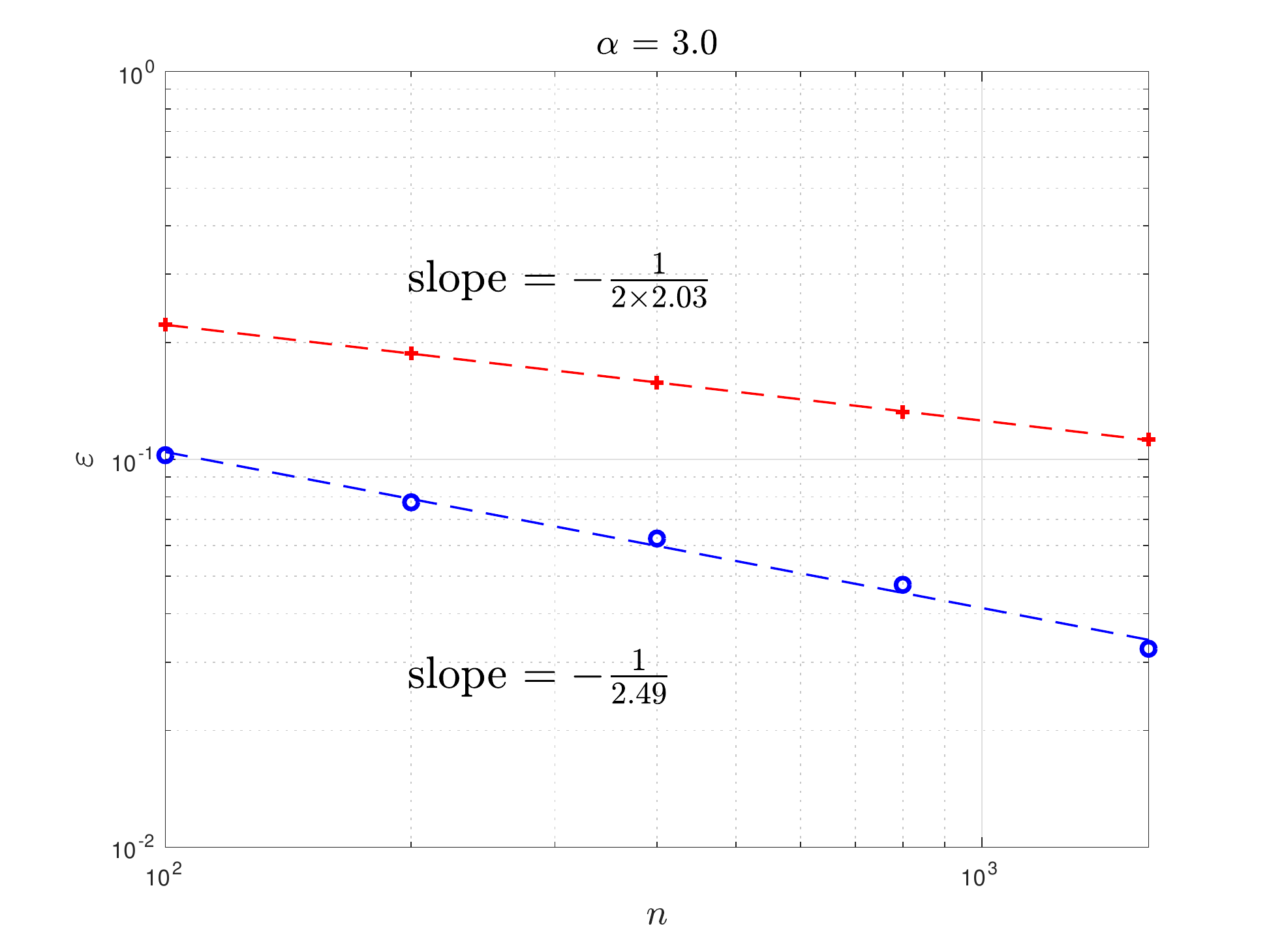}
\includegraphics[width=0.49\textwidth,trim=0cm 0cm 0cm 0cm,clip]{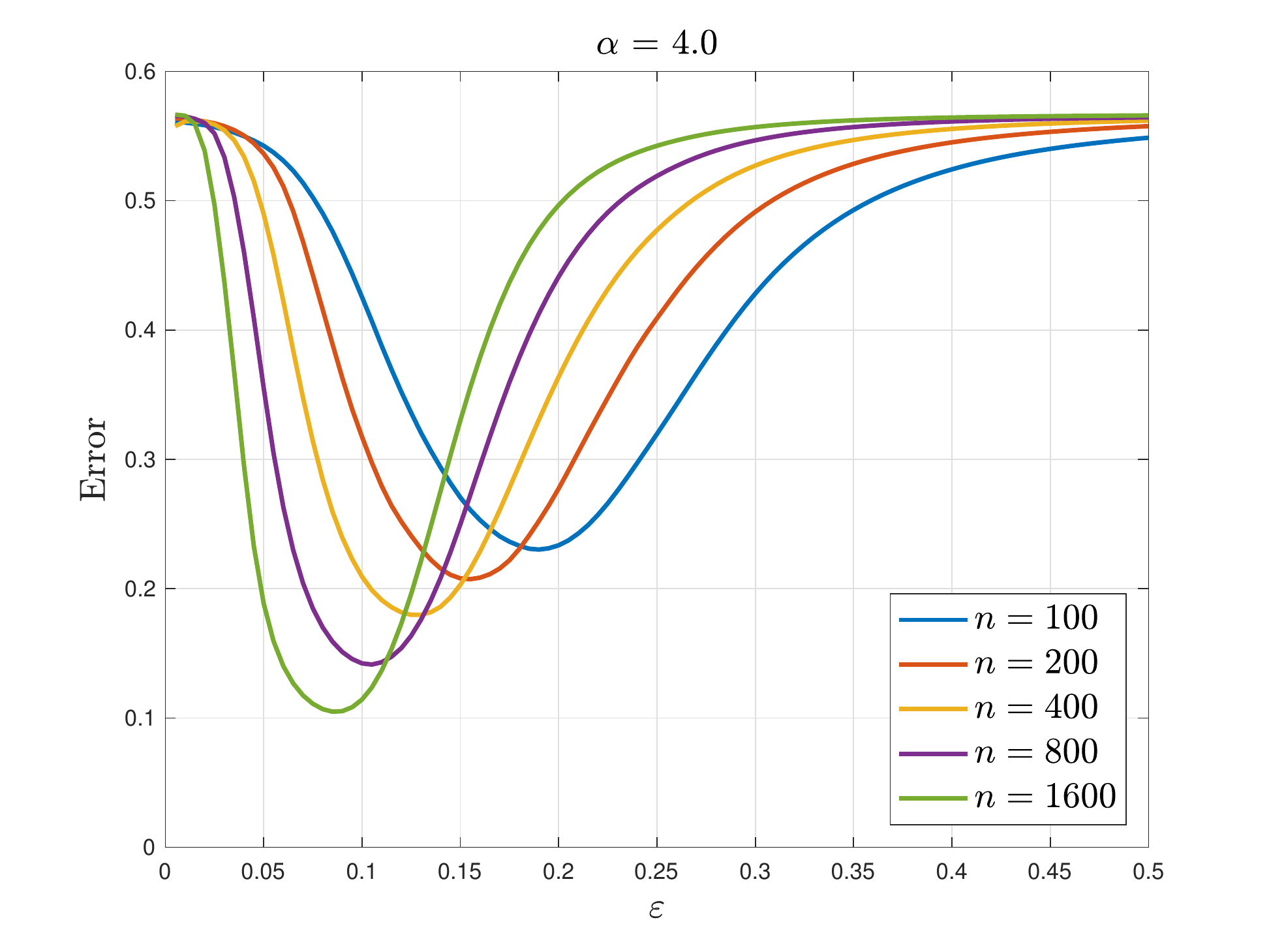}
\includegraphics[width=0.49\textwidth,trim=0cm 0cm 0cm 0cm,clip]{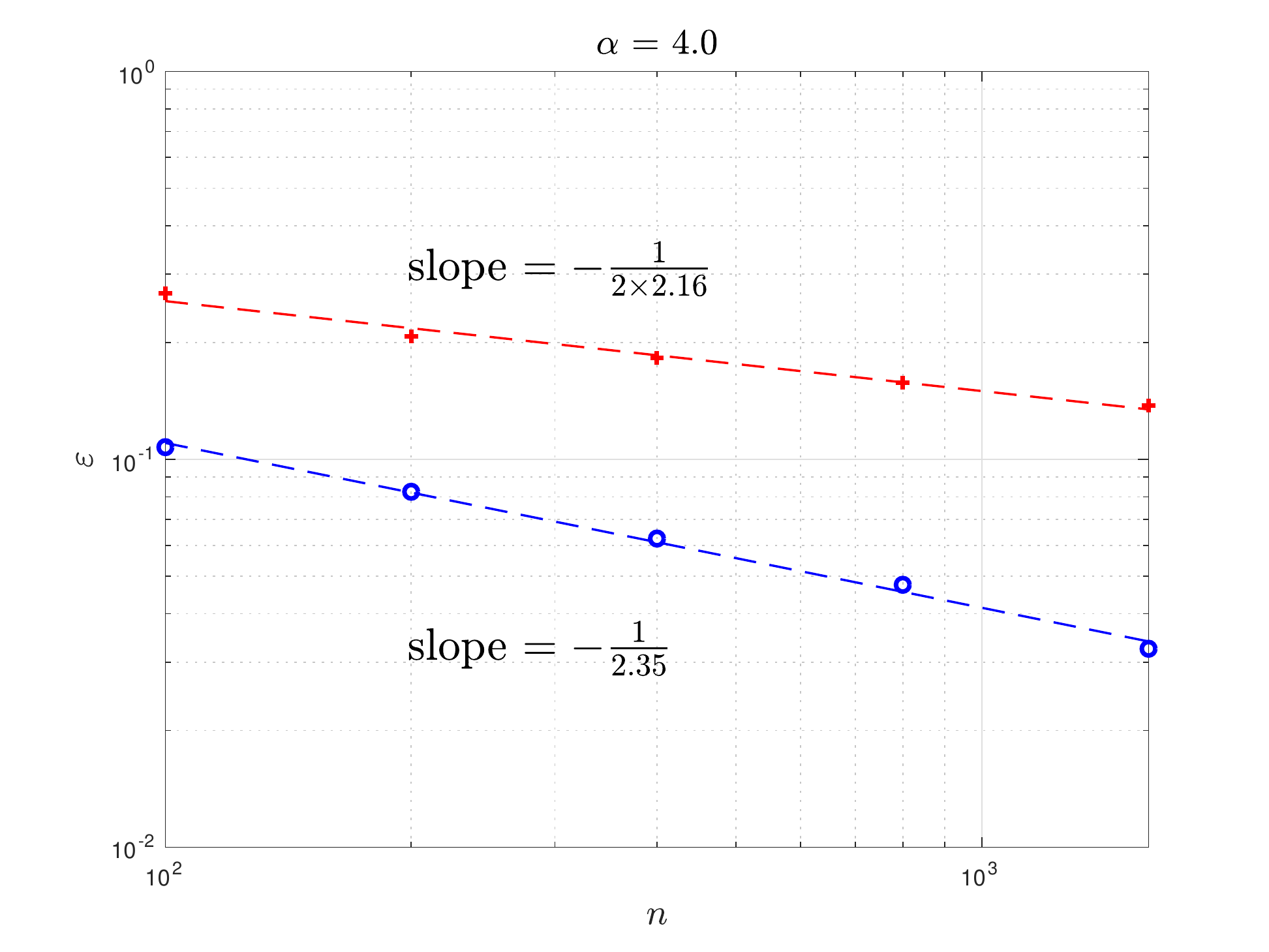}
\caption{(Left) The $L^2_{\mu_n}$ error between discrete minimizers and continuum minimizers of the kriging model versus localization parameter $\eps$, for different values of $n$. (Right) The upper and lower bounds for $\eps(n)$ to provide convergence. The slopes of the lines of best fit provide estimates of the rates.}
\label{fig:rates_krig}
\end{figure}

\begin{figure}
\centering
\includegraphics[width=0.49\textwidth,trim=0cm 0cm 0cm 0cm,clip]{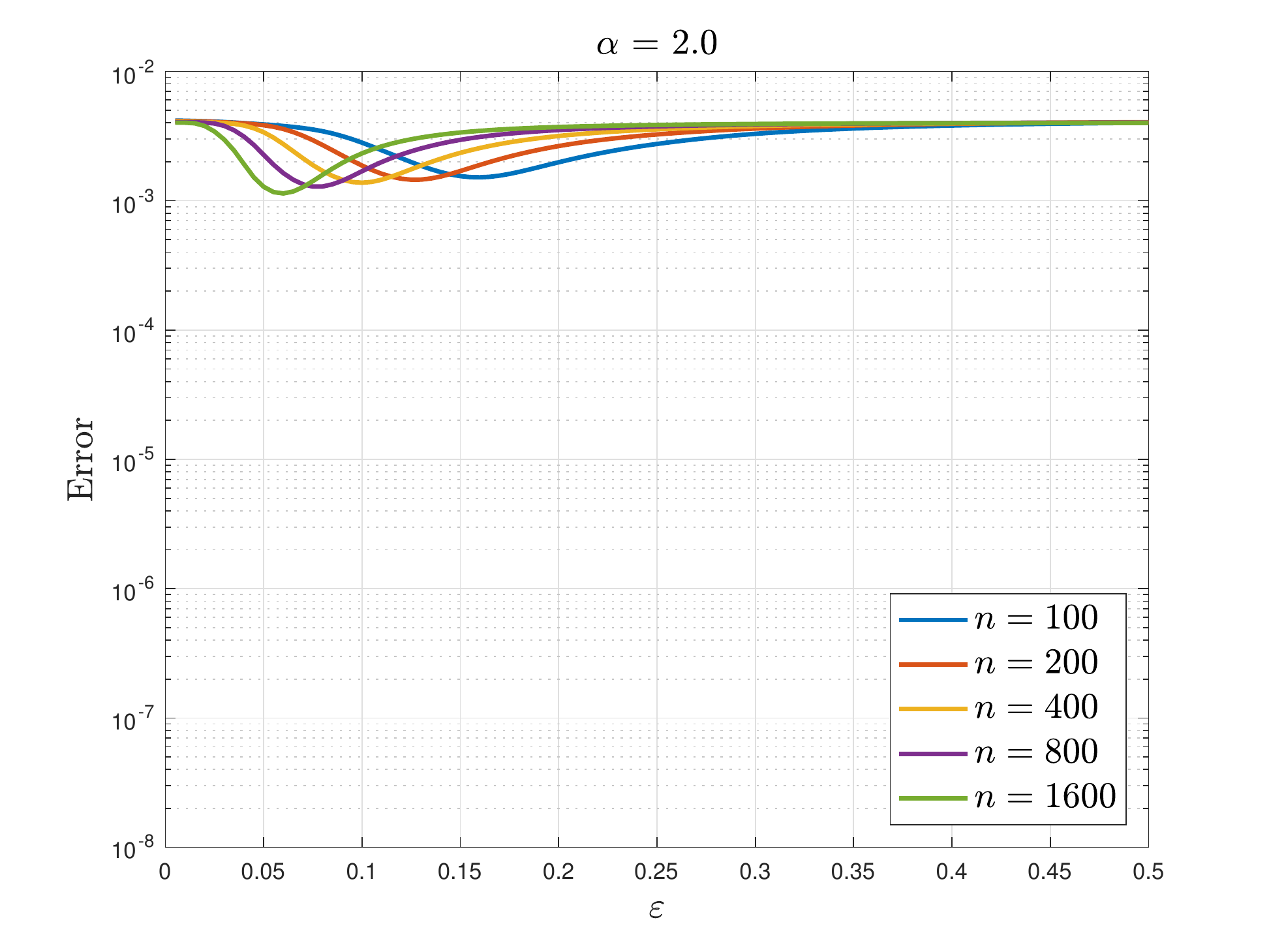}
\includegraphics[width=0.49\textwidth,trim=0cm 0cm 0cm 0cm,clip]{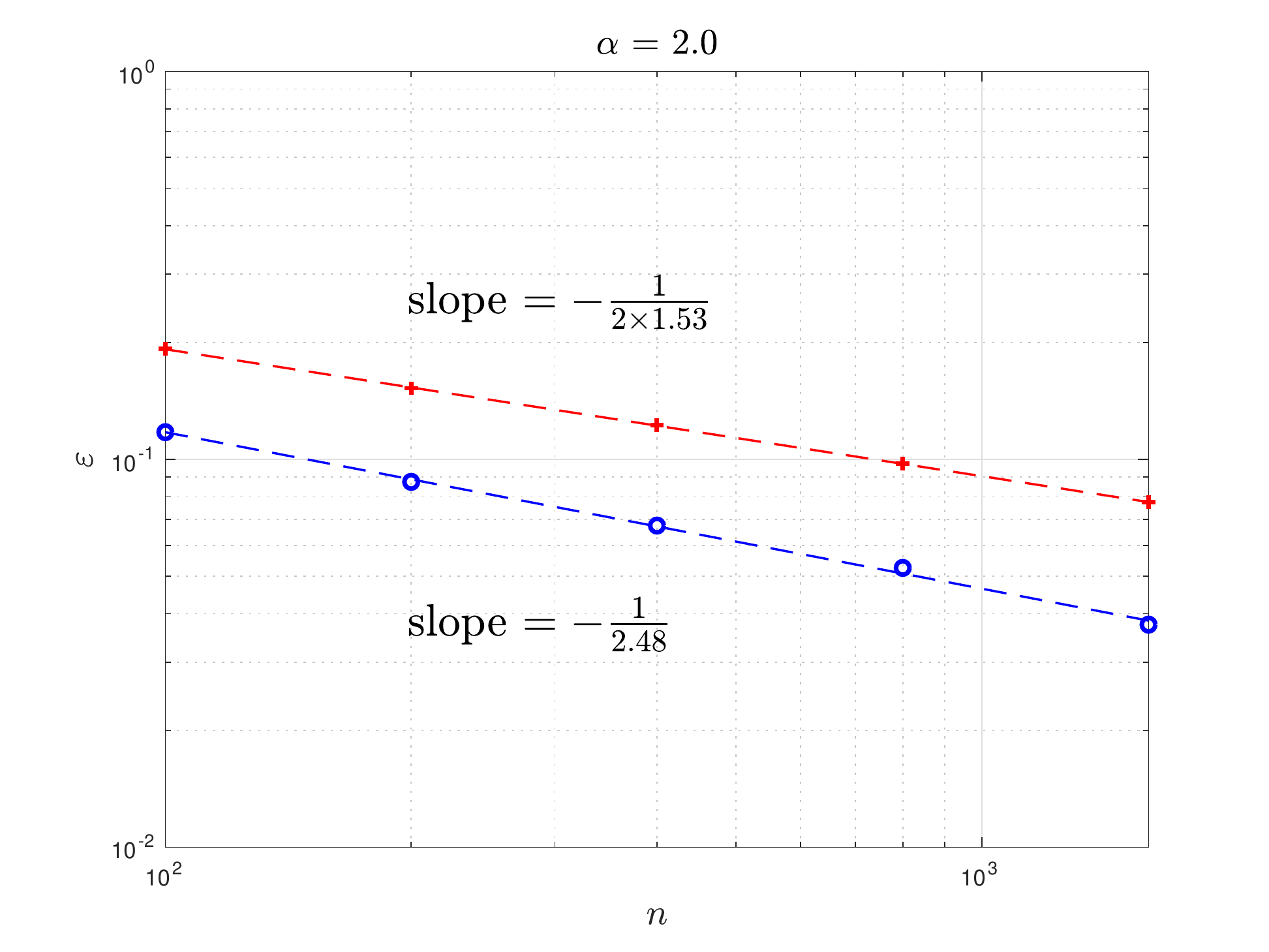}
\includegraphics[width=0.49\textwidth,trim=0cm 0cm 0cm 0cm,clip]{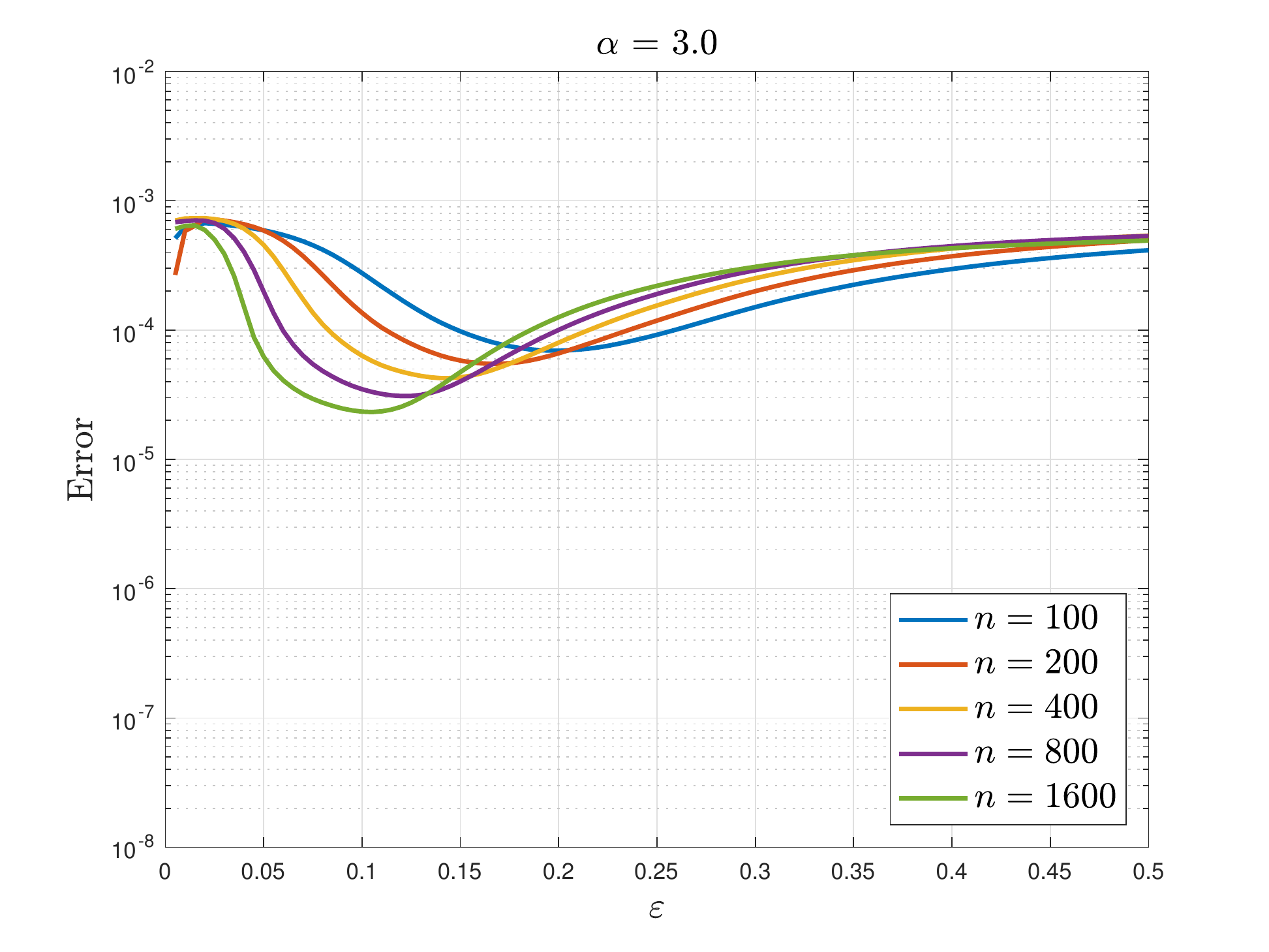}
\includegraphics[width=0.49\textwidth,trim=0cm 0cm 0cm 0cm,clip]{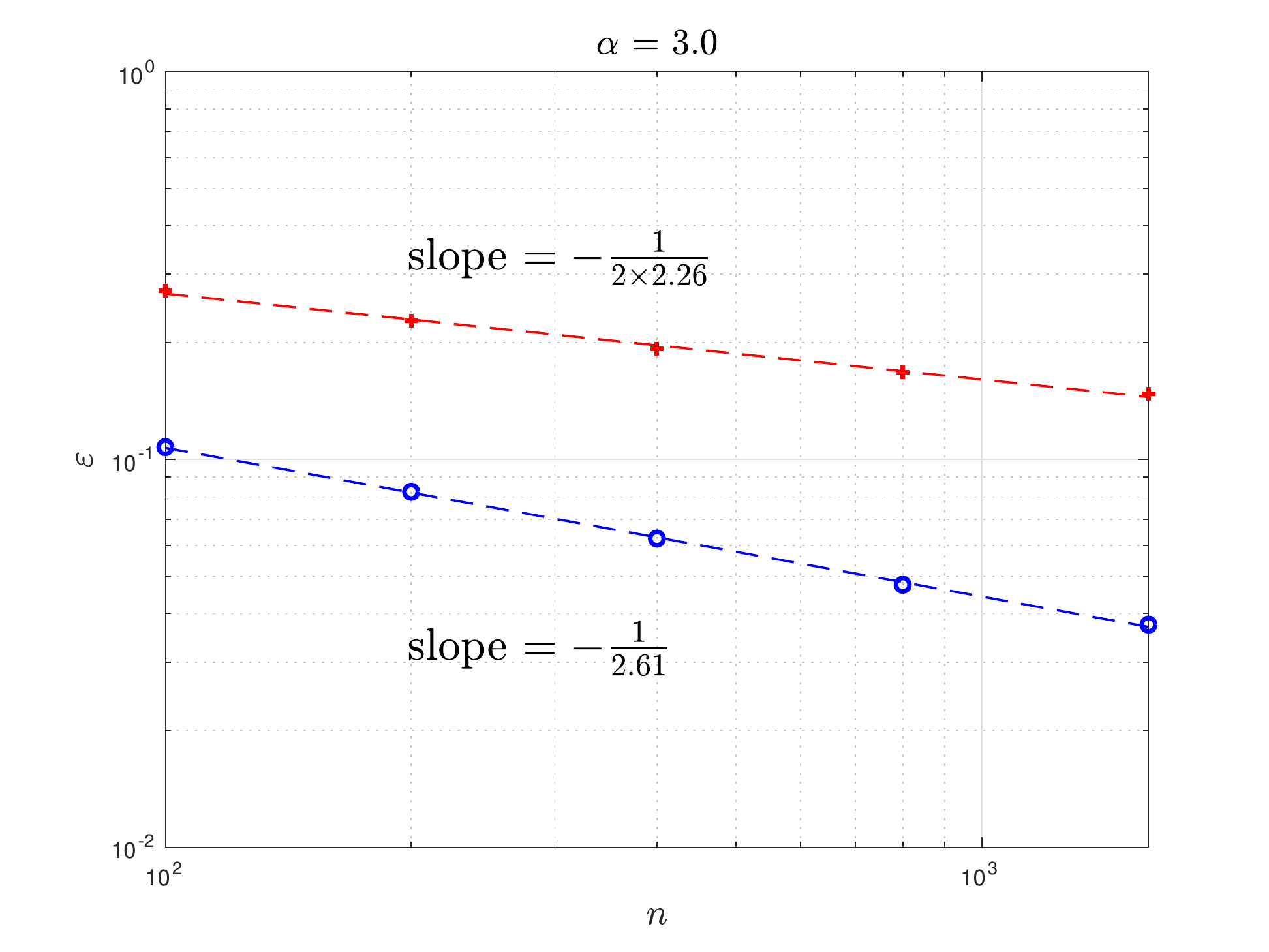}
\includegraphics[width=0.49\textwidth,trim=0cm 0cm 0cm 0cm,clip]{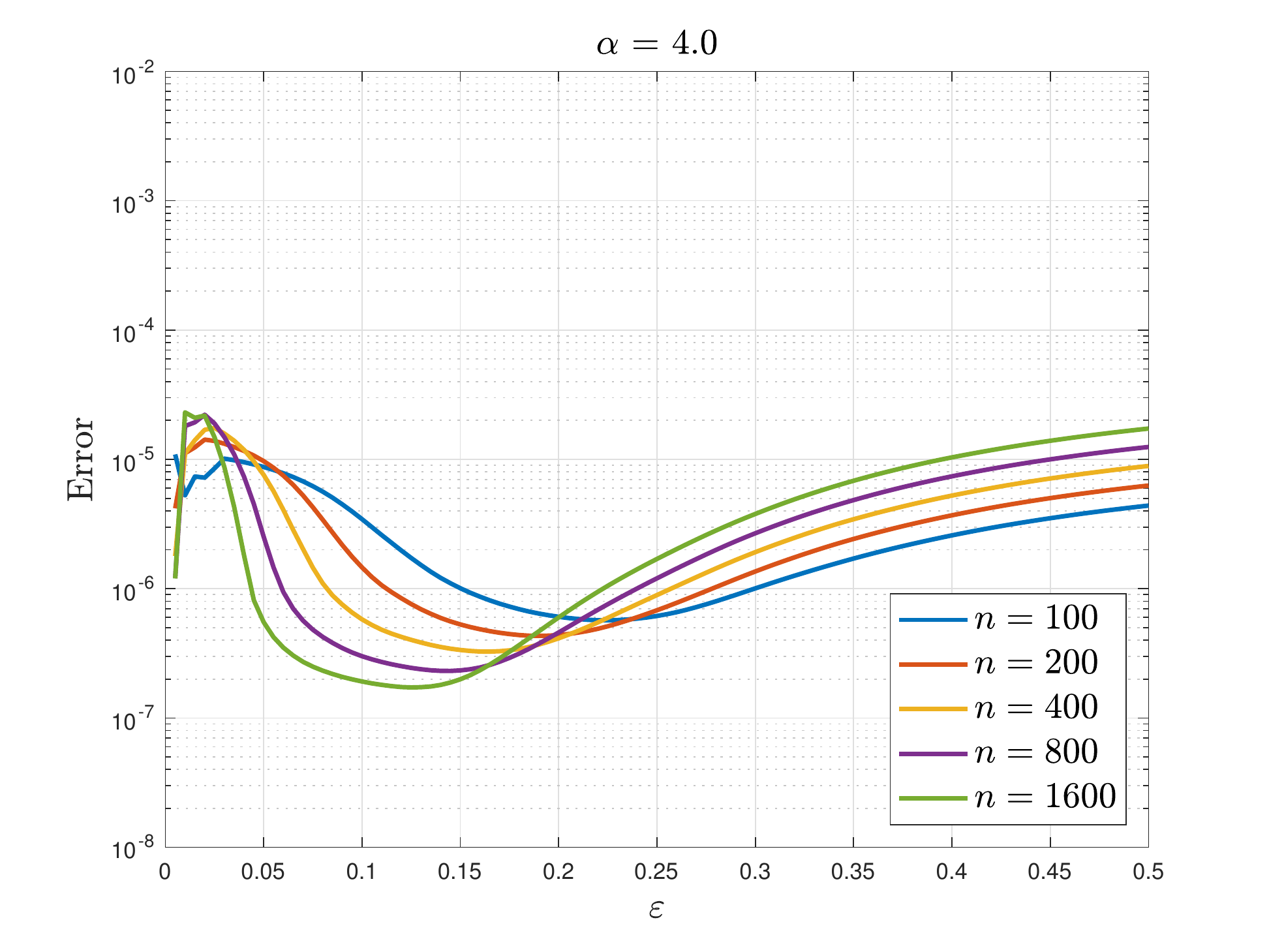}
\includegraphics[width=0.49\textwidth,trim=0cm 0cm 0cm 0cm,clip]{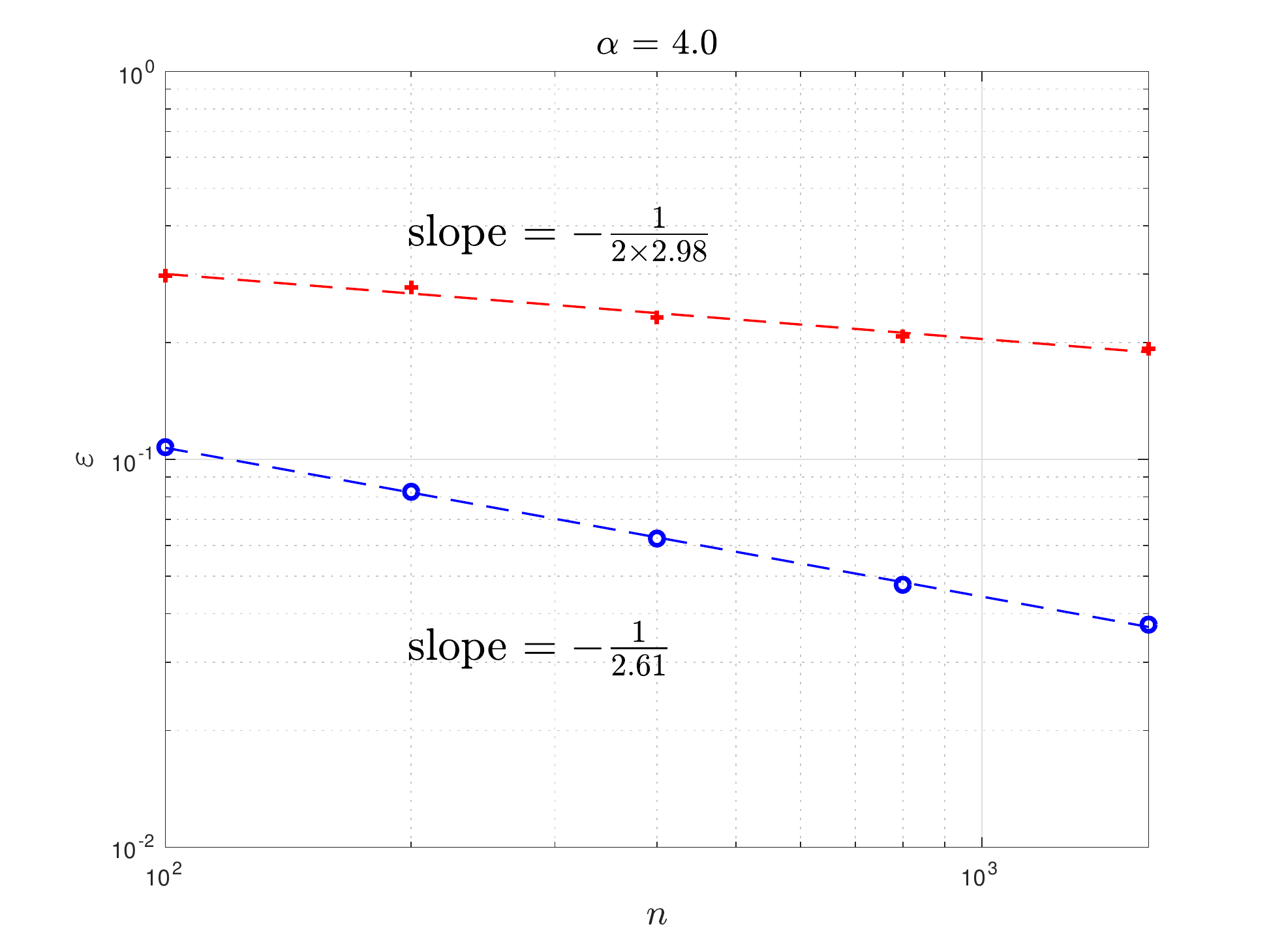}
\caption{(Left) The $L^2_{\mu_n}$ error between discrete minimizers and continuum minimizers of the probit model versus localization parameter $\eps$, for different values of $n$. (Right) The upper and lower bounds for $\eps(n)$ to provide convergence. The slopes of the lines of best fit provide estimates of the rates.}
\label{fig:rates_prob}
\end{figure}

\subsection{Extrapolation on Graphs}
\label{ssec:Extrap}

We consider the problem of smoothly extending a sparsely defined function on a graph to the entire graph. Such extrapolation was studied in \cite{shi2017weighted}, and was achieved via the use of a weighted nonlocal Laplacian. We use the kriging model with \hyperlink{labmod2}{{\bf Labelling Model 2}}, labelling two points with opposite signs, and setting $\gamma = 0$. We fix a set of datapoints $\{x_j\}_{j=1}^{n}$, $n = 1600$, drawn from the uniform density on the domain $\Omega = (0,1)^2$.   We fix $\tau = 1$ and look at how the smoothness of minimizers of the kriging functional $\Jkr^{(n)}$ varies with $\alpha$. 
The minimizers are computed directly from the closed form solution, as in subsection \ref{ssec:Epsilon}.
When $\alpha>d/2$ we choose $\eps$ to approximately minimize the $L^2_{\mu_n}$ errors between the discrete and continuum solutions (since the continuum solution is non-trivial).
When $\alpha\leq d/2$ a representative $\eps$ is chosen which is approximately twice the connectivity radius.
The minimizers are shown in Figure \ref{fig:extrap} for $\alpha = 0.5,1.0,1.5,2.0$. Spikes are clearly visible for $\alpha \leq d/2 = 1$: the requirement for $\alpha > d/2$ to avoid spikes appears to be essential.

\begin{figure}
\centering
\includegraphics[width=0.49\textwidth,trim=3cm 0cm 2.5cm 0cm,clip]{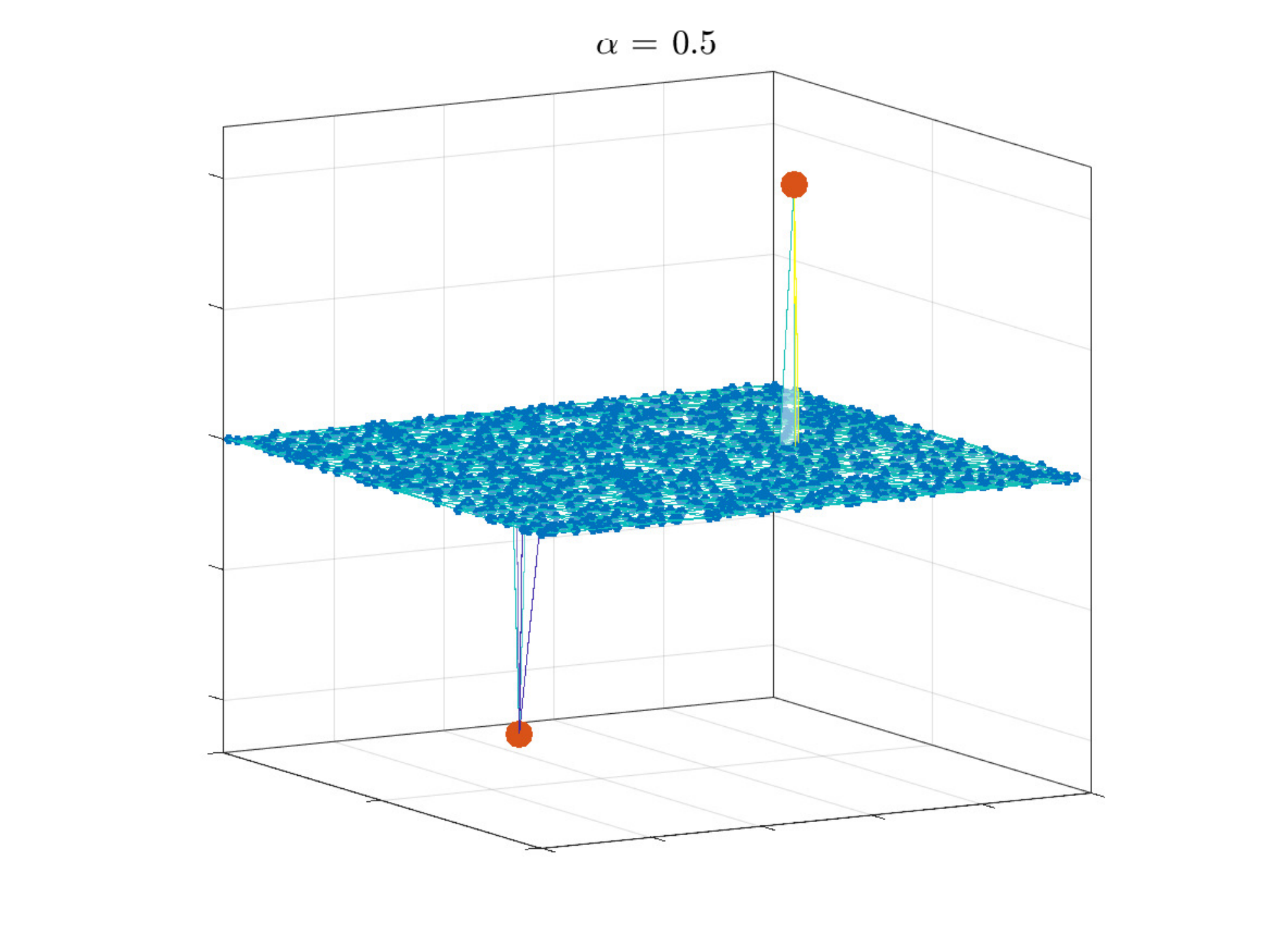}
\includegraphics[width=0.49\textwidth,trim=3cm 0cm 2.5cm 0cm,clip]{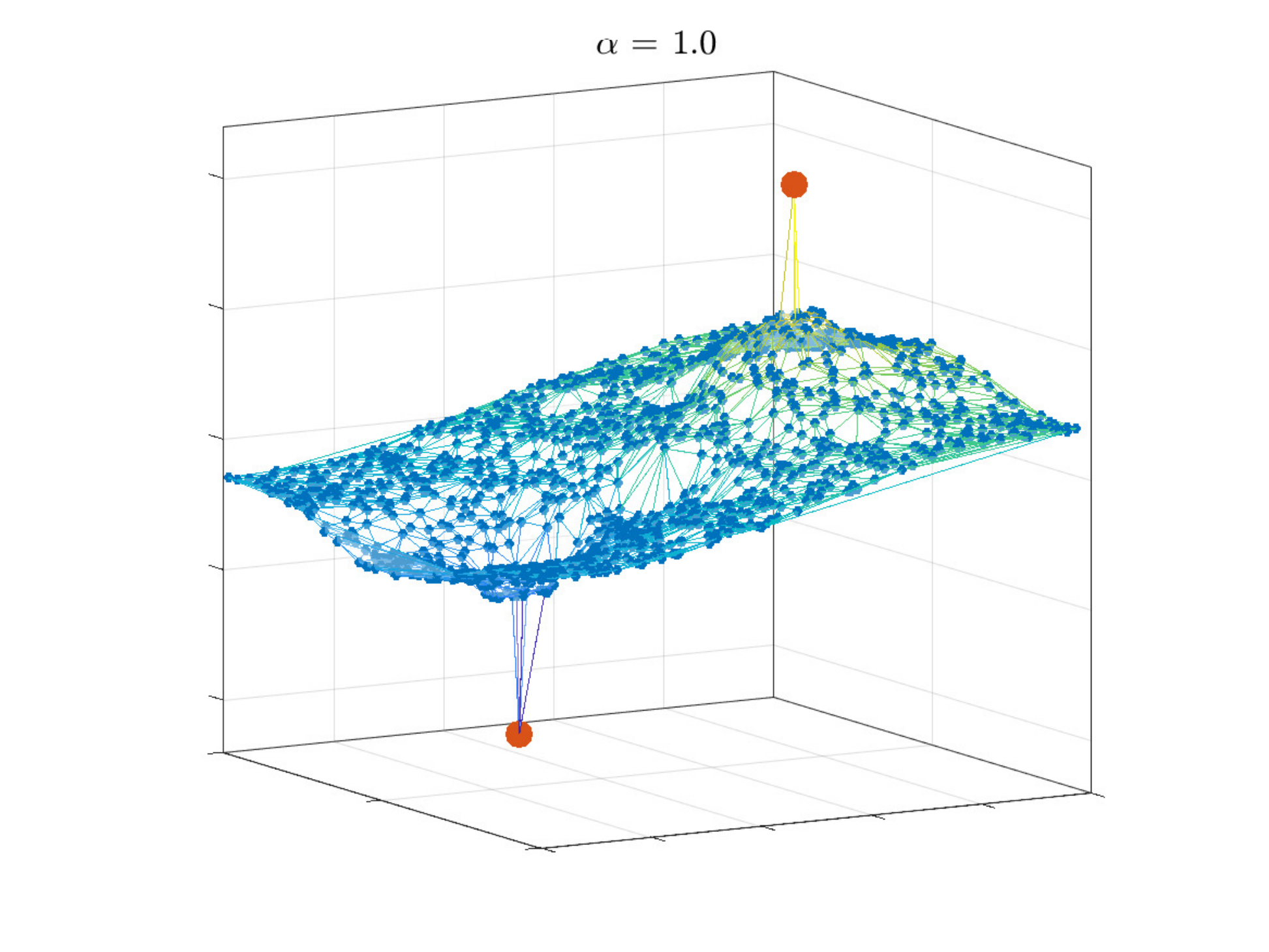}
\includegraphics[width=0.49\textwidth,trim=3cm 0cm 2.5cm 0cm,clip]{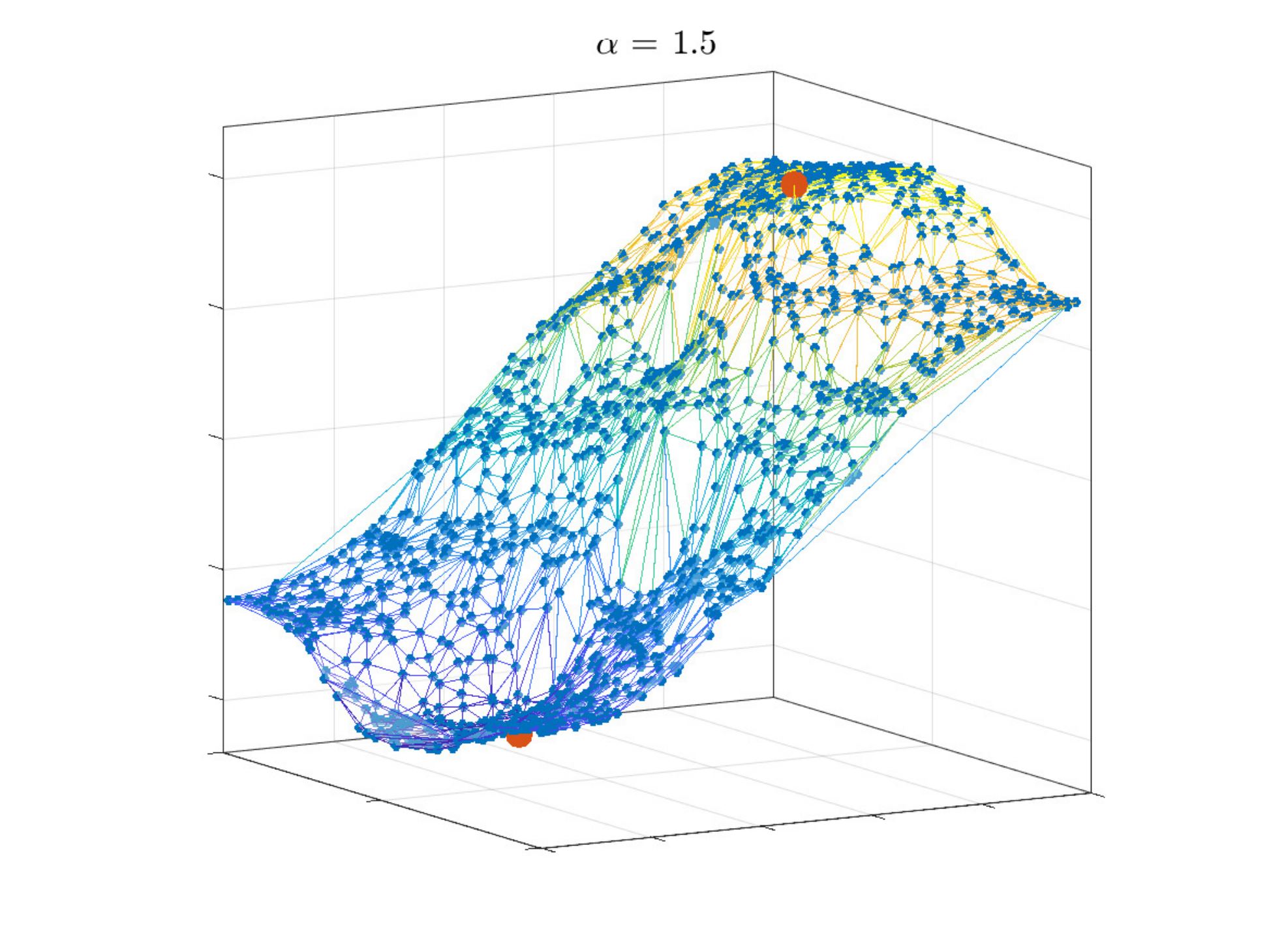}
\includegraphics[width=0.49\textwidth,trim=3cm 0cm 2.5cm 0cm,clip]{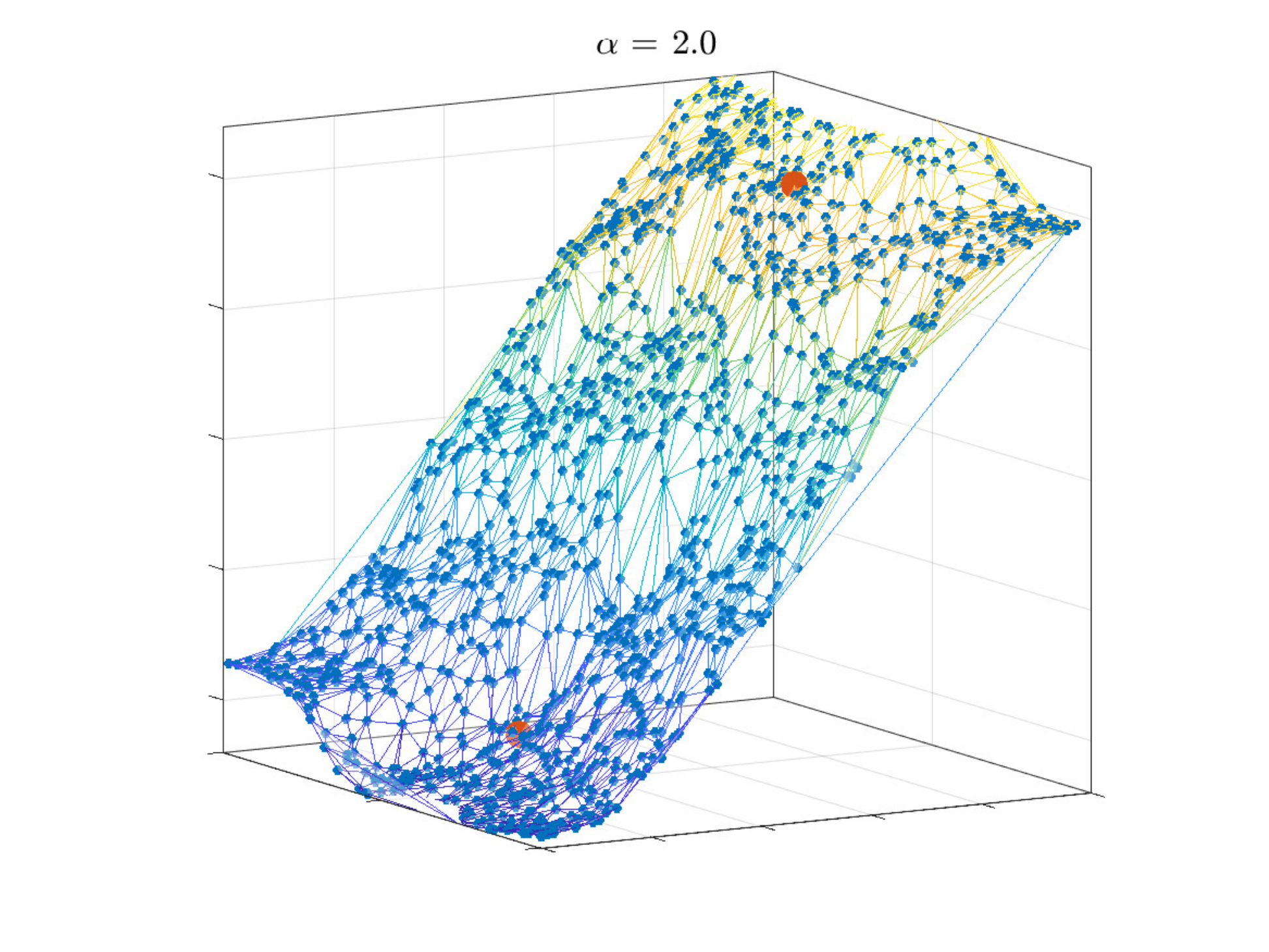}
\caption{The extrapolation of a sparsely defined function on a graph using the kriging model, for various choices of parameter $\alpha$.}
\label{fig:extrap}
\end{figure}

\subsection{Bayesian Level Set for Sampling}
\label{ssec:MCMC}

We now turn to the problem of sampling the conditioned continuum measures introduced in subsections \ref{ssec:PR2} and \ref{ssec:BLS2}, specifically their common $\gamma\to 0$ limit. From this sampling we can, for example, calculate the mean of the classification, which may be used to define a measure of uncertainty of the classification at each point. This is because, for binary random variables, the mean determines the variance. Knowing the uncertainty in classification has great potential utility, for example in active learning in  guiding where to place resources in labelling in order
to reduce uncertainty.

We fix $\Omega = (0,1)^2$. The data distribution $\rho$ is shown in Figure \ref{fig:rho}; it is constructed as a continuum analogue of the two moons distribution \cite{zhou2004learning}, with the majority of its mass concentrated on two curves.  The contrast ratio in the sampling density $\rho$ is approximately 100:1 between the values on and off of the curves. The resulting operator $\cL$ contains significant clustering information: in Figure \ref{fig:rho} we show the second eigenfunction of $\cL$, termed the Fiedler vector in analogy with second eigenvector of the graph Laplacian. The sign of this function provides a good estimate for the decision boundary in an unsupervised context. We use \hyperlink{labmod2}{{\bf Labelling Model 2}}, labelling a single point on each curve with opposing signs as indicated by $\bullet$ and $\circ$ in Figure \ref{fig:rho}.

Sampling is performed using the preconditioned Crank-Nicolson MCMC algorithm \cite{CRSW08}, which has favourable dimension-independent statistical properties,
as demonstrated in \cite{trillos2017consistency} in the graph-based
setting of relevance here. We consider three choices of $\alpha > d/2$, and two choices of inverse length-scale parameter $\tau$. In general we require $\alpha > d$ for the measure $\nu_2$ in Theorem \ref{thm:ltp:probit:ZeroNoise} to be well-defined. However numerical evidence suggests that
the conclusions of Proposition \ref{t:gold} are satisfied with this choice of $\rho$, implying that we may make use of Remark \ref{rem:zeronoise_alpha} and that
$\alpha > \frac{d}{2}$ suffices. The operator $\cL$ is discretized using a finite difference method on a square grid of $40000$ points, and sampling is performed on the span of its first $500$ eigenfunctions. 

In Figure \ref{fig:mcmc_1} we show the mean of the sign of samples on the left hand side, for each choice of $\alpha$, after fixing $\tau = 1$. Note that uncertainty is greater the further the values of the mean are from $\pm 1$: specifically we have that $\mathrm{Var}\big(S(u(x)\big) = 1 - \big[\mathbb{E}(S(u(x)))\big]^2$. We see that the classification on the curves where the data concentrates
is fairly certain, whereas classification away from the curves is uncertain; 
furthermore the certainty increases away from the curves slightly as $\alpha$ is increased. Samples $S(u)$ are also shown in the same figure; the uncertainty away from the curves is illustrated also by these samples.

In Figure \ref{fig:mcmc_2} we show the same results, but with the choice $\tau = 0.2$ so that samples possess a longer length scale. The classification certainty now propagates away from the curves more easily. The effect of the asymmetry of the labelling is also visible in the mean for the case $\alpha = 4$: uncertainty is higher in the bottom-left corner than the top-left corner.

Since the prior on the latent random field $u$ may be difficult to ascertain in applications, the sensitivity of the classification on the choice of the parameters $\alpha$, $\tau$ indicates that it could be wise to employ hierarchical Bayesian methods to learn appropriate values for them along with the latent field $u$. Dimension robust MCMC methods are available to sample such hierarchical distributions \cite{chen2018hierarchical}, and application to classification problems are shown in that paper.

%\todo[inline]{add something about relation of $\tau$ choice to the spectral gap?}

\begin{figure}
\centering
\includegraphics[width=0.49\textwidth,trim=7cm 1cm 6.5cm 0cm,clip]{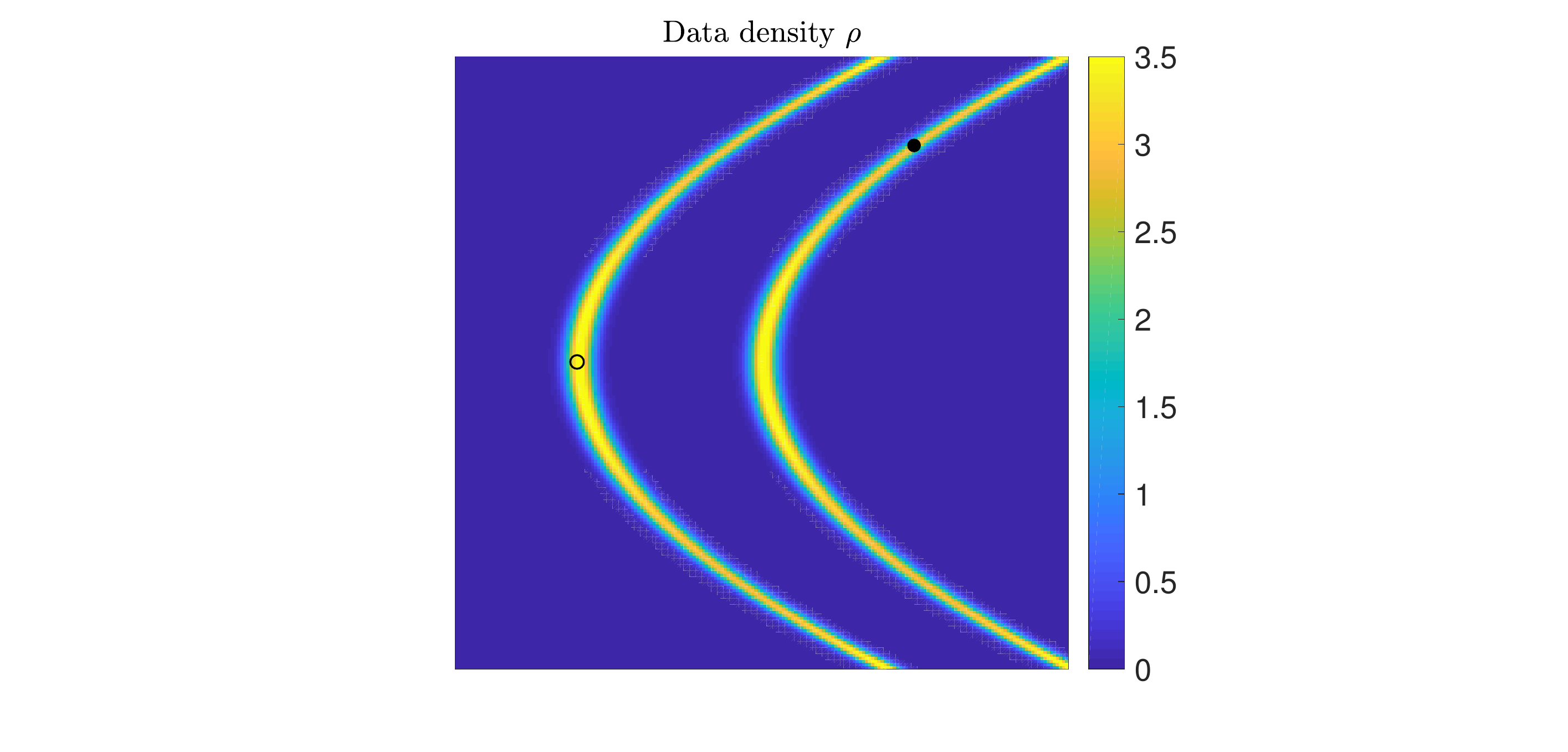}
\includegraphics[width=0.49\textwidth,trim=7cm 1cm 6.5cm 0cm,clip]{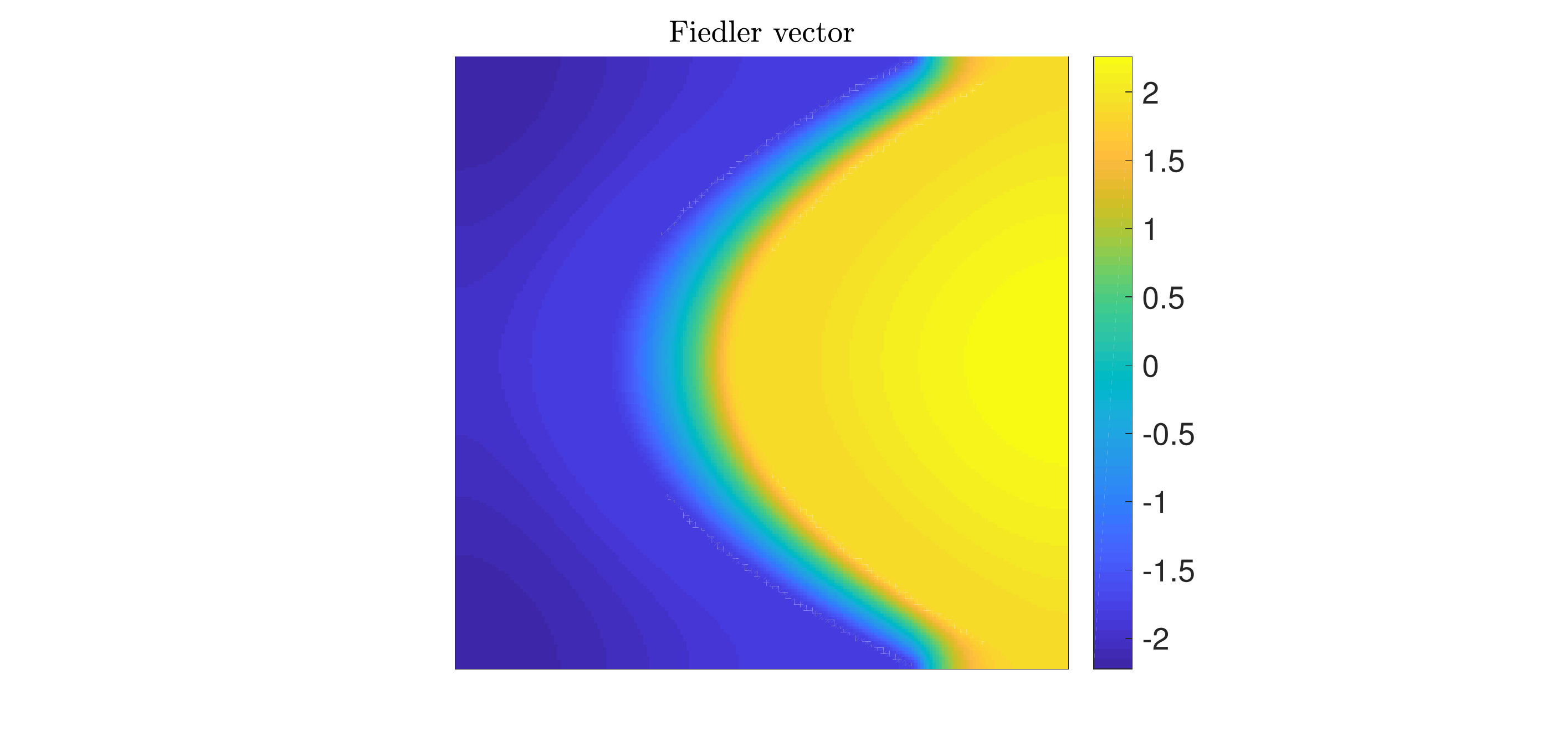}
\caption{(Left) The data distribution $\rho$ used in the MCMC experiments, and the locations of the two \mt{labeled} datapoints. (Right) The second eigenfunction of the operator $\cL$ corresponding to $\rho$.}
\label{fig:rho}
\end{figure}

\begin{figure}
\centering
\includegraphics[width=\textwidth,trim=1.1cm 0cm 3cm 0cm,clip]{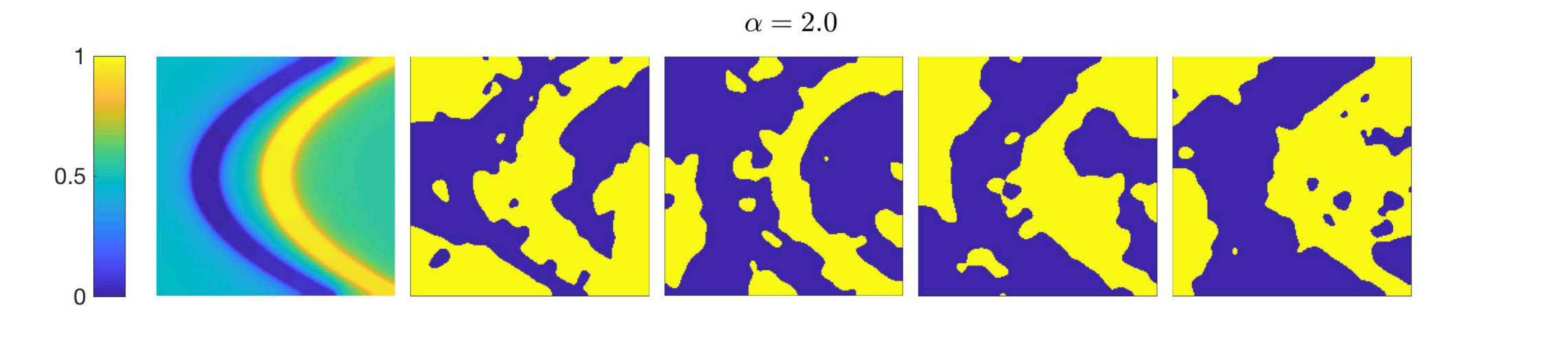}
\includegraphics[width=\textwidth,trim=1.1cm 0cm 3cm 0cm,clip]{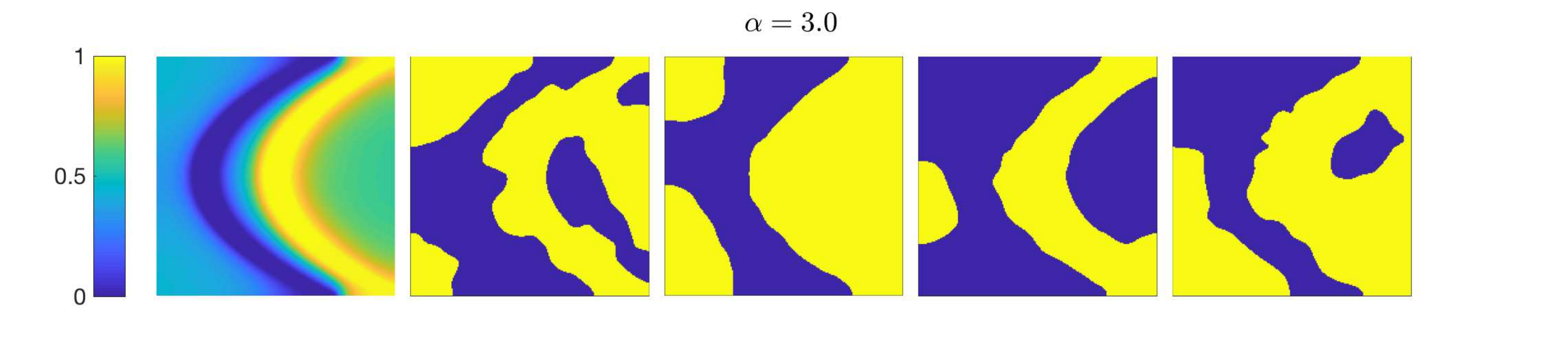}
\includegraphics[width=\textwidth,trim=1.1cm 0cm 3cm 0cm,clip]{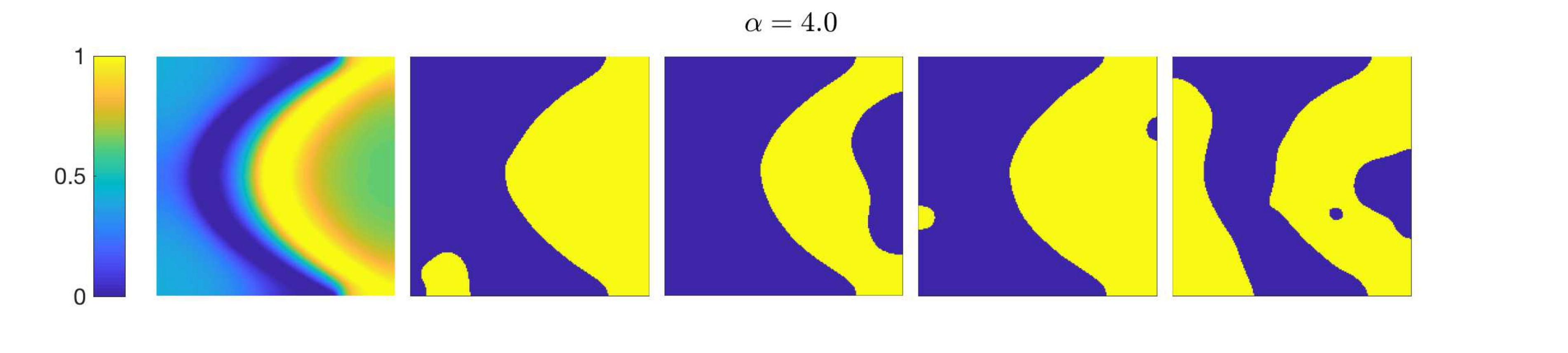}
\caption{(Left) The mean $\mathbb{E}\big(S(u)\big)$ of the classification arising from the conditioned measure $\nu_2$. (Right) Examples of samples $S(u)$ where $u \sim \nu_2$. Here we choose $\tau = 1$.}
\label{fig:mcmc_1}
\end{figure}

\begin{figure}
\centering
\includegraphics[width=\textwidth,trim=1.1cm 0cm 3cm 0cm,clip]{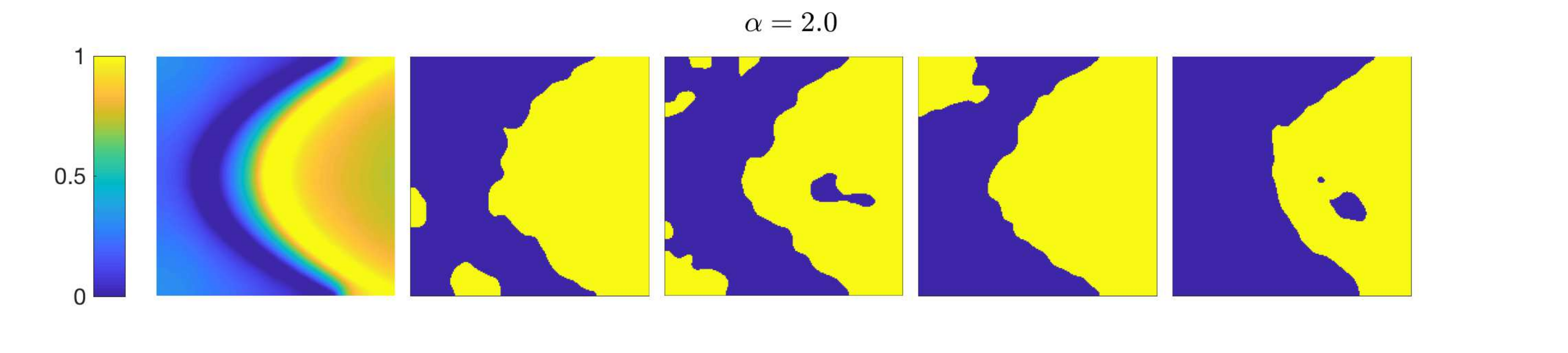}
\includegraphics[width=\textwidth,trim=1.1cm 0cm 3cm 0cm,clip]{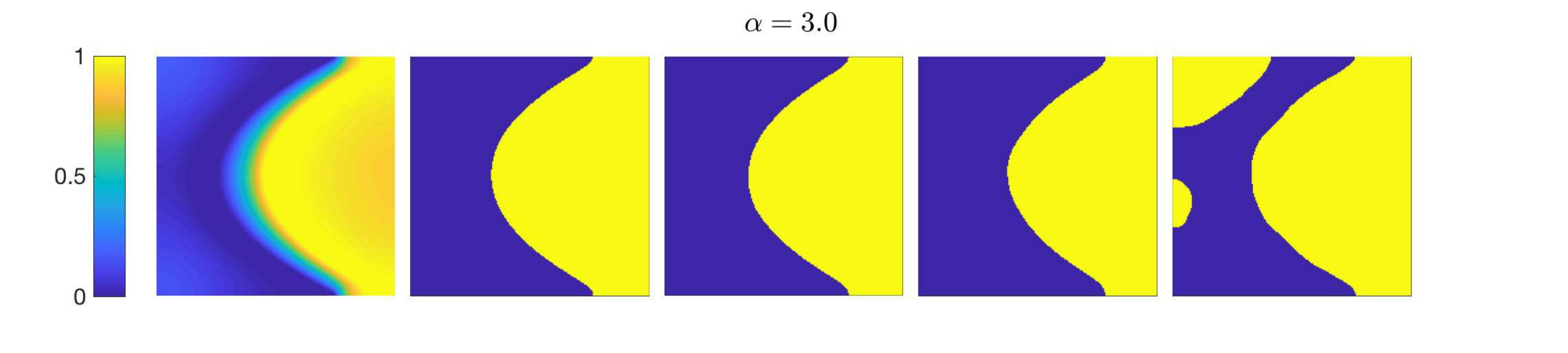}
\includegraphics[width=\textwidth,trim=1.1cm 0cm 3cm 0cm,clip]{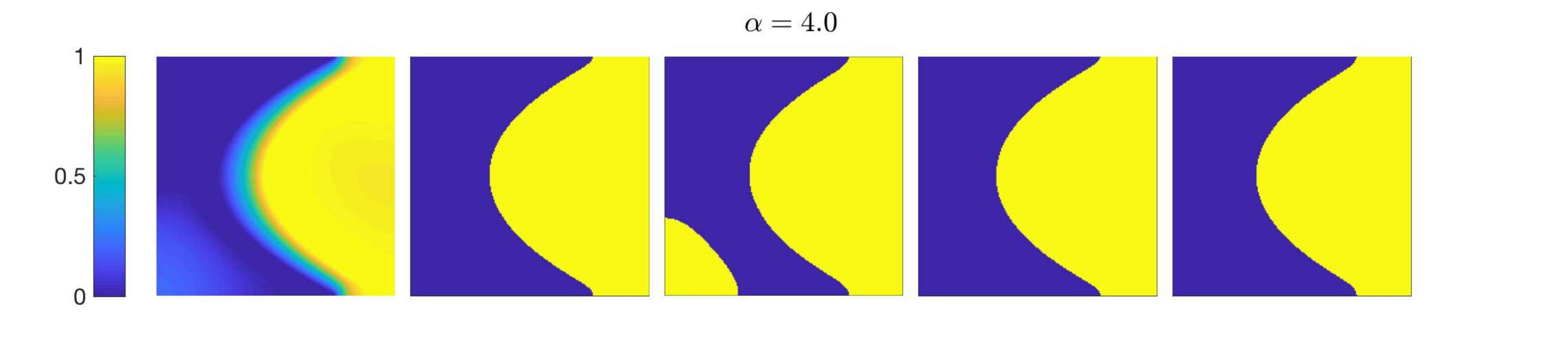}
\caption{(Left) The mean $\mathbb{E}\big(S(u)\big)$ of the classification arising from the conditioned measure $\nu_2$. (Right) Examples of samples $S(u)$ where $u \sim \nu_2$. Here we choose $\tau = 0.2$.}
\label{fig:mcmc_2}
\end{figure}

\section{Conclusions}
\label{sec:C}

In this paper we have studied large graph limits of semi-supervised
learning problems in which smoothness is imposed via a shifted graph 
Laplacian, raised to a power. Both optimization and Bayesian approaches
have been considered. To keep the exposition manageable in length
we have confined our attention to the unnormalized graph Laplacian.
However, one may instead choose to work with the normalized graph 
Laplacian $L=I-D^{-\frac12}WD^{-\frac12}$, in place of $L = D-W$. 
In the normalized case the continuum PDE operator is given by
\[ \cL u=-\frac{1}{\rho^{3/2}}\nabla \cdot \left(\rho^2 \nabla \Bigl(\frac{u}{\rho^{1/2}}\Bigr)\right) \]
with no flux boundary conditions:
$\nabla \bigl(\frac{u}{\rho^{1/2}}\bigr) \cdot \nu=0$ on $\partial\Omega$, where $\nu$ is the outside unit normal vector to $\partial \Omega$.
Theorems \ref{thm:LimitThmDir:LimitThmDir}, \ref{thm:LimitThmOpt:Probit:pr} and \ref{thm:ltp:probit:ZeroNoise} generalize in a straightforward way
to such a change in the graph Laplacian.

Future directions stemming from the work in this paper include:
(i) providing a limit theorem for probit MAP estimators under \hyperlink{labmod2}{{\bf Labelling Model 2}}; (ii) providing limit theorems for the Bayesian probability distributions considered,
using the machinery introduced in \cite{trillos2017consistency,trillos2017continuum};
(iii) using the limiting problems in order to analyze and quantify efficiency of algorithms on large graphs;
(iv) invoking specific sources of data and studying the effectiveness of 
PDE limits in comparison to non-local limits.
\vspace{\baselineskip}

\noindent{\bf Acknowledgements}
The authors are grateful to Ian Tice and Giovanni Leoni for valuable insights and references.
 The authors are thankful to Christopher Sogge and Steve Zelditch for useful background informtion.
DS acknowledges the support of  the National Science Foundation under the  grant DMS 1516677.
The authors are also grateful to the Center for Nonlinear Analysis (CNA) and Ki-Net (NSF Grant  RNMS11-07444). MMD and AMS are supported by AFOSR Grant FA9550-17-1-0185 and ONR Grant N00014-17-1-2079.
MT is grateful to the Cantab Capital Institute for the Mathematics of Information (CCIMI).

\normalem
\bibliography{references}
\bibliographystyle{siamplain}
\ULforem

\section{Appendix}
\label{sec:A}

\subsection{Function Spaces}
\label{app:funsp}

Here we establish the equivalence between the spectrally defined Sobolev spaces, $\cH^s(\Omega)$ and the standard Sobolev spaces. %To identify the proper subspace 

We denote by 
\[H^2_N(\Omega) = \left\{ u \in H^2(\Omega) \::\: \frac{\partial u}{\partial n} =0
\te{ on } \partial \Omega \right\} \]
the domain of $\cL$. Analogously we denote by $H^{2m}_N(\Omega)$ the domain of $\cL^m$, that is 
\[H^{2m}_N(\Omega) = \left\{ u \in H^{2m}(\Omega) \::\: \frac{\partial \cL^r u}{\partial n} =0 \te{ for all } 0\leq r \leq m-1
\te{ on } \partial \Omega \right \} \]
Finally we let $H^{2m+1}_N(\Omega)  = H^{2m+1}(\Omega) \cap H^{2m}_N(\Omega)$. 

For $m \geq 0$ and $u,v \in H^{2m+1}_N(\Omega)$   let  $\langle u,v \rangle_{2m+1, \mu} = \int_\Omega \nabla \cL^m u \cdot \nabla \cL^m v \rho^2 dx$ and for $u,v\in H^{2m}_N(\Omega)$ let $\langle u,v\rangle_{2m,\mu}=\int_\Omega (\cL^m u) (\cL^m v) \rho \, \dd x$. 
We note that on the $L^2_\mu$ orthogonal complement of the constant function $1$, 
$\langle \tacka , \tacka \rangle_{2m+1, \mu}$ defines an inner product, which due to Poincar\'e inequality is equivalent to the standard inner product on $H^{2m+1}(\Omega)$. We also note that 
$\langle \varphi_k, \varphi_k \rangle_{2m+1, \mu} = \lambda_k^{2m+1}$, where we recall that $\varphi_k$ is unit eigenvector of $\cL$ corresponding to $\lambda_k$.

\begin{lemma} \label{lem:HiscH}
Under Assumptions \ref{a:omega} - \ref{a:rho}, for any integer $s \geq 0$
\[ H_N^s(\Omega) = \cH^s(\Omega) \]
and the associated inner products  $\langle \tacka, \tacka \rangle_{s,\mu}$  and $\llangle \tacka, \tacka \rrangle_{s,\mu}$ are equivalent on the $L^2_\mu$ orthogonal complement of the constant function. 
\end{lemma}

\begin{proof}
For $s=0$, $H^0_N = L^2$ by definition and $\cH^0 = L^2$ by the fact that $\{\varphi_k \::\: k=1, \dots \}$ is an orthonormal basis.

To show the claim for $s=1$, we recall that $\int \nabla \varphi_k \cdot \nabla \varphi_j  \rho^2 dx = 
\int \varphi_k \cL \varphi_j \rho dx = \lambda_k \delta_k^j$. Therefore $\left \{   \frac{\varphi_k}{\sqrt \lambda_k} \::\: k \geq 1 \right \}$ is an orthonormal basis of the orthogonal complement of the constant function, $1^\perp$,   in $H^1_N$ with respect to \mt{the} inner product $(u,v) = \int \nabla_u \cdot \nabla v \rho^2 dx$ which is equivalent to the standard inner product  of $H^1_N$  on $1^\perp$. Since an expansion in the basis $\{\varphi_k\}_k$ is unique, this implies that for any $u \in H^1_N = H^1$ the series $\sum_k a_k \varphi_k$ converges in $H^1$ to $u$.  Consequently if $u \in H^1_N$ then $\infty> \int |\nabla u|^2 \rho^2 dx = \int | \sum_k a_k \nabla \varphi_k|^2 \rho^2 dx = \sum_k a_k^2 \lambda_k$ which implies that $u \in \cH^1$. So $H^1_N \subseteq \cH^1$. 

On the other hand, if $u \in \cH^1$ then $u =  \sum_k a_k \varphi_k$ with $\sum_k \lambda_k a_k^2 < \infty$. Therefore $u = \bar u + \sum_{k=2}^\infty a_k \sqrt{\lambda_k} \, \frac{\varphi_k}{\sqrt{\lambda_k}}$, where $\bar u$ is the average of $u$. 
Since $\frac{\varphi_k}{\sqrt{\lambda_k}}$ are orthonormal in scalar product with topology equivalent to $H^1$, the series converges in $H^1$. Therefore $u \in H^1 = H^1_N$. 

Assume now that the claim holds for all integers less than $s$.
We split the proof of the induction step into two cases: \\
\emph{Case $1^\circ$} Consider $s$ even; that is $s=2m$ for some integer $m>0$. 

Assume $u \in H_N^{2m}$. Then $\nabla \cL^r u \cdot \vec{n} = 0$ on $\partial \Omega$ for all $r<m$. 
By the induction hypothesis $\sum_k \lambda_k^{2m-1} a_k^2 < \infty$. Since $\cL$ is a continuous operator from $\cH^2$ to $L^2$
% since $\frac{1}{\lambda_k}{varphi_k$ are a orthonormal basis of $1^\perp$ in $\cH^2$ and the operator $\CL:\cH^2 \to L^2_\mu$ is bounded on the basis.
 one obtains by induction that 
  $\cL^{m-1} u= \sum_k a_k \cL^{m-1} \varphi_k = \sum a_k \lambda_k^{m-1} \varphi_k$. 
Let $v = \cL^{m-1}u$. By assumption $v \in H^2_N$. By above $v = \sum_k  a_k \lambda_k^{m-1} \varphi_k$.
 
Since $\varphi_k$ is solution of $\cL \varphi_k = \lambda_k \varphi_k$
\[ \langle \cL \varphi_k, v \rangle_\mu = \langle \lambda_k \varphi_k, v \rangle_\mu. \]
Using that $v \in H^2$, $\nabla v \cdot \vec{n} = 0$ on $\partial \Omega$ and integration by parts we obtain
\[ \langle  \varphi_k, \cL v \rangle_\mu = \langle \lambda_k \varphi_k, \sum_j a_j \lambda_j^{m-1} \varphi_j  \rangle_\mu = \lambda_k^m a_k. \]
Given that $\cL v$ is an $L^2_\mu$ function, we conclude that $\cL v = \sum_k \lambda_k^m a_k \varphi_k$. Therefore $\sum_k \lambda_k^{2m} a_k^2 < \infty$ and hence $u \in \cH^{2m}$.

To show the opposite inclusion, consider $u \in \cH^{2m}$. Then $u = \sum_k a_k \varphi_k$ and
$\sum_k \lambda_k^{2m} a_k^2 < \infty$. By induction step we know that $u \in H^{2m-2}_N$ and thus $v = \cL^{m-1} u \in L^2$. 
We conclude as before that $v = \sum_k \lambda_k^{m-1} a_k \varphi_k$. Let $b_k = \lambda_k^{m-1} a_k$. Assumptions on $u$ imply $\sum_k \lambda_k^2 b_k^2 < \infty$. 
Arguing as above in the case $s=1$ we conclude that the series converges in $H^1$ and that 
$\nabla v = \sum_k b_k \nabla \varphi_k$. Combining this with the fact  that $ \cL \varphi_k = \lambda_k \varphi_k$ in $\Omega$ for all $k$ implies that $v$ is a weak solution of 
\begin{align*}
\cL v & = \sum_k \lambda_k b_k \varphi_k \quad \te{ in } \Omega, \\
 \frac{\partial v}{\partial n}& = 0 \;\; \te{ on } \partial \Omega. 
\end{align*} 
Since RHS of the equation is in $L^2$ and $\partial \Omega$ is $C^{1,1}$, by elliptic regularity 
\cite{Grisvard}, $v \in H^2$ and $\|v\|_{H^2}^2 \leq C(\Omega, \rho) \sum_k b_k^2 \lambda_k^2$. Furthermore  $v$ satisfies the Neumann boundary condition and thus $v \in H^2_N$. 
\medskip

\emph{Case $2^\circ$} Consider $s$ odd; that is $s=2m+1$ for some integer $m>0$. 
Assume $u \in H^{2m+1}_N$. Let $v = \cL^m u$. Then $v \in H^1$. The result now follows 
analogously to the case $s=1$. If $u \in \cH^{2m+1}$ then, $u=\sum_k a_k \varphi_k$ with 
$\sum_k \lambda_k^{2m+1} a_k^2 < \infty$. By induction hypothesis,
$v=\cL^{m-1}u \in H^{1}_N$ and $v = \sum_k b_k \varphi_k$ where $b_k = \lambda^{m-1} a_k$. Thus
$ \sum_k \lambda_k b_k^2 < \infty$ and the argument proceeds as in the case $s=1$. 
\medskip

Proving the equivalence of inner products is straightforward. 
\end{proof}
\medskip

We now present the proof of Lemma \ref{lem:emb_frac}. 
\begin{proof}[Proof of  Lemma \ref{lem:emb_frac}]
If $s$ is \mt{an} integer the claim follows form Lemma \ref{lem:HiscH} and Sobolev embedding theorem.
Assume $s = m+\theta$ for some $\theta \in (0,1)$. 
Since $\Omega$ is Lipschitz, by extension theorem of Stein (Leoni \cite{Leoni} 2nd edition, Theorem 13.17)
there is a bounded  linear extension mapping $E_m: H^m(\Omega) \to H^m(\R^d)$ such that
$E_m(f)|_\Omega = f$. From the construction (see remark 13.9 in \cite{Leoni}) it follows that 
$E_m$ and $E_{m+1}$ agree on smooth functions and thus $E_{m+1} = {E_m}|_{H^m(\Omega)}$. 
Therefore, by Theorem 16.12 in Leoni's book (or Lemma 3.7 of Abels \cite{Abels}) 
% Ables Lemma 3.7 = Leoni 16.12 
$E_m$ provides a bounded mapping from the interpolation space $[H^m(\Omega), H^{m+1}(\Omega)]_{\theta,2} \to [H^m(\R^d), H^{m+1}(\R^d)]_{\theta,2}$.
As discussed above the statement of Lemma \ref{lem:emb_frac} $\,\cH^{m+\theta}(\Omega) =[\cH^m(\Omega), \cH^{m+1}(\Omega)]_{\theta,2}$.
By Lemma \ref{lem:HiscH}, $[\cH^m(\Omega), \cH^{m+1}(\Omega)]_{\theta,2}$ embeds into $[H^m(\Omega), H^{m+1}(\Omega)]_{\theta,2}$. Furthermore,
we use that, see Abels  \cite{Abels} Corollary 4.15, 
$[H^m(\R^d), H^{m+1}(\R^d)]_{\theta,2} = H^{m+\theta}(\R^d)$. 
Combining these facts yields the existence of an bounded, linear, extension mapping  $ \cH^{m+\theta}(\Omega) \to  H^{m+\theta}(\R^d)$.
The results (i) and (ii) follows by the Sobolev embedding theorem. 
\end{proof}
\subsection{Passage from Discrete to Continuum}
\label{ssec:Background:Passage}

There are two key tools we use to pass from the discrete to continuum limit.
The first is $\Gamma$-convergence.
$\Gamma$-convergence was introduced in the 1970's by De Giorgi as a tool for studying sequences of variational problems.
More recently this methodology has been applied to study the large data limits of variational problems that arise from statistical inference, e.g.~\cite{trillos,trillos2,thorpe16,thorpe17,thorpe15}.
Accessible introductions to $\Gamma$-convergence can be found in~\cite{braides02,dalmaso93}

The $\Gamma$-convergence methodology provides a notion of convergence of functionals that captures the behaviour of minimizers.
In particular the minimizers converge along a subsequence to a minimizer of the limiting functional.
In our setting, the objects of interest are functions on discrete domains and hence it is not immediate how one should define convergence.
This brings us to our second key tool.
Recently a suitable topology has been identified to characterize the convergence of discrete to continuum using an optimal transport framework~\cite{trillos2}.
The main idea is, given a discrete function $u_n:\Omega_n \to \mathbb{R}$ and a continuum function $u:\Omega \to \R$, to
include the measures with respect to which they are defined in the comparison. Namely, one can think of the function $u_n$ as belonging to the $L^p$ space over the empirical measure  $\mu_n = \frac{1}{n}\sum_{i=1}^n \delta_{x_i}$ and $u$ belonging to the $L^p$ space over the measure $\mu$. 
One defines a continuum function $\tilde{u}_n:\Omega\to \mathbb{R}$ by $\tilde{u}_n = u_n\circ T_n$ where $T_n:\Omega_n\to \Omega$ is a measure preserving map between $\mu$ and $\mu_n$.
One then compares $u_n$ and $\tilde{u}_n$ in the $L^p$ distance, and simultaneously compares $T_n$ and identity. In other words one considers  both the difference in values and the how far the matched points are. 
%One chooses $T_n$ by considering Monge's optimal transport problem.
%It can be shown that this idea defines a metric space which is known as $TL^p$.
We give a brief overview of $\Gamma$-convergence and the $TL^p$ space.

\subsubsection{A Brief Introduction to \texorpdfstring{$\Gamma$}{Gamma}-Convergence}
\label{sssec:Background:Passage:GammaConv}

We present the definition of $\Gamma$-convergence in terms of an abstract topology. In the next section we will discuss what topology we will use in our results.
For now, we simply point out that the space $\X$ needs to be general enough to include functions defined  with respect to different measures. 
%Our choice of topology will therefore be on couples of functions and measures, i.e. $\X$ is a suitable space of pairs $(f,\mu)$ where $f$ is measurable with respect to $\mu$.

\begin{definition}
\label{def:Background:Passage:GammaConv:GammaConv}
Given a topological space $\X$, we say that a sequence of functions $F_n:\X\to \R\cup\{+\infty\}$ $\Gamma$-converges to $F_\infty:\X\to\R\cup\{+\infty\}$, and we write $F_\infty=\Glim_{n\to \infty} F_n$, if the following two conditions hold:
\begin{itemize}
\item (the liminf inequality) for any convergent sequence $u_n\to u$ in $\X$
\[ \liminf_{n\to\infty} F_n(u_n) \geq F_\infty(u); \]
\item (the limsup inequality) for every $u\in\X$ there exists a sequence $u_n$ in $ \X$ with $u_n\to u$ and
\[ \limsup_{n\to \infty} F_n(u_n) \leq F_\infty(u). \]
\end{itemize}
\end{definition}

In the above definition we also call any sequence $\{u_n\}_{n=1, \dots}$ that satisfies the limsup inequality a recovery sequence.
The justification of $\Gamma$-convergence as the natural setting to study sequences of variational problems is given by the next proposition.
The proof can be found in, for example,~\cite{braides02}.

\begin{proposition}
\label{prop:Background:Passage:GammaConv:MinConv}
Let $F_n,F_\infty:\X\to \mathbb{R}\cup\{+\infty\}$.
Assume that $F_\infty$ is the $\Gamma$-limit of $F_n$ and the sequence of minimizers $\{u_n\}_{n=1,\dots}$ of $F_n$ is precompact.
Then
\[ \lim_{n\to \infty} \min_\X F_n = \lim_{n\to\infty} F_n(u_n) = \min_\X F_\infty \]
and furthermore, any cluster point $u$ of $\{u_n\}_{n=1,\dots}$ is a minimizer of $F_\infty$.
\end{proposition}

%Of course when $F_\infty$ has a unique minimizer then the theorem implies $u_n\to u$.
%The theorem is also true for for sequences of approximate minimizers, i.e. if $F_n(u_n)\leq \inf F_n + \delta_n$ with $\delta_n\to 0$ and $u_n$ is precompact then the conclusions of 
%Proposition~\ref{prop:Background:Passage:GammaConv:MinConv} remain true.

Note that $\Glim_{n\to\infty} F_n = F_\infty$ and $\Glim_{n\to\infty} G_n = G_\infty$ \mt{do} not imply $F_n+G_n$ $\Gamma$-converges to $G_\infty+F_\infty$.
Hence, in order to build optimization problems by considering individual terms it is not enough, in general, to know that each term $\Gamma$-converges.
In particular, we consider using the quadratic form $J_n^{(\alpha,\tau)}$ as a prior and adding fidelity terms, e.g.
\[ \Jk^{(n)}(u) = J_n^{(\alpha,\tau)}(u) + \Phi^{(n)}(u). \]
% Dejan: I did not put \rn above as it is not relevant here. 
%where $J_n(u)$ is defined in~\eqref{eq:Background:Discrete:Jn}.
We show that, with probability one, $\Glim_{n\to\infty} J_n^{(\alpha,\tau)} = J_\infty^{(\alpha,\tau)}$. % where $J_\infty$ is given by~\eqref{eq:Background:Cont:Jinfty}.
In order to show that $\Jk^{(n)}$ $\Gamma$-converges it suffices to show that $\Phi^{(n)}$ converges
along any sequence $(\mu_n, u_n)$ along which $ J_n^{(\alpha,\tau)}(u_n)$ is finite. This is similar to the notion of continuous convergence, which is typically used  ~\cite[Proposition 6.20]{dalmaso93}.
However we note that $\Phi^{(n)}$ does not converge continuously since as a functional on $TL^p(\Omega)$
it takes the value infinity whenever the measure considered is not $\mu_n$.

%A key property that we will take advantage of is that $\Gamma$-convergence is stable to continuous perturbations.
%The notion of convergence we  use for $\Phi^{(n)}$ is the continuous convergence.

%\begin{definition}
%\label{def:Background:Passage:GammaConv:CtsConv}
%Let $(\X,d_\X)$ be a metric space.
%We say that $F_n:\X\to\R$ continuously converges to $F:\X\to\R$ if for all $u\in\X$ and sequences $u_n\to u$ in $\X$ we have that $F_n(u_n)\to F(u)$. % and $v$ satisfying $d_\X(u,v)<\delta$.
%\end{definition}
%
%%One can also define a notion of continuously convergent in topological spaces (which avoids the use of a metric) but for our purposes it is enough to restrict ourselves to this setting.
%%The following result is a consequence of $\Glim_{n\to \infty} (F_n+G_n) = F_\infty+G_\infty$ where $\Glim_{n\to\infty} F_n=F_\infty$ and $G_n$ continuously converges $G$, see~\cite{dalmaso93}, and Proposition~\ref{prop:Background:Passage:GammaConv:MinConv}.
%The following result can be found in~\cite[Proposition 6.20]{dalmaso93}.
%
%\begin{proposition}
%\label{prop:Background:Passage:GammaConv:GammaStab}
%If $G_n:\X\to \R$ continuously converges to $G:\X\to\R$ and $F_n:\X\to\R\cup\{+\infty\}$ $\Gamma$-converges to $F:\X\to \R\cup\{+\infty\}$ then $F_n+G_n$ $\Gamma$-converges to $F+G$.
%\end{proposition}

\subsubsection{The \texorpdfstring{$TL^p$}{TLp} Space}
\label{sssec:Background:Passage:TLp}

In this section we give an overview of the topology that was introduced in~\cite{trillos2} to compare sequences of functions on graphs.
We motivate the topology in the setting considered in this paper.
Recall that $\mu\in\cP(\Omega)$ has density $\rho$ and that $\mu_n$ is the empirical measure.
Given $u_n:\Omega_n\to\R$ and $u:\Omega\to \R$ the idea is to consider pairs $(\mu, u)$ and $(\mu_n, u_n)$ and compare them as such.  We define the metric as follows.
\begin{definition}
\label{def:Background:Passage:TLp:TLpSpace}
Given a bounded  open set $\Omega$, the space $TL^p(\Omega)$ is the space of pairs $(\mu,f)$ such that $\mu$  is a probability measure supported on $\Omega$ and  $f\in L^p(\mu)$.
The metric on $TL^p$ is defined by
\[ d_{TL^p}((f,\mu),(g,\nu)) = \inf_{\pi\in\Pi(\mu,\nu)} \left( \int_{\Omega\times \Omega} |x-y|^p + |f(x)-g(y)|^p \, \mathrm{d} \pi(x,y) \right)^{\frac{1}{p}}. \]
\end{definition}
Above $\Pi(\mu, \nu)$ is the set of transportation plans (i.e. couplings) between $\mu$ and $\nu$; that is the set of probability measures on $\Omega \times \Omega$ whose first marginal is $\mu$ and second marginal in $\nu$.

For a proof that $d_{TL^p}$ is  a metric on $TL^p$ see \cite[Remark 3.4]{trillos2}.

To connect the $TL^p$ metric defined above with the ideas discussed previously we make several observations.
The first is that when $\mu$ has a continuous density then one can consider transport maps $T:\Omega\to \Omega_n$ that satisfy $T_{\#}\mu = \mu_n$ instead of transport plans $\pi\in\Pi(\mu,\nu)$.
Hence, one can show that
\[ d_{TL^p}((f,\mu),(g,\nu)) = \inf_{T \, : \, T_\#\mu = \nu} \left( \|\Id - T\|_{L^p(\mu)}^p + \|f - g\circ T \|_{L^p(\mu)}^p \right)^{\frac{1}{p}}. \]

In the setting when we compare $(\mu, u)$ and $(\mu_n, u_n)$ the second term is nothing but 
$\| u - \tilde{u}_n\|_{L^p(\mu)}^p$, where $\tilde u_n = u_n \circ T_n$ and $T_n:\Omega \to \Omega_n$ is a transport map. 

We note that for a sequence $(\mu_n, u_n)$ to $TL^p$ converge to $(\mu, u)$ it is necessary that 
$\|\Id - T\|_{L^p(\mu)}$ converges to zero, in other words it is necessary that the measures $\mu_n$ converge to $\mu$ in $p$-optimal transportation distance. We recall that since $\Omega$ is bounded this is equivalent to weak convergence of $\mu_n$ to $\mu$. 
%\[ d_{\mathrm{OT}^p}(\mu_n,\mu):= \inf_{T\, :\, T_\#\mu=\mu_n} \left( \int_\Omega |x-T(x)|^p \, \mathrm{d} \mu(x) \right)^{\frac{1}{p}} \to 0. \]
Assuming this to be the case, we call  any sequence of transportation maps $T_n$ satisfying $(T_n)_{\#}\mu = \mu_n$ and $\|\Id - T_n\|_{L^p(\mu)}\to 0$ a \emph{stagnating} sequence.  
%\[ (T_n)_{\#}\mu = \mu_n \quad \quad \text{and} \quad \quad \int_\Omega |x-T_n(x)|^p \, \mathrm{d} \mu(x) \to 0. \]
One can then show (see~\cite[Proposition 3.12]{trillos2}) that convergence in $TL^p$ is equivalent to weak* convergence of measures $\mu_n$ to $\mu$  and convergence $\| u - u_n\circ T_n\|_{L^p(\mu)} \to 0$
%\[ \int_\Omega | u(x) - u_n(T_n(x)) |^p \, \mathrm{d} \mu(x) \to 0 \]
for arbitrary sequence of stagnating transportation maps. Furthermore if  convergence  $\| u - u_n\circ T_n\|_{L^p(\mu)} \to 0$ holds for a sequence of  stagnating transportation maps it holds for every sequence of  stagnating transportation maps.
\medskip

The intrinsic scaling of the graph Laplacian, i.e. the parameter $\eps_n$, depends on how far one needs to move ``mass" to couple $\mu$ and $\mu_n$, that is on upper bounds on transportation distance between $\mu$ and $\mu_n$. 
The following result can be found in~\cite{garciatrillos15b}, the lower bound in the scaling of $\eps=\eps_n$ is so that there exists a stagnating sequence of transport maps with $\frac{\|T_n-\mathrm{Id}\|_{L^\infty}}{\eps_n}\to 0$.

\begin{proposition}
\label{prop:Background:Passage:TLp:TScaling}
Let $\Omega\subset\R^d$ with $d\geq 2$ be open, connected and bounded with Lipschitz boundary.
Let $\mu\in\cP(\Omega)$ with density $\rho$ which is bounded above and below by strictly positive constants.
Let $\Omega_n=\{x_i\}_{i=1}^n$ where $x_i\iid \mu$ and let $\mu_n=\frac{1}{n}\sum_{i=1}^n\delta_{x_i}$ be the associated empirical measure.
Then, there exists $C>0$ such that, with probability one, there exists a sequence of transportation maps $T_n:\Omega\to \Omega_n$ that pushes $\mu$ onto $\mu_n$ and such that
\[ \limsup_{n\to\infty} \frac{\|T_n-\mathrm{Id}\|_{L^\infty(\Omega)}}{\delta_n} \leq C \]
where
\[ \delta_n = \left\{ \begin{array}{ll} \frac{(\log n)^{\frac{3}{4}}}{\sqrt{n}} & \text{if } d = 2 \\ \left(\frac{\log n}{n}\right)^{\frac{1}{d}} & \text{if } d\geq 3. \end{array} \right. \]
\end{proposition}

\subsection{Estimates on Eigenvalues of the Graph Laplacian} \label{ssec:GLest}

The following lemma is nonasymptotic and holds for all $n$. However we will use it in the asymptotic regime and note that our assumptions on $\eps$, \eqref{eq:LimitThmDir:epsSca}, and 
results of Proposition \ref{prop:Background:Passage:TLp:TScaling} ensure that the assumptions of the lemma are satisfied. 

\begin{lemma}
Consider the operator $A^{(n)}$ defined in \eqref{eq:A} for $\alpha=1$ and $\tau\geq 0$. 
Assume  that $d_{\mathrm{OT}^\infty}(\mu_n,\mu)< \eps$.
Then the spectral radius $\lambda_{max}$ of $A^{(n)}$ is bounded by $C \frac{1}{\eps^2} +\tau^2$ where $C>0$ is independent of $n$ and $\eps$. 

Let $R>0$ be such that $\eta(3R)>0$. Assume  that $d_{\mathrm{OT}^\infty}(\mu_n,\mu)<  R \eps$. Then there exists $c>0$, independent of $n$ and $\eps$, such that 
$\lambda_{max} > c \frac{1}{\eps^2}  + \tau^2$.  
\end{lemma}

\begin{proof}
Let $\overline \eta(x) = \eta((|x|-1)_+)$. Note that $\overline  \eta \geq \eta(|\cdot|)$ and that since $\eta$ is decreasing and integrable $\int_{\R^d} \overline \eta(x) dx <\infty$. 

Let $T$ be the $d_{OT^\infty}$ transport map from $\mu$ to $\mu_n$. By assumption $\|T_n(x)- x\| \leq \eps$ a.e. By definition of $A^{(n)}$
\[ \lambda_{max} =   \sup_{\|u\|_{L^2_{\mu_n}} =1 } \langle u, A^{(n)} u \rangle_{\mu_n}  = \tau^2 + \sup_{\|u\|_{L^2_{\mu_n}} =1 } \langle u, s_n L u \rangle_{\mu_n} \]
We estimate 
\begin{align*}
 \sup_{\|u\|_{L^2_{\mu_n}} =1 } \langle u, s_n L u \rangle_{\mu_n} 
\leq  &  \sup_{\frac{1}{n} \sum_{i=1}^n u_i^2 = 1} \frac{4}{\sigma_\eta}  \sum_{i,j} \frac{1}{n^2 \eps^{d+2}} \eta \left(\frac{ |x_i - x_j|}{\eps} \right) (u_i^2 + u_j^2) \\
\lesssim &\sup_{\frac{1}{n} \sum_{i=1}^n u_i^2 = 1}\; \sum_{i=1}^n \sum_{j=1}^n \frac{1}{n^2 \eps^{d+2}} \eta \left(\frac{ |x_i - x_j|}{\eps} \right) u_i^2   \\
= &\sup_{\frac{1}{n} \sum_{i=1}^n u_i^2 = 1}\; \frac{1}{n \eps^{d+2}} \sum_{i=1}^n  u_i^2   \int_\Omega \eta \left(\frac{ |x_i - T(x)|}{\eps} \right) d\mu(x) \\
\leq & \sup_{\frac{1}{n} \sum_{i=1}^n u_i^2 = 1}\; \frac{1}{n \eps^{d+2}} \sum_{i=1}^n  u_i^2   \int_\Omega \overline \eta \left(\frac{ x_i - x}{\eps} \right)  d\mu(x) \\
\lesssim & \frac{1}{\eps^2}  \int_{\R^d} \overline \eta (z)  dz \lesssim \frac{1}{\eps^2}.
\end{align*}
Above $\lesssim$ means $\leq$ up to a factor independent of $\eps$ and $n$. 

To prove the second claim of the lemma consider $v = \sqrt{n} \delta_{x_i}$, a singleton concentrated at an arbitrary $x_i$, that is $v_i = \sqrt n$ and $v_j=0$ for all $j \neq i$. Then $\|v\|_{L^2_{\mu_n}} =1$. Using that for a.e. $x \in B(x_i,2\eps R)$, $\; |x_i - T(x)| \leq 3 \eps R$ we estimate:
\begin{align}
 \sup_{\|u\|_{L^2_{\mu_n}} =1 } \langle u, s_n L u \rangle_{\mu_n} &
\geq  \langle v,s_n Lv\rangle_{\mu_n} \notag \\
 & \gtrsim   \sum_{j\neq i} \frac{n}{n^2 \eps^{d+2}} \eta \left(\frac{ |x_i - x_j|}{\eps} \right) \notag \\
 & =  \frac{1}{ \eps^{d+2}} \int_{\Omega\setminus T^{-1}(x_i)} \eta \left(\frac{ |x_i - T(x)| }{\eps} \right) d\mu(x) \notag \\ 
 & \geq \frac{1}{ \eps^{d+2}} \int_{B(x_i, 2\eps R)\setminus B(x_i,\eps R)} \eta(3R) d \mu(x) 
\gtrsim  \frac{1}{\eps^2} \label{eq:DiracLowerEnergy}
\end{align}
which implies the claim.
\end{proof}

An immediate corollary of the claim is the characterization of the energy of a singleton. For any $\alpha \geq 1$ and $\tau \geq 0$. 
\begin{equation} \label{ene_sing}
J_n^{(\alpha,\tau)} (\delta_{x_i}) \sim \frac{1}{n} \left(\frac{1}{\eps^2_n}+\tau^2\right)^\alpha \sim \frac{1}{n\eps_n^{2\alpha}}. 
\end{equation}
The upper bound is immediate from the first part of the lemma, while the lower bound follows from the second part of the lemma via Jensen's inequality. 
Namely, $(\lambda_k^{(n)},q_k^{(n)})$ be eigenpairs of $L$ and let us expand $\delta_{x_i}$ in the terms of $q_k^{(n)}$: i.e. $\delta_{x_i} = \sum_{k=1}^n a_k q_k^{(n)}$ where $\sum_k a_k^2 = \|\delta_{x_i} \|^2_{L^2_{\mu_n}} = \frac{1}{n}$. 
We know that 
%$ \langle \delta_{x_i} , s_n L \delta_{x_i}  \rangle_{\mu_n} = 
$\sum_k \lambda_k^{(n)} a_k^2 \gtrsim \frac{1}{n \eps_n^2 s_n} \sim 1$, from~\eqref{eq:DiracLowerEnergy} (using the expansion~\eqref{eq:JnEigForm} and noting that $v=\sqrt{n}\delta_{x_i}$ in~\eqref{eq:DiracLowerEnergy}). 
Hence 
\[ J_n^{(\alpha,\tau)} (\delta_{x_i})  = \frac{1}{2n} \sum_{k=1}^n \left( s_n\lambda_k^{(n)} + \tau^2\right)^\alpha na_k^2 
\geq   \frac{1}{2n}\left( ns_n\sum_{k=1}^n \lambda_k^{(n)} a_k^2 + \tau^2\right)^\alpha \geq 
\frac{1}{2n} \left( \frac{1}{\eps_n^2} + \tau^2\right)^\alpha. \]

\subsection{The Limiting Quadratic Form}
\label{ssec:LQF}

Here we prove Theorem \ref{thm:LimitThmDir:LimitThmDir}. The key tool is to use
spectral decomposition of the relevant quadratic forms, and to rely on the
limiting properties of the eigenvalues and eigenvectors of $L$ established
in \cite{trillos}.

Let $(q_k^{(n)},\lambda_k^{(n)})$ be eigenpairs of $L$ with eigenvalues $\lambda_k$ ordered so that
\[ 0 = \lambda_1^{(n)} \leq \lambda_2^{(n)} \leq \lambda_3^{(n)} \leq \dots \lambda_n^{(n)} \]
where  $\lambda_1^{(n)} < \lambda_2^{(n)}$ provided
that the graph $G$ is connected.
We extend $F:\bbR \mapsto \bbR$ to a matrix-valued function $F$ via
$F(L) = Q^{(n)}(\Lambda_F^{(n)}) (Q^{(n)})^*$ where $Q^{(n)}$ is the matrix
with columns $\{q_k^{(n)}\}_{k=1}^{n}$ and $\Lambda_F^{(n)}$ is the diagonal matrix
with entries $\{F(\lambda_i^{(n)})\}_{i=1}^{n}$.
For constants $\alpha\geq 1$, $\tau\geq 0$ and a scaling factor $s_n$, given by~\eqref{eq:sn}, we recall the definition of the precision matrix $\disP^{(n)}$ is $\disP^{(n)}=(s_n L+\tau^2 I)^{\alpha}$ and the fractional Sobolev energy $J^{(\alpha,\tau)}_n$ is
%\begin{equation}
%\label{eq:Background:Discrete:Jn}
\[ J^{(\alpha,\tau)}_n : L^2_{\mu_n} \mapsto [0,+\infty), \quad \quad J^{(\alpha,\tau)}_n(u) = \frac{1}{2} \langle u, \disP^{(n)} u\rangle_{\mu_n}. \]
%\end{equation}
%The scale factor $s_n$ will be defined in the sequel.
Note that
\begin{equation} \label{eq:JnEigForm}
J_n^{(\alpha,\tau)}(u) = \frac{1}{2} \sum_{k=1}^{n} (s_n \lambda_k^{(n)}+\tau^2)^\alpha \langle u,q_k^{(n)} \rangle^2_{\mu_n}.
\end{equation}
When showing $\Gamma$-convergence, all functionals are considered as functionals on the $TL^p$ space. When evaluating  $J_n^{(\alpha,\tau)}$ at $(\nu,u)$ we consider it infinite for any measure $\nu$ other than $\mu_n$, and having the value  $J_n^{(\alpha,\tau)}(u)$ defined above if $\nu=\mu_n$.

We let $(q_k,\lambda_k)$ for $k=1,2,\dots$ be eigenpairs of $\cL$ ordered so that
\[ 0 = \lambda_1 \leq \lambda_2 \leq \lambda_3 \leq \dots. \]
We extend $F:\bbR \mapsto \bbR$ to an operator valued function  via the identity
$F(\cL) = \sum_{k=1}^\infty F(\lambda_k) \langle u,q_k\rangle_\mu q_k$.
For constants $\alpha\geq 1$ and $\tau\geq 0$ we recall the definition of the precision operator $\ctsP$ as $\ctsP = (\cL+\tau I)^\alpha$ and the continuum Sobolev energy $J^{(\alpha,\tau)}_\infty$ as
%\begin{equation}
%\label{eq:Background:Cont:Jinfty}
\[ J^{(\alpha,\tau)}_\infty: L^2_{\mu} \mapsto \mathbb{R}\cup\{+\infty\}, \quad \quad J^{(\alpha,\tau)}_\infty(u) = \frac{1}{2} \langle u, \ctsP u\rangle_{\mu}. \]
%\end{equation}
Note that the Sobolev energy can be written
%\begin{equation} 
%\label{Jinf_spectral}
\[ J^{(\alpha,\tau)}_\infty(u) = \frac{1}{2} \sum_{k=1}^\infty (\lambda_k+\tau^2)^\alpha \langle u,q_k \rangle_{\mu}^2. \]
%\end{equation}

%\noindent{\em Proof of Theorem \ref{thm:LimitThmDir:LimitThmDir}}

\begin{proof}[Proof of Theorem \ref{thm:LimitThmDir:LimitThmDir}]
We prove the theorem in three parts.
In the first part we prove the liminf inequality and in the second part the limsup inequality.
The third part is devoted to the proof of the two compactness results.

\paragraph*{The Liminf Inequality}
Let $u_n\to u$ in $TL^p$, we wish to show that
\[ \liminf_{n\to \infty} J_n^{(\alpha,\tau)}(u_n) \geq J_\infty^{(\alpha,\tau)}(u). \]
By~\cite[Theorem 1.2]{trillos}, if all eigenvalues of $\cL$ are simple, we have with probability one (where the set of probability one can be chosen independently of the sequence $u_n$ and $u$) that $s_n\lambda_k^{(n)}\to \lambda_k$ and $q_k^{(n)}$ converge in $TL^2$  to $q_k$.
If there are eigenspaces of $\cL$ of dimension higher than one then $q_k^{(n)}$ converge along a subsequence in $TL^2$ to eigenfunctions $\tilde q_k$ corresponding to the same eigenvalue as $q_k$. In this case we replace $q_k$ by $\tilde q_k$, which does not change any of the functionals considered.
We note that while eigenvectors in the general case only converge along subsequences, the projections to the relevant spaces of eigenvectors converge along the whole sequence, see  ~\cite[statement 3. Theorem 1.2]{trillos}. To prove the convergence of the functional one would need to use these projections, which makes the proof cumbersome. For that reason in the remainder of the proof we assume that all eigenvalues of $\cL$ are simple, in which case we can express the projections using  the inner product with eigenfunctions.

Since $q_k^{(n)} \to q_k$ and $u_n \to u$ in $TL^2$ as $n \to \infty$, 
 $\langle q_k^{(n)},u_n\rangle_{\mu_n}\to \langle q,u \rangle_\mu$ as $n \to \infty$.

First we assume that $J^{(\alpha,\tau)}_\infty(u)<\infty$.
Let $\delta>0$ and choose $K$ such that
\[ \frac12 \sum_{k=1}^K (\lambda_k+\tau^2)^\alpha \langle u,q_k\rangle_\mu^2 \geq J^{(\alpha,\tau)}_\infty(u) - \delta. \]
Now,
\begin{align*}
\liminf_{n\to \infty} J^{(\alpha,\tau)}_n(u_n) & \geq \liminf_{n\to \infty} \frac12 \sum_{k=1}^K (s_n\lambda_k^{(n)}+\tau^2)^\alpha \langle u_n,q_k^{(n)}\rangle_{\mu_n}^2 \\
 & = \frac12 \sum_{k=1}^K (\lambda_k+\tau^2)^\alpha \langle u_n,q_k \rangle_\mu^2 \\
 & \geq J^{(\alpha,\tau)}_\infty(u) - \delta.
\end{align*}
Let $\delta\to 0$ to complete the liminf inequality for when $J^{(\alpha,\tau)}_\infty(u)<\infty$.
If $J_\infty^{(\alpha,\tau)}(u)=+\infty$ then choose any $M>0$ and find $K$ such that $\frac12 \sum_{k=1}^K (\lambda_k+\tau^2)^\alpha \langle u_n, q_k\rangle_\mu^2 \geq M$, the same argument as above implies that
\[ \liminf_{n\to \infty} J^{(\alpha,\tau)}_n(u_n) \geq M \]
and therefore $\liminf_{n\to \infty} J^{(\alpha,\tau)}_n(u_n) = +\infty$.
\medskip

\paragraph*{The Limsup Inequality} 
As above, we assume for simplicity, that all eigenvalues of $\cL$ are simple. We remark that there are no essential difficulties to carry out the proof in the general case.
%, other than it being more involved due to the need to consider projections. 

Let $u\in L^2_{\mu}$ with $J_\infty^{(\alpha,\tau)}(u) < \infty$ (the proof is trivial if $J_\infty^{(\alpha,\tau)} =\infty$).
Define $u_n\in L^2_{\mu_n}$ by $u_n = \sum_{k=1}^{K_n} \psi_k q_k^{(n)}$ where $\psi_k = \langle u,q_k\rangle_\mu$.  Let $T_n$ be the transport maps from $\mu$ to $\mu_n$ as in Proposition \ref{prop:Background:Passage:TLp:TScaling}. Let $a_k^n = \psi_k q_k^{(n)} \circ T_n$ and $a_k= \psi_k q_k$. By Lemma \ref{lem:a}, there exists a sequence $K_n \to \infty$ such that $u_n$ converges to $u$ in $TL^2$ metric.

%The convergence of eigenvalues, \cite[Theorem 1.2]{trillos}, implies
%\begin{equation} \label{eq:app:LimitThmDir:LimsupPointwise}
%\lim_{n\to \infty} (s_n\lambda_k^{(n)} + \tau^2)^\alpha = (\lambda_k + \tau^2)^\alpha
%\end{equation}
We recall from the  proof of the liminf inequality that 
$\langle q_k^{(n)},u_n\rangle_{\mu_n} \to \langle q_k,u\rangle_\mu $ as $n \to \infty$. Combining with the convergence of eigenvalues, \cite[Theorem 1.2]{trillos}, implies
\[  (s_n\lambda_k^{(n)} + \tau^2)^\alpha \langle u_n,q_k^{(n)} \rangle_{\mu_n}^2 \to (\lambda_k+\tau^2)^\alpha \langle u,q_k\rangle_\mu^2\]
as $n \to \infty$. Taking $a_k^n = (s_n\lambda_k^{(n)} + \tau^2)^\alpha \langle u_n,q_k^{(n)} \rangle_{\mu_n}^2$ and $a_k = (\lambda_k+\tau^2)^\alpha \langle u,q_k\rangle_\mu^2$ and using Lemma \ref{lem:a} implies 
that there exists $\tilde K_n \leq K_n$ converging to infinity such that $\sum_{k=1}^{\tilde K_n} a_k^n \to \sum_{k=1}^\infty a_k$ as $n \to \infty$. Let $\tilde u_n = \sum_{k=1}^{\tilde K_n} \psi_k q_k^{(n)}$. Then $\tilde u_n \to u$ in $TL^2$. Furthermore
$J_n^{(\alpha,\tau)}(\tilde u_n) = \sum_{k=1}^{\tilde K_n} a_k^n $ and $J_\infty^{(\alpha,\tau)}(u) = \sum_{k=1}^\infty a_k$ which implies that $J_n^{(\alpha,\tau)}(\tilde u_n)  \to J_\infty^{(\alpha,\tau)}(u)$ as $n \to \infty$.

\paragraph*{Compactness}
If $\tau >0$ and $\sup_{n\in \bbN} J_n^{(\alpha,\tau)}(u_n)\leq C$ then
\[ \tau^{2\alpha} \|u_n\|_{L^2_{\mu_n}}^2 = \tau^{2\alpha} \sum_{k=1}^n \langle u_n,q_k^{(n)}\rangle_{\mu_n}^2 \leq \sum_{k=1}^n (s_n\lambda_k^{(n)} + \tau^2)^\alpha \langle u_n,q_k^{(n)}\rangle_{\mu_n}^2 \leq C. \]
Therefore $\|u_n\|_{L^2_{\mu_n}}$ is bounded.
Hence in statements 2 and 3 of the theorem we have that $\|u_n\|_{L^2_{\mu_n}}$ and $J^{(\alpha,\tau)}_n(u_n)$ are bounded.
%Since $J_n^{(\alpha,\tau)}(u_n) \geq J_n^{(1,0)}(u_n)$ then the compactness property holds directly from~\cite[Theorem 1.4]{trillos}.
That is there exists $C>0$ such that
\begin{equation} \label{eq:2bounds}
\|u\|_{L^2_{\mu_n}}^2 = \sum_{k=1}^n \langle u_n,q_k^{(n)}\rangle_{\mu_n} \leq C \quad \text{and} \quad s_n^\alpha \sum_{k=1}^n (\lambda_k^{(n)})^\alpha \langle u_n,q_k^{(n)} \rangle_{\mu_n}^2 \leq C.
\end{equation}
We will show there exists $u\in L^2_{\mu}$ and a subsequence $n_m$ such that $u_{n_m}$ converges to $u$ in $TL^2$.

Let $\psi_k^n =  \langle u_n,q_k^{(n)}\rangle_{\mu_n}$ for all $k\leq n$. Due to \eqref{eq:2bounds}
$|\psi_k^n|$ are uniformly bounded. Therefore, by a diagonal procedure, there exists a increasing sequence $n_m \to \infty$ as $m \to \infty$ such that for every $k$, $\psi_k^{n_m}$ converges as $m \to \infty$. Let
$\psi_k = \lim_{m \to \infty}  \psi_k^{n_m}$. By Fatou's lemma, $\sum_{k=1}^\infty |\psi_k|^2 \leq \liminf_{m \to \infty}  \sum_{k=1}^{n_m} |\psi_k^{n_m}|^2 \leq C$. 
Therefore $u := \sum_{k=1}^\infty \psi_k q_k \in L^2_\mu$.   Using Lemma \ref{lem:a} and arguing as in the proof of the limsup inequality we obtain that there exists a sequence $K_m$ increasing to infinity such that $
\sum_{k=1}^{K_m} \psi_k^{n_m} q_k^{(n_m)}$ converges to $u$ in $TL^2$ metric as $m \to \infty$.
To show that $u_{n_m}$ converges to $u$ in $TL^2$, we now only need to show that 
$\| u_{n_m} - \sum_{k=1}^{K_m} \psi_k^{n_m} q_k^{(n_m)} \|_{L^2_{\mu_{n_m}}}$ converges to zero. 
This follows from the fact that
% \| u_{n_m} - \sum_{k=1}^{K_m} \psi_k^{n_m} q_k^{(n_m)} \|^2_{L^2_{\mu_{n_m}}} =
\[ \sum_{k=K_m+1}^{n_m} |\psi_k^{n_m}|^2 \leq \frac{1}{\left(\lambda_{K_m}^{(n_m)}\right)^\alpha}
\sum_{k=K_m+1}^{n_m}  (\lambda_k^{(n_m)})^\alpha |\psi_k^{n_m}|^2 \leq \frac{C}{\left(s_{n_m}\lambda_{K_m}^{(n_m)}\right)^\alpha} \]
using that the sequence of eigenvalues is nondecreasing.
Now since $s_{n_m}\lambda_{K_m}^{(n_m)} \geq s_{n_m}\lambda_K^{(n_m)} \to \lambda_K$ for all $K_m\geq K$, and $\lim_{K\to\infty}\lambda_K=+\infty$ we have that $s_{n_m}\lambda_{K_m}^{(n_m)}\to +\infty$ as $m \to \infty$, hence $u_{n_m}$ converges to $u$ in $TL^2$. 
\end{proof}

\begin{rem}
\label{rem:LimitThmDir:alphaless1}
%Note that the $\Gamma$-convergence in the above theorem holds for any $\alpha>0$.
%The compactness result relied on the bound $J^{(\alpha,\tau)}_n\geq J^{1,0}_n$ which is valid when $\alpha\geq 1$.
%For $\alpha\in (0,1)$ we still expect the result to be true based on compactness of fractional Sobolev spaces in $L^2$, however it is not trivial to show this and since we are primarily concerned with larger choices of $\alpha$ (in particular $\alpha>d/2$) we omit this case.
Note that when $\alpha\geq 1$ the compactness property holds trivially from the compactness property for $\alpha=1$, see~\cite[Theorem 1.4]{trillos}, as $J_n^{(\alpha,\tau)}(u_n) \geq J_n^{(1,0)}(u_n)$.
\end{rem}

\subsection{Variational Convergence of Probit in Labelling Model 1}
\label{ssec:PRLabelMod1}

To prove minimizers of the Probit model in \hyperlink{labmod1}{{\bf Labelling Model 1}} converge we apply Proposition~\ref{prop:Background:Passage:GammaConv:MinConv}. This requires us to show that $\Jp^{(n)}$ $\Gamma$-converges to $\Jp^{(\infty)}$ and the compactness of sequences of minimizers.
Recall that  $\Jp^{(n)} = J_n^{(\alpha,\tau)} + \frac{1}{n} \Phip^{(n)}(\cdot;\gamma)$.
Hence Theorem~\ref{thm:LimitThmDir:LimitThmDir} establishes the $\Gamma$-convergence of the first term. We now show that $\frac{1}{n}\Phip^{(n)}(u_n;y_n;\gamma) \to \Phipo(u;y;\gamma)$ whenever $(\mu_n, u_n) \to (\mu,u)$ in the $TL^2$ sense, which is enough to establish $\Gamma$-convergence. Namely since, by definition,  $J_n^{(\alpha,\tau)}$ applied to an element $(\nu, v) \in TL^p(\Omega)$ is $\infty$ if $\nu \neq \mu_n$ it suffices to consider sequences of the form $(\mu_n, u_n)$ to show the liminf inequality. The limsup
inequality is also straightforward since the the recovery sequence for $J^{(\alpha,\tau)}_\infty$ is also of the form $(\mu_n, u_n)$.

\begin{lemma}
\label{lem:LimitThmOpt:Probit:CC}
Consider domain $\Omega$ and measure $\mu$ satisfying Assumptions~\ref{a:omega}--\ref{a:rho}. Let $x_i\iid \mu$ for $i=1,\dots, n$,   $\Omega_n= \{x_1, \dots, x_n\}$ and $\mu_n$ be  the empirical measure of the sample.
Let $\Omega'$ be an open subset of $\Omega$, $\mu_n'=\mu_n\lfloor_{\Omega'}$ and $\mu'=\mu\lfloor_\Omega$.
Let $y_n\in L^\infty(\mu_n')$ and $y\in L^\infty(\mu')$ and let  $\hat{y}_n\in L^\infty(\mu_n)$ and $\hat{y}\in L^\infty(\mu)$ be their extensions by zero.
Assume 
\[ (\mu_n, \hat{y}_n)\to (\mu, \hat{y}) \quad \te{  in } TL^\infty \te{ as } n \to \infty.   \]
% and $\max\{\sup_n \|y_n\|_{L^\infty(\mu'_n)},\|y\|_{L^\infty(\mu')}\}<\infty$;
Let $\Phip^{(n)}$ and $\Phipo$ be defined by~\eqref{eq:Background:Discrete:Probit:Phip} and~\eqref{eq:Background:Cont:Probit:Phip} respectively, where $Z'=\{j \, : \, x_j\in \Omega'\}$ and $\gamma>0$ (and where we explicitly include the dependence of $y_n$ and $y$ in $\Phip^{(n)}$ and $\Phipo$).  

Then, with probability one, if $(\mu_n,u_n) \to (\mu,u)$ in $TL^p$ then 
%\begin{equation} \label{eq:Phicon}
\[ \frac{1}{n}\Phip^{(n)}(u_n;y_n;\gamma) \to \Phipo(u;y;\gamma) \quad \te{ as } n \to \infty. \] 
%\end{equation}
%$\frac{1}{n}\Phip^{(n)}(\cdot;y_n)$ continuously converges to $\Phipo(\cdot;y)$ with respect to $TL^p$ topology. 
\end{lemma}

\begin{proof}
Let $(\mu_n,u_n) \to (\mu,u)$ in $TL^p$.
We first note that since $\Psi(uy; \gamma)= \Psi \left( \frac{uy}{\gamma} ;1 \right)$ and since multiplying all functions by a constant does not affect the $TL^p$ convergence, it suffices to consider $\gamma=1$. For brevity, we  omit $\gamma$ in the functionals that follow. 
We have that $\hat{y}_n\circ T_n\to \hat{y}$ and $u_n\circ T_n\to u$.
Recall that
\begin{align*}
\frac{1}{n} \Phip^{(n)}(u_n;y_n) & = \int_{T_n^{-1}(\Omega_n')} \log \Psi(y_n(T_n(x)) u_n(T_n(x))) \, \dd \mu(x) \\
\Phipo(u;y) & = \int_{\Omega'} \log \Psi(y(x) u(x)) \, \dd \mu(x),
\end{align*}
where $\Omega_n'= \{x_i \, : \, x_i\in \Omega', \text{ for } i=1,\dots, n\}$.  Recall also that symmetric difference of sets is denoted by $A \triangle B = (A \setminus B) \cup (B \setminus A)$. It follows that
\begin{align}
\begin{split}
& \left| \frac{1}{n} \Phip^{(n)}(u_n;y_n) - \Phipo(u;y) \right| \leq \left| \int_{\Omega' \triangle T_n^{-1}(\Omega_n')} \log \Psi(\hat{y}(x)u(x) ) \dd \mu(x) \right| \\
& \quad \quad + \left| \int_{T_n^{-1}(\Omega_n')} \log \left(\Psi(y_n(T_n(x)) u_n(T_n(x)) ; \gamma\right) - \log\left( \hat{y}(x) u(x)  \right) \, \dd \mu(x) \right|. \label{eq:LimitThmOpt:Probit:Decomp}
\end{split}
\end{align}
Define
\[ \partial_{\eps_n}\Omega'= \left\{ x \, : \, \mathrm{dist}(x,\partial \Omega')\leq \eps_n \right\}. \]
Then $\Omega' \triangle T_n^{-1}(\Omega_n')\subseteq \partial_{\eps_n}\Omega'$.
Since $\hat{y}\in L^\infty$ and $u\in L^2_\mu$ then $\hat{y}u\in L^2_\mu$ and so by Corollary~\ref{lem:Background:Cont:Probit:PhipL1} $\log\Psi(\hat{y}u )\in L^1$.
Hence, by the dominated convergence theorem
\[ \left| \int_{\Omega' \triangle T_n^{-1}(\Omega_n')} \log \Psi(\hat{y}(x)u(x) ) \dd \mu(x) \right| 
 \leq \int_{\partial_{\eps_n}\Omega'} \left| \log \Psi(\hat{y}(x) u(x)) \right| \, \dd \mu(x) \to 0.
\]

We are left to show that the second term on the right hand side of~\eqref{eq:LimitThmOpt:Probit:Decomp} converges to 0.
Let $F(w,v) = |\log\Psi(w)-\log \Psi(v)|$.
Let $M\geq 1$ and define the following sets
%\begin{align*}
%\mathcal{A}_{n,M}^+ & = \left\{ x\in T_n^{-1}(\Omega_n') \, : \, \hat{y}(x) u(x) \leq -M \right\} \\ 
%\mathcal{A}_{n,M}^- & = \left\{ x\in T_n^{-1}(\Omega_n') \, : \, y_n(T_n(x)) u_n(T_n(x)) \leq -M \right\} \\ 
%\mathcal{B}_{n,M} & = \left\{ x\in T_n^{-1}(\Omega_n') \, : \, \min\{ y_n(T_n(x))u_n(T_n(x)),\hat{y}(x)u(x)\} \geq -M \right\}\,.
%\mathcal{B}_{n,M}^- & = \left\{ x\in T_n^{-1}(\Omega_n') \, : \, \hat{y}(x)u(x) \geq y_n(T_n(x)) u_n(T_n(x)) \geq -M \right\}. 
%\end{align*}
\begin{align*}
\cA_{n,M} & = \left\{ x\in T_n^{-1}(\Omega_n') \, : \, \min\{\hat{y}(x) u(x),y_n(T_n(x)) u_n(T_n(x))\} \geq -M \right\} \\ 
\cB_{n,M} & = \left\{ x\in T_n^{-1}(\Omega_n') \, : \, \hat{y}(x) u(x) \geq y_n(T_n(x)) u_n(T_n(x)) \leq -M \right\} \\ 
\cC_{n,M} & = \left\{ x\in T_n^{-1}(\Omega_n') \, : \, y_n(T_n(x)) u_n(T_n(x)) \geq \hat{y}(x) u(x) \leq -M \right\}\,.
\end{align*}
%\cC_{n,M} & = \left\{ x\in T_n^{-1}(\Omega_n') \, : \, y_n(T_n(x)) u_n(T_n(x)) \geq \hat{y}(x) u(x) \geq -M \right\} \\ 
The quantity we want to estimate satisfies
\begin{align*}
& \left| \int_{T_n^{-1}(\Omega_n')} \log \left(\Psi(y_n(T_n(x)) u_n(T_n(x))\right) - \log \Psi\left( \hat{y}(x) u(x) \right) \, \dd \mu(x) \right| \\
& \qquad \qquad \qquad \leq  \int_{T_n^{-1}(\Omega_n')} F(y_n(T_n(x)) u_n(T_n(x)), \hat{y}(x) u(x)) \, \dd \mu(x).
\end{align*}
Since $T_n^{-1}(\Omega_n') = \mathcal{A}_{n,M} \cup \mathcal{B}_{n,M} \cup \mathcal{C}_{n,M}$ we proceed by estimating the integral over each of the sets, utilizing the bounds in Lemma~\ref{lem:LimitThmOpt:Probit:PsiBound}.

%
%& \quad \quad \leq \int_{\mathcal{A}_{n,M}^+} F(y_n(T_n(x)) u_n(T_n(x)), \hat{y}(x) u(x)) \, \dd \mu(x) \\
%& \quad \quad \quad \quad \quad \quad + \int_{\mathcal{B}_{n,M}^+} F(y_n(T_n(x)) u_n(T_n(x)), \hat{y}(x) u(x)) \, \dd \mu(x) \\
%& \quad \quad \quad \quad \quad \quad + \int_{\mathcal{A}_{n,M}^-} F(y_n(T_n(x)) u_n(T_n(x)), \hat{y}(x) u(x)) \, \dd \mu(x) \\
%& \quad \quad \quad \quad \quad \quad + \int_{\mathcal{B}_{n,M}^-} F(y_n(T_n(x)) u_n(T_n(x)), \hat{y}(x) u(x)) \, \dd \mu(x).
%\end{align*}
%We bound each of the integrals on the right hand side in turn using the bounds in Lemma~\ref{lem:LimitThmOpt:Probit:PsiBound}.

%% Done (A)
\begin{align*}
\begin{split} %\label{estam}
& \int_{\mathcal{A}_{n,M}} F(y_n(T_n(x)) u_n(T_n(x)), \hat{y}(x) u(x)) \, \dd \mu(x) \\
& \quad \quad \leq \frac{1}{\int_{-\infty}^{-M} e^{-\frac{t^2}{2}} \, \dd t} \int_{\mathcal{A}_{n,M}} \left| y_n(T_n(x)) u_n(T_n(x)) - \hat{y}(x) u(x) \right| \, \dd \mu(x) \\
& \quad \quad \leq \frac{1}{\int_{-\infty}^{-M} e^{-\frac{t^2}{2}} \, \dd t} \left(\|y_n\|_{L^2_{\mu_n}} \|u_n\circ T_n - u \|_{L^2_{\mu}} + \|u\|_{L^2_{\mu}}\| \hat{y}_n \circ T_n -\hat{y}\|_{L^2_{\mu}} \right).
\end{split}
\end{align*}
%& \quad \quad \leq \frac{1}{\int_{-\infty}^{-M} e^{-\frac{t^2}{2}} \, \dd t} \|y_n\circ T_n u_n\circ T_n - \hat{y} u\|_{L^1} \\

%% Done (B)
\begin{align*}
\begin{split} %\label{estbp}
& \int_{\mathcal{B}_{n,M}} F(y_n(T_n(x)) u_n(T_n(x)), \hat{y}(x) u(x)) \, \dd \mu(x) \\
& \quad \leq \int_{\mathcal{B}_{n,M}} 2 |y_n(T_n(x))|^2 |u_n(T_n(x))|^2 \, \dd \mu(x) + \frac{1}{M^2} \\
& \quad \leq 2 \|\hat{y}_n\|_{L^\infty_{\mu_n}}^2 \int_{\mathcal{B}_{n,M}} |u_n(T_n(x))|^2 \, \dd \mu(x) + \frac{1}{M^2} \\
& \quad \leq 4 \|\hat{y}_n\|_{L^\infty_{\mu_n}}^2 \left( \|u_n\circ T_n - u\|_{L^2_{\mu}}^2 + \int_\Omega |u(x)|^2 \mathbb{I}_{|y_n(T_n(x))u_n(T_n(x))|\geq M} \, \dd \mu(x) \right) + \frac{1}{M^2}.
\end{split}
\end{align*}

%% Done (C)
\begin{align*}
\begin{split} %\label{estap}
& \int_{\mathcal{C}_{n,M}} F(y_n(T_n(x)) u_n(T_n(x)), \hat{y}(x) u(x)) \, \dd \mu(x) \\
& \quad \leq \int_{\mathcal{C}_{n,M}} 2 |\hat{y}(x)|^2|u(x)|^2 \, \dd \mu(x) + \frac{1}{M^2} \\
& \quad \leq 2 \|\hat{y}\|^2_{L^\infty_\mu} \int_{\Omega} |u(x)|^2 \mathbb{I}_{|y(x)u(x)|\geq M} \, \dd \mu(x) + \frac{1}{M^2}. 
\end{split}
\end{align*}

%\begin{align}
%\begin{split} \label{estbm}
%& \int_{\mathcal{B}_{n,M}^-} F(\hat{y}(x) u(x),y_n(T_n(x)) u_n(T_n(x))) \, \dd \mu(x) \\
%& \quad \quad \leq \frac{1}{\int_{-\infty}^{-M} e^{-\frac{t^2}{2\gamma^2}} \, \dd t} \int_{\mathcal{B}_{n,M}^-} \left| \hat{y}(x) u(x) - y_n(T_n(x)) u_n(T_n(x)) \right| \, \dd \mu(x) \\
%& \quad \quad \leq \frac{1}{\int_{-\infty}^{-M} e^{-\frac{t^2}{2\gamma^2}} \, \dd t} \left(\|u\|_{L^2_{\mu}} \|\hat{y}_n\circ T_n - \hat{y} \|_{L^2_{\mu}} + \|\hat{y}_n\|_{L^2_{\mu_n}}\| u_n \circ T_n -u\|_{L^2_{\mu}} \right).
%\end{split}
%\end{align}

For every subsequence there exists a further subsequence such that $(y_n\circ T_n) (u_n\circ T_n) \to y u$ pointwise a.e., hence by the dominated convergence theorem
\[ \int_\Omega |u(x)|^2 \mathbb{I}_{|y_n(T_n(x))u_n(T_n(x))|\geq M} \, \dd \mu(x)\to \int_\Omega |u(x)|^2 \mathbb{I}_{|y(x)u(x)|\geq M} \, \dd \mu(x) \quad \text{as } n\to \infty. \]
Hence, for $M\geq 1$ fixed we have
\begin{align*}
& \limsup_{n\to \infty} \left| \int_{T_n^{-1}(\Omega_n')} \log \left(\Psi(y_n(T_n(x)) u_n(T_n(x)) ; \gamma\right) - \log\left( \hat{y}(x) u(x);\gamma \right) \, \dd \mu(x) \right| \\
& \quad \quad \quad \quad \quad \quad \leq \frac{2}{M^2}+ 6\|\hat{y} \|_{L^\infty_\mu} \int_\Omega |u(x)|^2 \mathbb{I}_{|\hat{y}(x)u(x)|\geq M} \, \dd \mu(x).
\end{align*}
Taking $M\to \infty$ completes the proof.
\end{proof}

The proof of Theorem~\ref{thm:LimitThmOpt:Probit:pr} is now just a special case of the above lemma and an easy compactness result that follows from Theorem~\ref{thm:LimitThmDir:LimitThmDir}.

\begin{proof}[Proof of Theorem  \ref{thm:LimitThmOpt:Probit:pr}]
The following statements all hold with probability one.
Let 
\[ y(x) = \left\{ \begin{array}{ll} 1 & \text{if } x\in \Omega^+ \\ -1 &\text{if } x\in \Omega^-. \end{array} \right. \]
Since $\dist(\Omega^+,\Omega^-) > 0$ there exists a minimal Lipschitz extension $\hat{y}\in L^\infty$ of $y$ to $\Omega$. 
Let  $y_n=y\lfloor_{\Omega_n}$ and $\hat{y}_n = \hat{y}\lfloor_{\Omega_n}$.
Since
\begin{align*}
\|\hat{y}_n\circ T_n - \hat{y} \|_{L^\infty(\mu)} & = \muesssup_{x\in \Omega} |\hat{y}_n(T_n(x)) - \hat{y}(x)| \\
 &= \muesssup_{x\in \Omega} |\hat{y}(T_n(x)) - \hat{y}(x)| \\
 & \leq \Lip(\hat{y}) \|T_n-\Id\|_{L^\infty}
 \end{align*}
we conclude that $(\mu_n, \hat{y}_n)\to (\mu, \hat{y})$ in $TL^\infty$.
Hence, by Lemma~\ref{lem:LimitThmOpt:Probit:CC}, $\frac{1}{n}\Phip^{(n)}(u_n;\gamma) \to \Phipo (u;\gamma)$
whenever $(\mu_n,u_n) \to (\mu,u)$ in $TL^p$. Combining with
 Theorem~\ref{thm:LimitThmDir:LimitThmDir} implies that $\Jp^{(n)}$ $\Gamma$-converges to $\Jp^{(\infty)}$ via a straightforward argument.

If $\tau>0$ then the compactness of minimizers follows from Theorem~\ref{thm:LimitThmDir:LimitThmDir} 
using that  $\sup_{n\in \bbN} \min_{v_n\in L^2_{\mu_n}} \Jp^{(n)}(v_n) \leq \sup_{n\in \bbN} \Jp^{(n)}(0) = \frac{1}{2}$.

When $\tau=0$ we consider the sequence $w_n=v_n-\bar{v}_n$ where $v_n$ is a minimizer of $\Jp^{(n)}$ and $\bar{v}_n = \langle v_n,q_1\rangle_{\mu_n} = \int_\Omega v_n(x) \, \dd \mu_n(x)$.
Then, $J_n^{(\alpha,0)}(w_n) = J_n^{(\alpha,0)}(v_n)$ and
\[
\| w_n \|_{L^2_{\mu_n}}^2  = \|v_n - \bar{v}_n \|_{L^2_{\mu_n}}^2  = \sum_{k=2}^n \langle v_n, q_k \rangle_{\mu_n}^2 \leq \frac{1}{(s_n \lambda_2^{(n)})^\alpha} J_n^{(\alpha,0)}(v_n).
\]
As in the case $\tau>0$ the quadratic form is bounded, i.e. $\sup_{n\in \bbN} \Jp^{(n)}(v_n) \leq \frac12$.
Hence $J_n^{(\alpha,\tau)}(w_n)\leq \frac12$ and $\| w_n \|_{L^2_{\mu_n}}^2 \leq \frac{1}{\lambda_2^\alpha}$ for $n$ large enough.
By Theorem~\ref{thm:LimitThmDir:LimitThmDir} $w_n$ is precompact in $TL^2$.
Therefore $\sup_{n\in \bbN} \|v_n\|_{L^2_{\mu_n}} \leq M+\sup_{n\in \bbN} |\bar{v}_n|$ for some $M>0$. 
Since $J_n^{(\alpha,\tau)}$ is insensitive to the addition of a constant, and $-1\leq y \leq 1$, then for any minimiser $v_n$ one must have $\bar{v}_n\in [-1,1]$.
%Indeed, suppose $|v_n|=1+\delta$ for $\delta>0$, then $J_n^{(\alpha,0)}(v_n) = J_n^{(\alpha,0)}(v_n-\delta)$ and $\Phipo^{(n)}(v_n) \geq \Phipo^{(n)}(v_n-\delta)$.
Hence $\sup_{n\in \bbN} \|v_n\|_{L^2_{\mu_n}}\leq M+1$ so by Theorem~\ref{thm:LimitThmDir:LimitThmDir} $\{v_n\}$ is precompact in $TL^2$.

Since the minimizers  of $\Jp^{(\infty)}$ are unique (due to convexity, see  Lemma~\ref{lem:Limits:Probit:Unique}),
by Proposition~\ref{prop:Background:Passage:GammaConv:MinConv} 
we have that the sequence of minimizers  $v_n$ of $\Jp^{(n)}$ converges to the minimizer of $\Jp^{(\infty)}$.
%This follows from strict convexity as in~\cite[Proposition 1]{UQ17}.
\end{proof}

\subsection{Variational Convergence of Probit in Labelling Model 2}
\label{ssec:PRLabelMod2}

\begin{proof}[Proof of Theorem \ref{thm:LimitThmOpt:Probit:pr2neg}] 
It suffices to show that $\Jp^{(n)}$ $\Gamma$-converges in $TL^2$ to  $J_\infty^{(\alpha,\tau)}$ and that the sequence of minimizers $v_n$ of $\Jp^{(n)}$ is precompact in $TL^2$. We note that the liminf statement of the $\Gamma$-convergence follows immediately from statement 1. of Theorem \ref{thm:LimitThmDir:LimitThmDir}.

To complete the proof of  $\Gamma$-convergence it suffices to construct a recovery sequence. The strategy is analogous to the one of the proof on Theorem 4.9 of \cite{SlepcevThorpe}. 
Let $v \in \cH^\alpha(\Omega)$. Since $J_{n}^{(\alpha,\tau)}$ $\Gamma$-converges to $J_\infty^{(\alpha,\tau)}$ by Theorem \ref{thm:LimitThmDir:LimitThmDir} there exists
Let $v^{(n)} \in L^2_{\mu_n}$ such that $J_{n}^{(\alpha,\tau)}(v^{(n)}) \to J_\infty^{(\alpha,\tau)}(v)$ as $n \to \infty$. Consider the functions
\[
\tilde v^{(n)}(x_i) = \begin{cases}
c_n y(x_i) \quad & \te{if } i=1, \dots, N.  \\
v^{(n)}(x_i) \quad & \te{if } i=N+1, \dots, n
\end{cases}
\]
where $c_n \to \infty$ and $\frac{c_n}{\eps_n^{2\alpha}n} \to 0$ as $n \to \infty$. 

Note that condition \eqref{eq:LimitThmDir:epsSca} implies that when $\alpha<\frac{d}{2}$ then
\eqref{epsn_spike} still holds. Therefore \eqref{ene_sing} implies that $J_{n}^{(\alpha,\tau)}(c_n \delta_{x_i}) \to 0$ as $n \to \infty$. Also note that since $c_n \to \infty$, $\Phip^{(n)}(\tilde{v}^{(n)};\gamma) \to 0$ as $n \to \infty$.
It is now straightforward to show, using the form of the functional, the estimate on the energy of a singleton and the fact that $\eps_n n^\frac{1}{2\alpha} \to \infty$ as $n \to \infty$, that $\Jp^{(n)}(\tilde v^{(n)}) \to J_\infty^{(\alpha,\tau)}(v)$ as desired. 

The precompactness of $\{v_n\}_{n \in \N}$ follows from Theorem~\ref{thm:LimitThmDir:LimitThmDir}.
Since $0$ is the unique minimizer of $J_\infty^{(\alpha,\tau)}$, due to $\tau>0$, the above results imply that $v^{(n)}$ converge to $0$.
\end{proof}

\subsection{Small Noise Limits}
\label{ssec:SNL3}

\begin{proof}[Proof of Theorem \ref{thm:ltp:probit:ZeroNoise}] %$ $\newline
\hspace{0cm}
First observe that since Assumptions \ref{a:omega}--\ref{a:rho} hold and $\alpha > d/2$, the measure $\nu_0$, and hence the measures $\nipo, \nipt, \nu_1$, are all well-defined measures on $L^2(\Omega)$ by Theorem \ref{t:g}.
\begin{enumerate}
\item[(i)]   For any continuous bounded function $g: C(\Omega;\bbR)
\to \bbR$ we have
$$\bbE^{\nipo} g(u)=\frac{\bbE^{\nu_0} e^{-\Phipo(u;\gamma)}g(u)}{\bbE^{\nu_0} e^{-\Phipo(u;\gamma)}},
\quad \bbE^{\nu_1} g(u)=\frac{\bbE^{\nu_0} \one_{B_{\infty,1}}(u)g(u)}{\bbE^{\nu_0} \one_{B_{\infty,1}}(u)}.$$ 
For the first convergence it thus suffices to prove that, as $\gamma \to 0$,
$$\bbE^{\nu_0} e^{-\Phipo(u;\gamma)}g(u) \to \bbE^{\nu_0} \one_{B_{\infty,1}}(u)g(u)$$
for all continuous functions $g: C(\Omega;\bbR) \to [-1,1]$.

We first define the standard normal cumulative distribution function $\varphi(z) = \Psi(z,1)$, and note that we may write
\[
\Phipo(u;\gamma) = -\int_{x \in \Omega'}\log \Bigl(\varphi(y(x)u(x)/\gamma)\Bigr)\dd x \geq 0.
\]

In what follows it will be helpful to recall the following standard Mills ratio bound: for all $t > 0$,
\begin{align}
\label{eq:ltp:mills}
\varphi(t) \geq 1 - \frac{e^{-t^2/2}}{t\sqrt{2\pi}}.
\end{align}

Suppose first that $u \in B_{\infty,1}$, then $y(x)u(x)/\gamma > 0$ for a.e. $x \in \Omega'$. The assumption that $\overline{\Omega^+} \cap \overline{\Omega^-} = \emptyset$ ensures that $y$ is continuous on $\Omega' = \Omega^+\cup\Omega^-$. As $u$ is also continuous on $\Omega'$, given any $\VE > 0$, we may find $\Omega_\VE'\subseteq\Omega'$ such that $y(x)u(x)/\gamma > \VE/\gamma$ for all $x \in \Omega_\VE'$. Moreover, these sets may be chosen such that ${\rm leb}(\Omega'\setminus\Omega_\VE') \to 0$ as $\VE\to 0$. Applying the bound \eqref{eq:ltp:mills}, we see that for any $x \in \Omega_\VE'$,
\[
\varphi(y(x)u(x)/\gamma) \geq 1-\gamma \frac{e^{-u(x)^2y(x)^2/2\gamma^2}}{u(x)y(x)\sqrt{2\pi}} \geq 1-\gamma \frac{e^{-\VE^2/2\gamma^2}}{\VE\sqrt{2\pi}}.
\]
Additionally, for any $x \in \Omega'\setminus\Omega_\VE'$, we have $\varphi(y(x)u(x)/\gamma) \geq \varphi(0) = 1/2$. We deduce that
\begin{align*}
\Phipo(u;\gamma) &= -\int_{\Omega_\VE'} \log(\varphi(y(x)u(x)/\gamma) \, \dd \mu(x) -\int_{\Omega'\setminus\Omega_\VE'} \log(\varphi(y(x)u(x)/\gamma) \, \dd \mu(x)\\
&\leq -\log\left(1-\gamma \frac{e^{-\VE^2/2\gamma^2}}{\VE\sqrt{2\pi}}\right)\cdot\rop\cdot {\rm leb}(\Omega_\VE') + \log(2)\cdot\rop \cdot{\rm leb}(\Omega'\setminus\Omega_\VE').
\end{align*}
The right-hand term may be made arbitrarily small by choosing $\VE$ small enough. For any given $\VE > 0$, the left-hand term tends to zero as $\gamma \to 0$, and so we deduce that $\Phipo(u;\gamma)\to 0$ and hence
\[
e^{-\Phipo(u;\gamma)}g(u) \to g(u) = \one_{B_{\infty,1}}(u)g(u).
\]

Now suppose that $u \notin B_{\infty,1}$, and assume first that there is a subset $E\subseteq\Omega'$ with ${\rm leb}(E) > 0$ and $y(x)u(x) < 0$ for all $x \in E$. Then similarly to above, there exists $\VE>0$ and $E_\VE\subseteq E$ with ${\rm leb}(E_\VE) > 0$ such that $y(x)u(x)/\gamma < -\VE/\gamma$ for all $x \in E_\VE$. Observing that $\varphi(t) = 1-\varphi(-t)$, we may apply the bound (\ref{eq:ltp:mills}) to deduce that, for any $x \in E_\VE$,
\begin{align*}
\varphi(y(x)u(x)/\gamma) \leq -\gamma \frac{e^{-u(x)^2y(x)^2/2\gamma^2}}{u(x)y(x)\sqrt{2\pi}} \leq \frac{\gamma}{\VE\sqrt{2\pi}}.
\end{align*}
We therefore deduce that
\begin{align*}
\Phipo(u;\gamma) &\geq \int_{E_\VE}  -\log(\varphi(y(x)u(x)/\gamma) \, \dd \mu(x)\\
&\geq -\log\left(\frac{\gamma}{\VE\sqrt{2\pi}}\right)\cdot\rom\cdot{\rm leb}(E_\VE) \to \infty
\end{align*}
from which we see that
\[
e^{-\Phipo(u;\gamma)}g(u) \to 0 = \one_{B_{\infty,1}}(u)g(u).
\]
Assume now that $y(x)u(x) \geq 0$ for a.e. $x \in \Omega'$. Since $u \notin B_{\infty,1}$ there is a subset $\Omega''\subseteq \Omega'$ such that $y(x)u(x) = 0$ for all $x \in \Omega''$, $y(x)u(x) > 0$ a.e. $x \in \Omega'\setminus\Omega''$, and ${\rm leb}(\Omega'') > 0$. We then have
\begin{align*}
\Phipo(u;\gamma) &= -\int_{\Omega''} \log(\varphi(0)) \, \dd \mu(x) -\int_{\Omega'\setminus\Omega''} \log(\varphi(y(x)u(x)/\gamma) \, \dd \mu(x) \\
&= \log(2)\mu(\Omega'') -\int_{\Omega'\setminus\Omega''} \log(\varphi(y(x)u(x)/\gamma) \, \dd \mu(x) \\
&\to \log(2)\mu(\Omega'').
\end{align*}
We hence have $e^{-\Phip(u;y,\gamma)}g(u) \not\to 0 = \one_{B_{\infty,1}}(u)g(u)$. However, the event
\begin{align*}
D &:= \{u \in C(\Omega;\R)\,|\,\text{There exists $\Omega''\subseteq\Omega'$ with ${\rm leb}(\Omega'') > 0$ and $u|_{\Omega''}= 0$}\}\\
&\subseteq \{u \in C(\Omega;\R)\,|\,{\rm leb}\big(u^{-1}\{0\}\big) > 0\} = D'
\end{align*}
has probability zero under $\nu_0$. 
This can be deduced from Proposition 7.2 in \cite{iglesias2015bayesian}: since Assumptions \ref{a:omega}--\ref{a:rho} hold and $\alpha >d$, Theorem \ref{t:g} tells us that draws from $\nu_0$ are almost-surely continuous, which is sufficient
 in order to deduce the conclusions of the proposition, and so $\nu_0(D)\leq \nu_0(D') = 0$. We thus have pointwise convergence of the integrand on $D^c$, and so using the boundedness of the integrand by $1$ and the dominated convergence theorem, 
\[
\bbE^{\nu_0} e^{-\Phipo(u;\gamma)}g(u) = \bbE^{\nu_0} e^{-\Phipo(u;\gamma)}g(u)\one_{D^c}(u) \to \bbE^{\nu_0} \one_{B_{\infty,1}}(u)g(u)
\]
which proves that $\nipo\toweak \nu_1$.

For the convergence $\nilo\toweak \nu_1$ it similarly suffices to prove that, as $\gamma \to 0$,
$$\bbE^{\nu_0} e^{-\Philo(u;\gamma)}g(u) \to \bbE^{\nu_0} \one_{B_{\infty,1}}(u)g(u)$$
for all continuous functions $g: C(\Omega;\bbR) \to [-1,1]$. For fixed $u \in B_{\infty,1}$ we have $e^{-\Philo(u;\gamma)}=\one_{B_{\infty,1}}(u)=1$ and hence
$e^{-\Philo(u;\gamma)}g(u)=\one_{B_{\infty,1}}(u)g(u)$ for all $\gamma>0.$ For fixed $u \notin B_{\infty,1}$
 there is a set $E \subseteq \Omega'$ with positive Lebesgue measure
on which $y(x)u(x) \le 0$. As a consequence
$\Philo(u;\gamma) \ge \frac{1}{2\gamma^2}{\rm leb}(E)\rom$ and so
$e^{-\Philo(u;\gamma)}g(u) \to 0=\one_{B_{\infty,1}}(u)g(u)$ as $\gamma \to 0.$
Pointwise convergence of the integrand, combined with boundedness by $1$ 
of the integrand, gives the result.

\item[(ii)] The structure of the proof is similar to part (i). To prove $\nipt\toweak\nu_2$, it suffices to show that, as $\gamma\to 0$,
$$\bbE^{\nu_0} e^{-\Phipt(u;\gamma)}g(u) \to \bbE^{\nu_0} \one_{B_{\infty,2}}(u)g(u)$$
for all continuous functions $g: C(\Omega;\bbR) \mapsto [-1,1].$
We write
\[
\Phip^{(n)}(u;\gamma) = -\frac{1}{n}\sum_{j \in Z'}\log \Bigl(\varphi(y(x_j)u(x_j)/\gamma)\Bigr) \geq 0.
\]
Note that $\Phip^{(n)}(u;\gamma)$ is well-defined almost-surely on samples from $\nu_0$ since $\nu_0$ is supported on continuous functions (Theorem \ref{t:g}). Suppose first that $u \in B_{\infty,2}$, then $y(x_j)u(x_j)/\gamma > 0$ for all $j \in Z'$ and $\gamma > 0$. It follows that for each $j \in Z'$, $y(x_j)y(x_j)/\gamma \to \infty$ as $\gamma \to 0$ and so $\varphi(y(x_j)u(x_j)/\gamma)\to 1$. Thus, $\Phipt(u;\gamma) \to 0$ and so
\[
e^{-\Phipt(u;\gamma)}g(u) \to g(u) = \one_{B_{\infty,2}}(u)g(u).
\]
Now suppose that $u \notin B_{\infty,2}$. Assume first that there is a $j \in Z'$ such that $y(x_j)u(x_j) < 0$, so that $y(x_j)u(x_j)/\gamma \to -\infty$ and hence $\varphi(y(x_j)u(x_j)/\gamma)\to 0$. Then we may bound
\[
\Phipt(u;\gamma) \geq -\log(\varphi(y(x_j)u(x_j)/\gamma) \to \infty
\]
from which we see that
\[
e^{-\Phipt(u;\gamma)}g(u) \to 0 = \one_{B_{\infty,2}}(u)g(u).
\]
Assume now that $y(x_j)u(x_j) \geq 0$ for all $j \in Z'$, then since $u \notin B_{\infty,2}$ there is a subcollection $Z''\subseteq Z'$ such that $y(x_j)u(x_j) = 0$ for all $j \in Z''$ and $y(x_j)u(x_j) > 0$ for all $j \in Z'\setminus Z''$. We then have
\begin{align*}
\Phipt(u;\gamma) &= -\frac{1}{n}\sum_{j \in Z''}\log \Bigl(\varphi(0)\Bigr) -\frac{1}{n}\sum_{j \in Z'\setminus Z''}\log \Bigl(\varphi(y(x_j)u(x_j)/\gamma)\Bigr)\\
&= \frac{|Z''|}{n}\log(2)-\frac{1}{n}\sum_{j \in Z'\setminus Z''}\log \Bigl(\varphi(y(x_j)u(x_j)/\gamma)\Bigr)\\
&\to \frac{|Z''|}{n}\log(2).
\end{align*}
Thus, in this case $e^{-\Phipt(u;\gamma)}g(u) \not\to 0 = \one_{B_{\infty,2}}(u)g(u)$. However, the event
\[
D = \{u \in C(\Omega;\R)\,|\, u(x_j) = 0\text{ for some }j \in Z'\}
\]
has probability zero under $\nu_0$. To see this, observe that $\nu_0$ is a non-degenerate Gaussian measure on $C(\Omega;\R)$ as a consequence of Theorem \ref{t:g}. Thus $u \sim \nu_0$ implies that the vector $(u(x_1),\ldots,u(x_{n^++n^-}))$ is a non-degenerate Gaussian random variable on $\R^{n^++n^-}$. Its law is hence equivalent to the Lebesgue measure, and so the probability that it takes value in any given hyperplane is zero. We therefore have pointwise convergence of the integrand on $D^c$. Since the integrand is bounded by $1$, we deduce from the dominated convergence theorem that 
\[
\bbE^{\nu_0} e^{-\Phipt(u;\gamma)}g(u) = \bbE^{\nu_0} e^{-\Phipt(u;\gamma)}g(u)\one_{D^c}(u) \to \bbE^{\nu_0} \one_{B_{\infty,2}}(u)g(u)
\]
which proves that $\nipt\toweak \nu_2$.

To prove $\nilt\toweak\nu_2$ we show that, as $\gamma \to 0$,
$$\bbE^{\nu_0} e^{-\Philt(u;\gamma)}g(u) \to \bbE^{\nu_0} \one_{B_{\infty,2}}(u)g(u)$$
for all continuous functions $g: C(\Omega;\bbR) \mapsto [-1,1].$
For fixed $u \in B_{\infty,2}$ we have $e^{-\Philt(u;\gamma)}=\one_{B_{\infty,2}}(u)=1$ and hence
$e^{-\Philt(u;\gamma)}g(u)=\one_{B_{\infty,2}}(u)g(u)$ for all $\gamma>0.$ For fixed $u \notin B_{\infty,2}$
 there is at least one $j \in Z'$ such that $y(x_j)u(x_j) \leq 0$. As a consequence
$\Philt(u;\gamma) \ge \frac{1}{2\gamma^2}\frac{1}{n}\rom$ and so
$e^{-\Philt(u;\gamma)}g(u) \to 0=\one_{B_{\infty,2}}(u)g(u)$ as $\gamma \to 0.$
Pointwise convergence of the integrand, combined with boundedness by $1$ 
of the integrand, gives the desired result.
\end{enumerate}
\end{proof}

\subsection{Technical lemmas}
\label{ssec:Lem}

We include technical lemmas which are used in the main $\Gamma$-convergence result (Theorem~\ref{thm:LimitThmDir:LimitThmDir}) and in the proof of convergence for the probit model.
\begin{lemma} \label{lem:a}
Let $X$ be a normed space and $a_k^{(n)} \in X$ for all $n \in \N$ and $k=1, \dots, n$. Assume $a_k \in X$ be such that $\sum_{k=1}^\infty \|a_k\| < \infty$ and that for all $k$
\[ a_k^{(n)} \to a_k \quad \te{ as } n \to \infty. \]
Then there exists a sequence $\{K_n\}_{n=1, \dots}$ converging to infinity as $n \to \infty$ such that 
\[ \sum_{k=1}^{K_n} a_k^{(n)} \to \sum_{k=1}^\infty a_k \quad \te{ as } n \to \infty. \] 
\end{lemma}
Note that if the conclusion holds for one sequence $K_n$ it also holds for any other sequence converging to infinity and majorized by $K_n$. 
\begin{proof}
Note that by our assumption for any fixed $s$, $ \sum_{k=1}^{s} a_k^n \to \sum_{k=1}^s a_k$ as $n \to \infty$. Let $K_n$ be the largest number such that for all $m \geq n$, $\left\|  \sum_{k=1}^{K_n} a_k^{(m)} - \sum_{k=1}^{K_n} a_k\right\| < \frac{1}{n}$. % It is sufficient that $K_n$ satisfies $K_n\max_{k\in\{1,...,K_n\}} b_k^{(m)}<1/n$ for all $m\geq n$ where $b_k^{(m)} = \|a_k^{(m)}-a_k\|$
Due to observation above, $K_n \to \infty$ as $n \to \infty$. Furthermore
\[ \left\| \sum_{k=1}^{K_n} a_k^n - \sum_{k=1}^\infty a_k \right\| \leq  \left\| \sum_{k=1}^{K_n} a_k^n - \sum_{k=1}^{K_n} a_k \right\| +  \left\| \sum_{k=K_n+1}^\infty a_k \right\| \]
which converges to zero an $n \to \infty$. 
\end{proof}

The second result is an estimate on the behavior of the function $\Psi$ defined in 
\eqref{eq:Background:Discrete:Probit:Psi}

\begin{lemma}
\label{lem:LimitThmOpt:Probit:PsiBound}
Let $F(w,v) = \log \Psi(w;1) - \log\Psi(v;1)$ where $\Psi$ is defined by~\eqref{eq:Background:Discrete:Probit:Psi} with $\gamma=1$.
For all $w>v$ and $M\geq 1$,
%\begin{align*}
\[ F(w,v) \leq 
\begin{cases}
 2v^2 + \frac{1}{M^2}
 \quad & \text{if } v\leq -M  \vspace*{5pt}\\
  \frac{|w-v|}{\int_{-\infty}^{-M} e^{-\frac{t^2}{2}} \, \dd t} & \text{if } v \geq -M. 
\end{cases}
\]
%\end{align*}
\end{lemma}

\begin{proof}
We consider the two cases: $v\leq -M$ and $v\geq -M$ separately.
From inequality 7.1.13 in \cite{AbramowitzStegun} directly follows that 
%\begin{equation} \label{tpmpe2}
\[ \forall u \leq 0, \quad \quad \sqrt{\frac{2}{\pi}}  \, \frac{1  }{- u + \sqrt{ u^2+4}}\, e^{-\frac{u^2}{2}} \leq \Psi(u) \]
% \end{equation}
%The following inequalities will be useful
%\begin{equation} \label{tpmpe}
%\forall u\leq 0, \quad \quad -\frac{u  }{\sqrt{2\pi} (u^2+1)} e^{-\frac{u^2}{2}} \leq \Psi(u) \leq -\frac{1}{\sqrt{2\pi} u}  e^{-\frac{u^2}{2}}.
%\end{equation}

When $v\leq -M$, by taking the logarithm we obtain
\begin{align*}
F(w,v) &
 \leq -\log \Psi(v;\gamma) 
 \leq -\log\left(\sqrt{\frac{2}{\pi}}  \, \frac{1}{-  v + \sqrt{  v^2+4}}\, e^{-\frac{ v^2}{2}} \right) 
 \leq \sqrt{\frac{\pi}{2}} \left( \sqrt{v^2+4}-v \right) + \frac{v^2}{2} \\
& \leq \sqrt{\frac{\pi}{2}}|v|\left( \sqrt{1+\frac{4}{M^2}}-1\right) + \frac{v^2}{2}
 \leq \frac{\sqrt{2\pi}|v|}{M} + \frac{v^2}{2}
 \leq 2v^2 + \frac{1}{M^2}
\end{align*}
using the elementary bound $|\sqrt{1+x^2}-1|\leq |x|$ for all $x\geq 0$.
% & = \log\left( \sqrt{2\pi} \left(\frac{|v|}{\gamma} + \frac{\gamma}{|v|} \right) \right) + \frac{v^2}{2\gamma^2} \\
% & \leq \sqrt{2\pi} \left(\frac{|v|}{\gamma} + \frac{\gamma}{|v|} \right) + \frac{v^2}{2\gamma^2} &\text{since } \log x \leq x \\
% & \leq \sqrt{2\pi} \left(\frac{|v|}{\gamma} + \frac{\gamma}{M} \right) + \frac{v^2}{2\gamma^2} & \text{since } |v| \geq M.
When $v\geq -M$,
\[
F(w,v)  = \log \frac{\Psi(w)}{\Psi(v)} 
  = \log \left( 1 + \frac{\int_v^w e^{-\frac{t^2}{2}} \, \dd t}{\int_{-\infty}^v e^{-\frac{t^2}{2}} \, \dd t} \right) 
 \leq \frac{\int_v^w e^{-\frac{t^2}{2}} \, \dd t}{\int_{-\infty}^v e^{-\frac{t^2}{2}} \, \dd t} 
  \leq \frac{w-v}{\int_{-\infty}^{-M} e^{-\frac{t^2}{2}} \, \dd t} 
\]
This completes the proof.
\end{proof}

\begin{corollary}
\label{lem:Background:Cont:Probit:PhipL1}
Let $\Omega'\subset \bbR^d$ be open and bounded.
Let $\mu'$ be a bounded, nonnegative measure on $\Omega'$  and $\gamma>0$.
Define $\Psi(\cdot;\gamma)$ as in~\eqref{eq:Background:Discrete:Probit:Psi}.
If $v\in L^2_{\mu'}$ %satisfies $J_\infty^{(\alpha,\tau)}(v)<+\infty$ for $\alpha,\tau>0$
then $\log \Psi(v;\gamma)\in L^1(\mu')$.
\end{corollary}
\begin{proof}
Lemma \ref{lem:LimitThmOpt:Probit:PsiBound}, and using that  $\Psi(v;\gamma) = \Psi(v/\gamma;1)$,
%By this, and since $\log\Psi(v)<0$, it is enough to show that
 shows that $-\log \Psi(v, \gamma)$ grows quadratically as $v \to -\infty$. Note that $-\log \Psi(v, \gamma)$ asymptotes to zero as $v \to \infty$. Therefore $|\log \Psi (v, \gamma)| \leq C (|v|^2+1)$ for some $C>0$, which implies the claim.
\end{proof}

\subsection{Weyl's Law}
\label{app:weyl}

\begin{lemma} \label{lem:weyl}
Let $\Omega$ and $\rho$ satisfy Assumtptions~\ref{a:omega}--\ref{a:rho} and let $\lambda_k$ be the eigenvalues of $\cL$ defined by~\eqref{eq:Background:Cont:cL}.
Then, there exist positive constants $c$ and $C$ such that for all $k$ large enough
\[ c k^{\frac{2}{d}} \leq  \lambda_k \leq C k^{\frac{2}{d}}.\]
\end{lemma}
\begin{proof}
Let $B$ be a ball compactly contained in $\Omega$ and $U$ a ball which compactly contains $\Omega$.
By assumptions on $\rho$ for all $u \in H^1_0(B) \backslash \{0\}$
\[ \frac{\int_B |\nabla u|^2 dx}{\int_B u^2 dx} \geq c_2  \frac{\int_\Omega |\nabla u|^2 \rho^2 dx}{\int_\Omega u^2 \rho dx}  \]
where on RHS we consider the extension by zero of $u$ to $\Omega$. Therefore for any $k$-dimensional subspace $V_k$ of  $H^1_0(B)$
\[ \max_{u \in V_k \backslash \{0\}}  \frac{\int_B |\nabla u|^2 dx}{\int_B u^2 dx} \geq c_2  \max_{u \in V_k \backslash \{0\}}  \frac{\int_\Omega |\nabla u|^2 \rho^2 dx}{\int_\Omega u^2 \rho dx}.  \]
Consequently, using the Courant--Fisher characterization of eigenvalues,
\[ \alpha_k = \inf_{\substack{V_k \subset H_0^1(B), \\ \te{dim } V_k = k}}  \max_{u \in V_k \backslash \{0\}}  \frac{\int_B |\nabla u|^2 dx}{\int_B u^2 dx} \geq c_2
 \inf_{\substack{V_k \subset H^1(\Omega), \\ \te{dim } V_k = k} }
\max_{u \in V_k \backslash \{0\}}  \frac{\int_\Omega |\nabla u|^2 \rho^2 dx}{\int_\Omega u^2 \rho dx} = c_2 \lambda_k \]

Since $\overline \Omega$ is an extension domain (as it has a Lipschitz boundary), there exists an bounded extension operator $E: H^1(\Omega) \to H^1_0(U)$. Therefore for some constant $C_2$ and all $u \in H^1(\Omega)$,  $C_2 \int_\Omega |\nabla u|^2 \rho^2+ u^2  \rho dx \geq
\int_U |\nabla Eu|^2 dx$. Arguing as above gives $C_2(\lambda_k+1) \geq \beta_k$.

These inequalities imply the claim of the lemma, since
the Dirichlet eigenvalues of the Laplacian on $B$, $\alpha_k$ satisfy
$\alpha_k \leq C_1  k^{\frac{2}{d}}$ for some $C_1$ and that Dirichlet eigenvalues of the Laplacian on $U$, $\beta_k$ satisfy $\beta_k \geq c_1  k^{\frac{2}{d}}$ for some $c_1>0$.
\end{proof}

\end{document}